\documentclass{tlp}

\usepackage[utf8]{inputenc}
\usepackage{amsmath,amssymb}
\usepackage{microtype}
\usepackage{times,helvet,courier}
\usepackage{url}
\usepackage{subcaption}
\usepackage{rotating}
\usepackage{tikz}
\usepackage{pgfplots}
\usepackage{textcomp}
\usetikzlibrary{arrows, shapes}

\usepackage{listings}
\lstdefinelanguage{clingo}{
  keywordstyle=[1]\usefont{OT1}{cmtt}{m}{n},%
  keywordstyle=[2]\textbf,%
  keywordstyle=[3]\usefont{OT1}{cmtt}{m}{n},%\textit
  alsoletter={\#,\&},%
  keywords=[1]{not,from,import,def,if,else,elif,return,while,break,and,or,for,in,del,and,class,with,print,as,is},%
  keywords=[2]{\#const,\#show,\#minimize,\#base,\#theory,\#count,\#external,\#program,\#script,\#end,\#heuristic,\#edge,\#project,\#show,\#sum},%
  keywords=[3]{&,&dom,&sum,&diff,&show},%
  morecomment=[l]{\#\ },%
  morecomment=[l]{\%\ },%
  morestring=[b]",%
  stringstyle={\textit},%
  commentstyle={\color{darkgray}}%
}

\newcommand{\sysfont}{\textit}
\newcommand{\anthem}{\sysfont{anthem}}
\newcommand{\aspic}{\sysfont{aspic}}
\newcommand{\asprin}{\sysfont{asprin}}

\newcommand{\clasp}{\sysfont{clasp}}
\newcommand{\clingcon}{\sysfont{clingcon}}
\newcommand{\clingo}{\sysfont{clingo}}

\newcommand{\dlv}{\sysfont{dlv}}
\newcommand{\jdlv}{\sysfont{jdlv}}
\newcommand{\dflat}{\sysfont{dflat}}
\newcommand{\dlvhex}{\sysfont{dlvhex}}
\newcommand{\acthex}{\sysfont{acthex}}
\newcommand{\embasp}{\sysfont{embasp}}

\newcommand{\gringo}{\sysfont{gringo}}
\newcommand{\iclingo}{\sysfont{iclingo}}
\newcommand{\idp}{\sysfont{idp}}
\newcommand{\lparse}{\sysfont{lparse}}

\newcommand{\minisat}{\sysfont{minisat}}

\newcommand{\oclingo}{\sysfont{oclingo}}
\newcommand{\rosoclingo}{\sysfont{rosoclingo}}

\newcommand{\wasp}{\sysfont{wasp}}

\newcommand{\zzz}{\sysfont{z3}}

\newcommand{\aspif}{\sysfont{aspif}}

\newcommand{\python}{Python}
\newcommand{\lua}{Lua}
\newcommand{\C}{C}
\newcommand{\cpp}{C++}
\newcommand{\java}{Java}

\newcommand{\naf}[1]{\ensuremath{{\sim}{#1}}}
\newcommand{\poslits}[1]{\ensuremath{#1^+}}
\newcommand{\neglits}[1]{\ensuremath{#1^-}}
\newcommand{\body}[1]{\ensuremath{B(#1)}} % {\ensuremath{\mathit{body}(#1)}}
\newcommand{\pbody}[1]{\poslits{\body{#1}}} % {\ensuremath{\mathit{body}^+(#1)}}
\newcommand{\nbody}[1]{\neglits{\body{#1}}} % {\ensuremath{\mathit{body}^-(#1)}}
\newcommand{\head}[1]{\ensuremath{h(#1)}} % {\ensuremath{\mathit{head}(#1)}}

\newcommand{\PRG}{\ensuremath{P}}
\newcommand{\QRG}{\ensuremath{Q}}
\newcommand{\RRG}{\ensuremath{R}}
\newcommand{\atom}[1]{\ensuremath{A(#1)}} % {\ensuremath{\mathit{atom}(#1)}}
\newcommand{\Head}[1]{\ensuremath{H(#1)}} 
\newcommand{\ground}[1]{\ensuremath{\mathit{grd}(#1)}}
\newcommand{\cground}[2]{\ensuremath{\mathit{grd}_{#2}(#1)}}
\newcommand{\dep}[1]{\ensuremath{G(#1)}}

\newcommand{\module}[1]{\ensuremath{\mathbb{#1}}}
\newcommand{\prog}[1]{\ensuremath{P(#1)}}
\newcommand{\inp}[1]{\ensuremath{I(#1)}}
\newcommand{\out}[1]{\ensuremath{O(#1)}}
\newcommand{\inst}[2]{\ensuremath{#1(#2)}}

\usepackage[]{hyperref} % hidelinks

\title{Multi-shot ASP solving with \textit{Clingo}}

\author[M. Gebser, R. Kaminski, B. Kaufmann, and T. Schaub]{%
  Martin Gebser
  \\
  University of Potsdam, Germany
  \and
  Roland Kaminski
  \\
  University of Potsdam, Germany
  \and
  Benjamin Kaufmann
  \\
  University of Potsdam, Germany
  \and
  Torsten Schaub\thanks{Also affiliated with the
                              School of Computing Science at
                              Simon Fraser University,
                              Burnaby, Canada,
                              and the
                              Institute for Integrated and Intelligent Systems
                              at
                              Griffith University,
                              Brisbane, Australia.}
  \\
  INRIA Rennes, France, and
  University of Potsdam, Germany
  }

\submitted{[n/a]}
\revised{[n/a]}
\accepted{[n/a]}

\newtheorem{proposition}{Proposition}

\begin{document}

\maketitle

\begin{abstract}
We introduce a new flexible paradigm of grounding and solving in Answer Set Programming (ASP),
which we refer to as multi-shot ASP solving,
and present its implementation in the ASP system \clingo.

Multi-shot ASP solving features grounding and solving processes that deal with
continuously changing logic programs.
In doing so, they remain operative and accommodate changes in a seamless way.
For instance,
such processes allow for advanced forms of search, as in optimization or theory solving,
or interaction with an environment, as in robotics or query-answering.
Common to them is that the problem specification evolves during the reasoning process,
either because data or constraints are added, deleted, or replaced.
This evolutionary aspect adds another dimension to ASP since it brings about state changing operations.
We address this issue by providing an operational semantics that characterizes grounding and solving processes
in multi-shot ASP solving.
This characterization provides a semantic account of grounder and solver states along with the operations manipulating them.

The operative nature of multi-shot solving avoids redundancies in relaunching grounder and solver programs and benefits from the solver's
learning capacities.
\clingo{} accomplishes this by complementing ASP's declarative input language with control capacities.
On the declarative side, a new directive allows for structuring logic programs into named and parameterizable subprograms.
The grounding and integration of these subprograms into the solving process is completely modular and fully controllable from the procedural side.
To this end,
\clingo{} offers a new application programming interface that is conveniently accessible via scripting languages.
By strictly separating logic and control,
\clingo{} also abolishes the need for dedicated systems for incremental and reactive reasoning, like \iclingo{} and \oclingo, respectively,
and its flexibility goes well beyond the advanced yet still rigid solving processes
of the latter.

\medskip\noindent
{\em Under consideration for publication in Theory and Practice of Logic Programming (TPLP)}
\end{abstract}
%
%%% Local Variables:
%%% mode: latex
%%% TeX-master: "paper"
%%% End:

\section{Introduction}\label{sec:introduction}

Standard Answer Set Programming (ASP; \cite{baral02a}) follows a one-shot process in computing stable models of logic programs.
This view is best reflected by the input/output behavior of monolithic ASP systems like \dlv~\cite{dlv03a} and (original) \clingo~\cite{gekakaosscsc11a}
that take a logic program and output its stable models.
Internally, however, both follow a fixed two-step process.
First, a grounder generates a (finite) propositional representation of the input program.
Then, a solver computes the stable models of the propositional program.
This rigid process stays unchanged when grounding and solving with separate systems.
In fact, up to series~3, \clingo{} was a mere combination of the grounder \gringo{} and the solver \clasp.
Although more elaborate reasoning processes are performed by the extended systems \iclingo~\cite{gekakaosscth08a} and \oclingo~\cite{gegrkasc11a} for incremental and reactive reasoning, respectively,
they also follow a pre-defined control loop evading any user control.
Beyond this, however,
there is substantial need for specifying flexible reasoning processes,
for instance, when it comes to
interactions with an environment (as in assisted living, robotics, or with users),
advanced search (as in multi-objective optimization, planning, theory solving, or heuristic search),
or recurrent  query answering (as in hardware analysis and testing or stream processing).
Common to all these advanced forms of reasoning is that the problem specification evolves during the respective reasoning processes,
either because data or constraints are added, deleted, or replaced.

For mastering such complex reasoning processes,
we propose the paradigm of \emph{multi-shot ASP solving}
in order to deal with continuously changing logic programs.
In contrast to the traditional single-shot approach, 
where an ASP system takes a logic program, computes its answer sets, and exits,
the idea is to consider evolving grounding and solving processes.
Such processes lead to operative ASP systems that possess an internal state that can be manipulated by certain operations.
Such operations allow for adding, grounding, and solving logic programs as well as
setting truth values of (external) atoms.
The latter does not only provide a simple means for incorporating external input but
also for enabling or disabling parts of the current logic program.
These functionalities allow for dealing with changing logic programs in a seamless way.
To capture multi-shot solving processes,
we introduce an operational semantics centered upon a formal definition of an ASP system state along with its (state changing) operations.
Such a state reflects the relevant information gathered in the system's grounder and solver components.
This includes 
(i) a collection of non-ground logic programs subject to grounding,
(ii) the ground logic programs currently held by the solver,
(iii) and a truth assignment of externally defined atoms.
Changing such a state brings about several challenges absent in the single-shot case,
among them, contextual grounding and logic program composition.

Given that the theoretical foundations of multi-shot solving are a means to an end,
we interleave their presentation with the corresponding features of ASP system \clingo.
This new generation of \clingo%
\footnote{Multi-shot solving was introduced with the \clingo~4 series by~\citeN{gekakasc14b}. 
  However, we describe its functionalities in the context of the current \clingo~5 series.
  The advance from series~4 to~5 only smoothed some multi-shot related interfaces but left the principal functionality unaffected.}
offers novel high-level constructs for realizing multi-shot ASP solving.
This is achieved within a single ASP grounding and solving process
that
avoids redundancies otherwise caused by relaunching grounder and solver programs and 
benefits from the learning capacities of modern ASP solvers.
To this end,
\clingo{} complements ASP's declarative input language by manifold control capacities.
The latter are provided by an imperative application programming interface (API) implemented in C.
Corresponding bindings for \python{} and \lua{} are available and can also be embedded into \clingo's input language 
(via the \lstinline{#script} directive).
On the declarative side, 
\clingo{} offers a new directive \lstinline{#program} that allows for structuring logic programs into named and parametrizable subprograms.
The grounding and integration of these subprograms into the solving process is completely modular and fully controllable from the procedural side.
For exercising control,
the latter benefits from a dedicated library furnished by \clingo's API, which does not only expose grounding and solving functions
but moreover allows for continuously assembling the solver's program.
This can be done in combination with externally controllable atoms that allow for enabling or disabling rules.
Such atoms are declared by the \lstinline{#external} directive.
Hence, by strictly separating logic and control,
\clingo{} abolishes the need for special-purpose systems for incremental and reactive reasoning, like \iclingo{} and \oclingo, respectively,
and its flexibility goes well beyond the advanced yet still rigid grounding and solving processes
of such systems.
In fact, \clingo's multi-shot solving capabilities rather enable users to engineer novel forms of reasoning,
as we demonstrate by four case studies.

The rest of the paper is organized as follows.
Section~\ref{sec:background} provides a brief account of formal preliminaries.
Section~\ref{sec:glance} gives an informal overview on the new features of \clingo{} in order to pave the way for their formal underpinnings presented
in Section~\ref{sec:approach}.
There, we lay the formal foundations of multi-shot solving and present its aforementioned operational semantics.
In Section~\ref{sec:practice}, we illustrate the power of multi-shot ASP solving in several use cases and highlight some features of interest.
We use \python{} throughout the paper to illustrate the multi-shot functionalities of \clingo's API.
Further API-related aspects are described in Section~\ref{sec:api}.
Section~\ref{sec:experiments} gives an empirical analysis of some selected features of multi-shot solving with \clingo.
Finally, we relate our approach to the literature in Section~\ref{sec:related:work} before we conclude in Section~\ref{sec:discussion}.

%%% Local Variables: 
%%% mode: latex
%%% TeX-master: "paper"
%%% End: 

\section{Formal preliminaries}\label{sec:background}

A (normal\footnote{For the sake of simplicity, we confine our formal elaboration to normal logic programs.
\par Multi-shot solving with \clingo{} also works with disjunctive logic programs.}) \emph{rule}~$r$ is an expression of the form
\[
a_0\leftarrow a_1,\dots,a_m,\naf{a_{m+1}},\dots,\naf{a_n}
,
\]
where $a_i$, for $0\leq m\leq n$,
is an \emph{atom} of the form $p(t_1,\dots,t_k)$, 
$p$ is a predicate symbol of arity $k$, also written as $p/k$,
and
$t_1,\dots,t_k$ are terms,
built from constants, variables, and functions.
Letting
$\head{r}=a_0$,
$\body{r}=\{a_1,\dots,a_m,\naf{a_{m+1}},\dots,\naf{a_n}\}$,
$\pbody{r}=\{a_1,\dots,a_m\}$, and
$\nbody{r}=\{a_{m+1},\dots,a_n\}$,
we also denote~$r$ by
\(
\head{r}\leftarrow\body{r} % \pbody{r}\cup\{\naf{a}\mid a\in\nbody{r}\}
\).
A rule is called fact, whenever $\body{r}=\emptyset$.
A (normal) \emph{logic program}~\PRG{} is a set (or list) of rules
(depending on whether the order of rules matters or not).
We write
$\Head{\PRG}=\{\head{r}\mid r\in\PRG\}$
and
$\atom{\PRG}=\Head{\PRG}\cup\bigcup_{r\in \PRG}(\pbody{r}\cup\nbody{r})$
to denote the set of all head atoms and atoms occurring in~\PRG, respectively.
A term, atom, rule, or program is \emph{ground} if it does not contain variables.

We denote the set of all ground terms constructible from
constants (including all integers) and function symbols by $\mathcal{T}\!$,
and let $\mathcal{C}$ stand for the subset of symbolic (i.e.~non-integer) constants.
The \emph{ground instance} of~\PRG, denoted by $\ground{\PRG}$,
is the set of all ground rules constructible from rules $r\in \PRG$
by substituting every variable in~$r$ with some element of $\mathcal{T}$.
We associate~\PRG{} with its
\emph{positive atom dependency graph}
\[
\dep{\PRG}=
(\atom{\ground{\PRG}},
 \{(a_0,a)\mid r\in\ground{\PRG},\head{r}=a_0,a\in\pbody{r}\})
\]
and call a maximal non-empty subset of $\atom{\ground{\PRG}}$
inducing a strongly connected subgraph\footnote{That is, each pair of atoms in the subgraph is connected by a path.}
of~$\dep{\PRG}$ a strongly connected component of~\PRG.

A set~$X$ of ground atoms is a \emph{model} of~\PRG, if
$\head{r}\in X$, $\pbody{r}\nsubseteq X$, or $\nbody{r}\cap X\neq\emptyset$
holds for every $r\in\ground{\PRG}$.
Moreover, 
$X$ is a \emph{stable model} of~\PRG{}~\cite{gellif88b}, 
if $X$ is a $\subseteq$-minimal model of
$\{\head{r}\leftarrow\pbody{r} \mid r\in\ground{\PRG},\linebreak[1]\nbody{r}\cap X=\nolinebreak\emptyset\}$.

Following~\citeN{oikjan06a},
a \emph{module}~$\module{\PRG}$ is a triple
\(
(\PRG,I,O)
\)
consisting of
a ground logic program~\PRG\
along with sets~$I$ and~$O$
of ground \emph{input} and \emph{output}
atoms such that
\begin{enumerate}
\item $I\cap\nolinebreak O=\nolinebreak\emptyset$,
\item $\atom{\PRG}\subseteq I\cup O$, and
\item $\Head{\PRG}\subseteq O$.
\end{enumerate}
We also denote the constituents of $\module{\PRG}=(\PRG,I,O)$ by
$\prog{\module{\PRG}}=\nolinebreak\PRG$,
$\inp{\module{\PRG}}=\nolinebreak I$, and 
$\out{\module{\PRG}}=\nolinebreak O$.
A set~$X$ of ground atoms is a \emph{stable model} of a module~$\mathbb{\PRG}$,
if $X$ is a (standard) stable model of
\(
\prog{\module{\PRG}}\cup\{{a\leftarrow} \mid a\in \inp{\module{\PRG}}\cap X\}
\).

Two modules~$\module{\PRG}_1$ and~$\module{\PRG}_2$ are
\emph{compositional}, if
\begin{enumerate}
\item $\out{\module{\PRG}_1}\cap\out{\module{\PRG}_2}=\emptyset$
and
\item $\out{\module{\PRG}_1}\cap C=\emptyset$ or $\out{\module{\PRG}_2}\cap C=\emptyset$
\\for every strongly connected component~$C$ of $\prog{\module{\PRG}_1}\cup\prog{\module{\PRG}_2}$.
\end{enumerate}
In other words,
all rules defining an atom must belong to the same module. 
And any positive recursion is within modules;
no positive recursion is allowed among modules.

Provided that $\module{\PRG}_1$ and~$\module{\PRG}_2$
are compositional, their \emph{join} is defined as the module
\[
\module{\PRG}_1\sqcup\module{\PRG}_2
=
(
\prog{\module{\PRG}_1}\cup\prog{\module{\PRG}_2},\linebreak[1]
(\inp{\module{\PRG}_1}\setminus\out{\module{\PRG}_2})\cup
(\inp{\module{\PRG}_2}\setminus\out{\module{\PRG}_1}),\linebreak[1]
\out{\module{\PRG}_1}\cup\out{\module{\PRG}_2}
)
\ .
\]
The module theorem~\cite{oikjan06a} shows that a set~$X$ of ground
atoms is a stable model of $\module{\PRG}_1\sqcup\module{\PRG}_2$ iff
$X=X_1\cup X_2$ for stable models~$X_1$ and~$X_2$
of~$\module{\PRG}_1$ and~$\module{\PRG}_2$, respectively,
such that
$X_1\cap(\inp{\module{\PRG}_2}\cup\out{\module{\PRG}_2})=
 X_2\cap(\inp{\module{\PRG}_1}\cup\out{\module{\PRG}_1})$.%
\footnote{Note that the module theorem is strictly stronger than the splitting set theorem \cite{liftur94a}.
  For instance, there is no non-trivial splitting set of $\module{\PRG}_1\sqcup\module{\PRG}_2$,
  since neither $\{a\}$ nor $\{b\}$ is one.}
For example,
the modules 
\(
\module{\PRG}_1
=
(\{a\leftarrow\naf{c};\ c\leftarrow\naf{b}\},\{b\},\{a,c\})
\)
and
\(
\module{\PRG}_2
=
(\{b\leftarrow a\},\{a\},\{b\})
\)
are compositional,
and combining their stable models,
$\{a,b\}$ and $\{c\}$ for $\module{\PRG}_1$ as well as
$\{a,b\}$ and~$\emptyset$ for $\module{\PRG}_2$,
yields the stable models $\{a,b\}$ and $\{c\}$ of
\(
\module{\PRG}_1\sqcup\module{\PRG}_2
=
(\prog{\module{\PRG}_1}\cup\prog{\module{\PRG}_2},
 \emptyset,
 \{a,b,c\})
\).
Unlike that,
\(
\module{\PRG}_1'
=
(\{a\leftarrow b;\ c\leftarrow\naf{a}\},\{b\},\{a,c\})
\)
and 
$\module{\PRG}_2$
are not compositional
because the strongly connected component $\{a,b\}$
of $\prog{\module{\PRG}_1'}\cup\prog{\module{\PRG}_2}$
includes $a\in\out{\module{\PRG}_1'}$ and $b\in\out{\module{\PRG}_2}$.
Moreover, $\{a,b\}$ is a stable model of $\module{\PRG}_1'$ and 
$\module{\PRG}_2$, but not of $\prog{\module{\PRG}_1'}\cup\prog{\module{\PRG}_2}$.

% \comment{T: decide ultimate notation $v$ vs $V$}
An assignment $v$ over a set $A$ of ground atoms is a function $v: A\rightarrow\{t,f,u\}$
whose range consists of truth values, standing for true, false, and undefined.
Given an assignment $v$, we define the sets
$V^x=\{a\in A\mid v(a)=x\}$ for $x\in\{t,f,u\}$.
In what follows,
we represent partial assignments like $v$ either by $(V^t,V^f)$ or $(V^t,V^u)$
by leaving the respective variables with default values implicit.

Finally,
we use typewriter font and symbols \lstinline{:-} and \lstinline{not} instead of $\leftarrow$ and $\sim$, respectively,
whenever we deal with source code accepted by \clingo.
In such a case, we also make use of extended language constructs like integrity or cardinality constraints,
all of which can be reduced to normal logic programs, as detailed in the literature (cf.~\cite{siniso02a}).

%%% Local Variables: 
%%% mode: latex
%%% TeX-master: "paper"
%%% End: 

\section{Multi-shot solving with \clingo{} at a glance}
\label{sec:glance}

Let us begin with an informal overview of the central features and corresponding language constructs of \clingo's multi-shot solving capacities.

A key feature, distinguishing \clingo{} from its predecessors,
is the possibility to structure (non-ground) input rules into subprograms.
To this end,
a program can be partitioned into several subprograms by means of the directive \lstinline{#program};
it comes with a name and an optional list of parameters.
Once given in the input,
the directive gathers all rules up to the next such directive (or the end of file)
within a subprogram identified by the supplied name and parameter list.
As an example,
two subprograms \lstinline{base} and \lstinline{acid(k)} can be specified as follows:
\begin{lstlisting}[caption={Logic program with \texttt{\#program} declarations},label={lst:example:program},language=clingo]
a(1).
#program acid(k).
b(k).
c(X,k) :- a(X).
#program base.
a(2).
\end{lstlisting}
Note that \lstinline{base} is a dedicated subprogram (with an empty parameter list):
in addition to the rules in its scope,
it gathers all rules not preceded by any \lstinline{#program} directive.
Hence, in the above example, the \lstinline{base} subprogram includes the facts \lstinline{a(1)} and \lstinline{a(2)},
although, only the latter is in the actual scope of the directive in Line~5.
Without further control instructions (see below),
\clingo{} grounds and solves the \lstinline{base} subprogram only,
essentially, yielding the standard behavior of ASP systems.
The processing of other subprograms such as \lstinline{acid(k)}
% with the schematic rules \lstinline{b(k)} and `\lstinline{c(X,k) :- a(X)}',
is subject to explicitly given control instructions.

For such customized control over grounding and solving,
a \lstinline{main} routine
(taking a control object representing the state of \clingo{} as argument, here \lstinline{prg})
can be supplied.
For illustration, let us consider two \python{} \lstinline{main} routines:%
\footnote{The \lstinline{ground} routine takes a list of pairs as argument.
  Each such pair consists of a subprogram name (e.g.\ \lstinline{base} or \lstinline{acid}) and a list of actual parameters (e.g.\ \lstinline{[]} or \lstinline{[42]}).}
\\
\hspace*{2\parindent}%
\begin{minipage}{0.5\linewidth}
\begin{lstlisting}[firstnumber=7,language=clingo]
#script(python)
def main(prg):
  prg.ground([("base",[])])
  prg.solve()
#end.
\end{lstlisting}  
\end{minipage}
\begin{minipage}{0.5\linewidth}
\begin{lstlisting}[firstnumber=7,language=clingo]
#script(python)
def main(prg):
  prg.ground([("acid",[42])])
  prg.solve()
#end.
\end{lstlisting}  
\end{minipage}
\\
While the control program on the left matches the default behavior of \clingo,
the one on the right ignores all rules in the \lstinline{base} program but rather
contains a \lstinline{ground} instruction for \lstinline{acid(k)} in Line~8,
where the parameter~\lstinline{k} is to be instantiated with the term \lstinline{42}.
Accordingly, the schematic fact \lstinline{b(k)} is turned into \lstinline{b(42)},
no ground rule is obtained from `\lstinline{c(X,k) :- a(X)}' due to lacking instances of \lstinline{a(X)},
and the \lstinline{solve} command in Line~10 yields a stable model consisting of
\lstinline{b(42)} only.
Note that \lstinline{ground} instructions apply to the subprograms
given as arguments,
while \lstinline{solve} triggers reasoning w.r.t.\ all accumulated ground rules.

In order to accomplish more elaborate reasoning processes,
like those of \iclingo{} and \oclingo{} or other customized ones,
it is indispensable to activate or deactivate ground rules on demand.
For instance, former initial or goal state conditions need to be
relaxed or completely replaced when modifying a planning problem, e.g.,
by extending its horizon.
While the predecessors of \clingo{} relied on the \lstinline{#volatile} directive
to provide a rigid mechanism for the expiration of transient rules,
\clingo{} captures the respective functionalities and customizations
thereof in terms of the \lstinline{#external} directive.
The latter goes back to \lparse~\cite{lparseManual} and was also
supported by \clingo's predecessors to exempt (input) atoms
from simplifications (and fixing them to false).
As detailed in the following,
the \lstinline{#external} directive of \clingo{} provides a generalization
that, in particular, allows for a flexible handling of yet undefined atoms.

For continuously assembling ground rules evolving at different stages of a
reasoning process, \lstinline{#external} directives declare atoms that may
still be defined by rules added later on.
In terms of module theory,
such atoms correspond to inputs, which (unlike undefined output atoms) must not be simplified.
For declaring input atoms,
\clingo{} supports schematic \lstinline{#external} directives that are instantiated along with
the rules of their respective subprograms.
To this end, a directive like
\begin{lstlisting}[numbers=none,language=clingo]
#external p(X,Y) : q(X,Z), r(Z,Y).
\end{lstlisting}
is treated similar to a rule
`\lstinline{p(X,Y) :- q(X,Z), r(Z,Y)}'
during grounding.
However, the head atoms of the resulting ground instances are merely collected as inputs,
whereas the ground rules as such are discarded.

Once grounded, the truth value of external atoms can be changed via the \clingo{} API
(until the atoms becomes defined by corresponding rules).
By default, the initial truth value of external atoms is set to false.
For example, with \clingo's \python{} API,
\lstinline{assign_external(self,p(a,b),True)}%
\footnote{%
In order to construct atoms, symbolic terms, or function terms, respectively, the \clingo{} API function \lstinline{Function} has to be used.
Hence, the expression \lstinline{p(a,b)} actually stands for \lstinline{Function("p", [Function("a"), Function("b")])}.}
can be used to set the truth value of the external atom \lstinline{p(a,b)} to true.
Among others,
this can be used to activate and deactivate rules in logic programs.
For instance,
the integrity constraint `\lstinline{:- q(a,c), r(c,b), p(a,b)}' is ineffective whenever \lstinline{p(a,b)} is false.

%%% Local Variables:
%%% mode: latex
%%% TeX-master: "paper"
%%% End:

\section{Multi-shot solving}
\label{sec:approach}

Having set the practical stage in the previous section,
let us now turn to posing the formal foundations of multi-shot ASP solving.
We begin with a characterization of grounding subprograms with external directives
in the context of previously grounded subprograms.
This provides us with a formal account of the interplay of \lstinline{ground} routines with
\lstinline{#program} and \lstinline{#external} directives.
Next, we show how module theory can be used for characterizing the composition of ground subprograms during multi-shot solving.
This gives us a precise idea on the successive logic programs contained in the ASP solver at each invocation of \lstinline{solve}.
Finally,
all this culminates in an operational semantics for multi-shot solving in terms of state-changing operations.

The concepts introduced in this section are mainly illustrated by succinct, technical examples.
More illustration is provided in the next section discussing several use cases in detail.

\subsection{Parameterizable subprograms}
\label{sec:programs:parametrizable}

A \emph{program declaration} is of form
\begin{equation}\label{eq:program:declaration}
  \texttt{\#program }n(p_1,\dots,p_k)
\end{equation}
where $n,p_1,\dots,p_k$ are symbolic constants.
We call $n$ the name of the declaration and $p_1,\dots,p_k$ its parameters.
For simplicity,
we suppose that different occurrences of program declarations with the same name also share the same parameters
(although this is not required by \clingo).
In this way, each name is associated with a unique parameter specification.

The \emph{scope} of a program declaration in a list of rules and declarations
consists of the set of all rules and non-program declarations 
following the directive up to the next program declaration or the end of the list.%
\footnote{That is, the end of file in practice.}
In Listing~\ref{lst:example:program},
the scope of the declaration in Line~2 consists of \lstinline{b(k)} and `\lstinline{c(X,k) :- a(X)}',
while that in Line~5 contains \lstinline{a(2)}.
Given a list $\RRG$ of (non-ground) rules and declarations along with a non-integer constant $n$,
we define $\RRG(n)$ as the set of all (non-ground) rules and (non-program) declarations
in the scope of all occurrences of program declarations with name $n$.
We often refer to $\RRG(n)$ as a \emph{subprogram} of $\RRG$.
All rules and non-program declarations outside the scope of any (explicit) program declaration
are thought of being implicitly preceded by a `\lstinline{#program base}' declaration.
Hence,
if $R$ consists of Line~1--6 above,
we get\footnote{We drop the typewriter font, whenever our emphasis shifts to a more formal context.}
\(
\RRG(\texttt{base})=\{a(1)\leftarrow{},a(2)\leftarrow{}\}
\)
and
\(
\RRG(\texttt{acid})=\{b(k)\leftarrow{}, c(X,k)\leftarrow a(X)\}
\).
Each such list $\RRG$ induces a collection $(\RRG(c))_{c\in\mathcal{C}}$ of (non-disjoint) subprograms (most of which are empty).
For example, all subprograms obtained from Line~1--6 are empty,
except for \lstinline{base} and \lstinline{acid}.

Given a name $n$ with associated parameters $p_1,\dots,p_k$,
the instantiation of subprogram $\RRG(n)$ with terms $t_1,\dots,t_k$
results in the set $\RRG(n)[p_1/t_1,\dots,p_k/t_k]$,
obtained by replacing in $\RRG(n)$ each occurrence of $p_i$ by $t_i$ for $1\leq i\leq k$.%
\footnote{\clingo{} uses a more general instantiation process involving unification and arithmetic evaluation; see~\cite{PotasscoUserGuide} for details.}
For instance, $\RRG(\texttt{acid})[\mathtt{k}/\mathtt{42}]$ consists of
\lstinline{b(42)} and `\lstinline{c(X,42) :- a(X)}'.

\subsection{Contextual grounding}\label{sec:contextual:grounding}
\label{sec:grounding:contextual}

The definition of a program's ground instance \ground{\PRG} depends on \PRG{} and its underlying set of terms.
For instance,
grounding accordingly program 
\(
\RRG(\texttt{base})\cup\RRG(\texttt{acid})[k/42]
\)
% viz.\
% \[
% \{a(1)\leftarrow{},a(2)\leftarrow{}\}
% \cup
% \{b(42)\leftarrow{}, c(X,42)\leftarrow a(X)\}
% \]
yields
\begin{align}\label{ex:grounding}
\{a(1)\leftarrow{},a(2)\leftarrow{}\}
\cup
\{b(42)\leftarrow{}\}
\cup
\left\{
\begin{array}{rcl}
   c(1,42)&\leftarrow& a(1),\\
   c(2,42)&\leftarrow& a(2), \\
  c(42,42)&\leftarrow& a(42)
\end{array}
\right\}
\end{align}
In practice,%
\footnote{In one-shot grounding, a program is partitioned via the strongly connected components of its dependency graph.}
however,
rules are grounded relative to a set of atoms, which we refer to as an \emph{atom base}.
In our example, this avoids the generation of the irrelevant rule `$c(42,42)\leftarrow a(42)$'.
To see this, note that the relevance of instances of `$c(X,42)\leftarrow a(X)$' depends upon the available ground atoms of predicate $a/1$.
Now, grounding first \RRG(\texttt{base}) establishes --- in addition to \ground{\RRG(\texttt{base})} --- the atom base
\(
\{a(1),a(2)\}
\).
Grounding then $\RRG(\texttt{acid})[k/42]$ relative to this atom base,
yields
\begin{align}\label{ex:contextual:grounding}
\{b(42)\leftarrow{}\}
\cup
\left\{
\begin{array}{rcl}
   c(1,42)&\leftarrow& a(1),\\
   c(2,42)&\leftarrow& a(2)
\end{array}
\right\}
\end{align}
because only rule instances are created if their positive body literals 
either belong to the atom base 
or are derivable through other rules instances.%
\footnote{This is a simplification of semi-naive database evaluation~\cite{abhuvi95a}, used in ASP grounding components~\cite{kalepesc16a}.
  Notably, this technique allows for dealing with recursive function symbols
  and guarantees termination for a wide class of programs, \PRG,
  even though their ground instantiation \ground{\PRG} is infinite.}
Hence, rule $c(42,42)\leftarrow a(42)$ is dropped.
This is made precise in view of our purpose in the following definition.

Given a set \RRG\ of (non-ground) rules and two sets $C,D$ of ground atoms,
we define an instantiation of \RRG\ relative to atom base $C$ as a ground program \cground{\RRG}{C} over atom base $D$
subject to the following conditions:
\begin{align}
  \label{grd:C}
  % C\setminus \Head{&\ground{\RRG}}\subseteq D \subseteq C \cup\Head{\cground{\RRG}{C}}
  D & = C \cup\Head{\cground{\RRG}{C}}
  \\
  \label{grd:G}
  \cground{\RRG}{C}&{} \subseteq
  \{
  \head{r}\leftarrow\pbody{r}\cup \{\naf{a} \mid
  {} \begin{array}[t]{@{}r@{}l@{}}
  a\in{}& \nbody{r}\cap D\}
  \mid
  {} \\
  r\in{}& \ground{\RRG},
  \pbody{r}\cup\{\head{r}\}\subseteq D
  \}
  \end{array}
  \\
  \label{grd:EQ}
  \cground{\RRG}{C}&{}\cup\QRG
  \text{ and }
  \ground{\RRG}\cup\QRG
  \text{ have the same stable models} % for }
%  Q=\{\{a\}\leftarrow{}\mid a\in C\setminus\Head{\ground{\RRG}}\}
\end{align}
where
$Q=\{{\{a\}\leftarrow}{} \mid a\in C\setminus\Head{\ground{\RRG}}\}$.%
\footnote{A \emph{choice rule}~\cite{siniso02a} of the form $\{a\}\leftarrow{}$ corresponds to
  (normal) rules $a\leftarrow \naf{a'}$ and $a'\leftarrow \naf{a}$, where $a'$ is a fresh atom.}

Atom base $D$ gives the extension of $C$ obtained by grounding $R$ relative to $C$.
To this end,
Condition~\eqref{grd:C} limits~$D$ to atoms either belonging to $C$ or emerging as heads of rules in $\cground{\RRG}{C}$,
while Condition~\eqref{grd:G} projects \ground{P} to $D$ by reducing its rules to those parts possibly relevant to stable models, 
as expressed in~\eqref{grd:EQ}.
The exact scope of the obtained atom base $D$ as well as \cground{\RRG}{C} are left open
to account for potential simplifications during grounding.%
\footnote{%
  In fact, the instantiation process of \clingo{} iteratively extends an atom base by heads of rules whose positive body
  atoms are contained in it.
  Moreover, the finite instantiation of a program like $\{a(9);\ b(1);\ b(X{+}1) \leftarrow b(X),\naf{a(X)}\}$
  relies on the evaluation of $\naf{a(X)}$ during grounding, allowing \clingo{} to stop
  the successive generation of rule instances at $b(10) \leftarrow b(9),\naf{a(9)}$.}
The correctness of the specific scope is warranted by Condition~\eqref{grd:EQ}

We capture the composition of ground programs below in Section~\ref{sec:programs:composition} in terms of modules.
To this end,
note that the choices in $Q$ mimic the role of input atoms of modules.
In fact, Condition~\eqref{grd:EQ} is equivalent to requiring that the modules
\(
(\cground{\RRG}{C},\linebreak[1]C\setminus\Head{\ground{\RRG}},\linebreak[1]\Head{\ground{\RRG}})
\)
and
\(
(\ground{\RRG},\linebreak[1]C\setminus\Head{\ground{\RRG}},\linebreak[1]\Head{\ground{\RRG}})
\)
have the same stable models.

Resuming our example,
we see that
\(
\cground{\RRG(\texttt{acid})[k/42]}{\{a(1),a(2)\}}
\)
corresponds to the rules in~\eqref{ex:contextual:grounding}.
Together with 
\(
\cground{\RRG(\texttt{base})}{\emptyset}
\),
the resulting program is equivalent to
\(
\ground{\RRG(\texttt{base})\cup\RRG(\texttt{acid})[k/42]}
\),
viz.\ the rules in~\eqref{ex:grounding}.

For further illustration, consider
\(
R= \{ a(X) \leftarrow f(X), e(X);\ b(X) \leftarrow f(X), \naf{e(X)} \}
\)
along with
\(
C=\{f(1),f(2),e(1)\}
\).
Focusing on relevant (parts of) rule instances relative to~$C$ leads to
\[
\cground{R}{C}
=
\left\{
  \begin{array}{ll}
    a(1) \leftarrow f(1), e(1) & b(1) \leftarrow f(1), \naf{e(1)}\\
                               & b(2) \leftarrow f(2)
  \end{array}
\right\}
\]
over $D=C\cup\{a(1),b(1),b(2)\}$.
In particular, note that an inapplicable rule instance including $e(2)$ is dropped,
while $\naf{e(2)}$ is simplified away to thus obtain the last of the above ground rules.

Although $\cground{\RRG}{C}$ reflects a potential simplification of the full ground program $\ground{\RRG}$,
both can likewise be augmented with additional rules.
Moreover, the restriction of the resulting atom base preserves equivalence.
The next result makes this precise in terms of module theory.
%
% ------------------------------------------------------------
\begin{proposition}
Let \RRG\ be a set of (non-ground) rules,
let $\cground{\RRG}{C}$ be an instantiation of \RRG\ relative to
a set~$C$ of ground atoms, and
let \module{\PRG} be a module.

If \module{\PRG} and
\(
(\ground{\RRG},\linebreak[1]C\setminus\Head{\ground{\RRG}},\linebreak[1]\Head{\ground{\RRG}})
\)
are compositional, then \module{\PRG} and
\(
(\cground{\RRG}{C},\linebreak[1]C\setminus\Head{\ground{\RRG}},\linebreak[1]\Head{\ground{\RRG}})
\)
are compositional as well,
where
\(
 \module{\PRG}\sqcup
 (\ground{\RRG},\linebreak[1]C\setminus\Head{\ground{\RRG}},\linebreak[1]\Head{\ground{\RRG}})
\)
and
\(
 \module{\PRG}\sqcup
 (\cground{\RRG}{C},\linebreak[1]C\setminus\Head{\ground{\RRG}},\linebreak[1]\Head{\ground{\RRG}})
\)
have the same stable models.
\end{proposition}
% ------------------------------------------------------------

\begin{proof}
Assume that
\module{\PRG} and
\(
(\ground{\RRG},\linebreak[1]C\setminus\Head{\ground{\RRG}},\linebreak[1]\Head{\ground{\RRG}})
\)
are compositional.
By the construction of $\cground{\RRG}{C}$ in~\eqref{grd:G},
$\dep{\cground{\RRG}{C}}$ is a subgraph of $\dep{\ground{\RRG}}$,
which implies that \module{\PRG} and
\(
(\cground{\RRG}{C},\linebreak[1]C\setminus\Head{\ground{\RRG}},\linebreak[1]\Head{\ground{\RRG}})
\)
are compositional as well.
As
\(
(\ground{\RRG},\linebreak[1]C\setminus\Head{\ground{\RRG}},\linebreak[1]\Head{\ground{\RRG}})
\)
and
\(
(\cground{\RRG}{C},\linebreak[1]C\setminus\Head{\ground{\RRG}},\linebreak[1]\Head{\ground{\RRG}})
\)
are equivalent by the condition in~\eqref{grd:EQ},
the module theorem~\cite{oikjan06a} yields that
\(
 \module{\PRG}\sqcup
 (\ground{\RRG},\linebreak[1]C\setminus\Head{\ground{\RRG}},\linebreak[1]\Head{\ground{\RRG}})
\)
and
\(
 \module{\PRG}\sqcup
 (\cground{\RRG}{C},\linebreak[1]C\setminus\Head{\ground{\RRG}},\linebreak[1]\Head{\ground{\RRG}})
\)
have the same stable models.
\end{proof}

More illustration of contextual grounding is given throughout the following sections.

\subsection{Extensible logic programs}
\label{sec:programs:extensible}

We define a (non-ground) logic program $\RRG$ as \emph{extensible}, if it contains
some (non-ground) \emph{external declaration} of the form
\begin{equation}\label{eq:external:declaration}
  \text{\lstinline{#external}}~a : B
\end{equation}
where $a$ is an atom and $B$ a rule body.

As an example,
consider the extensible program in Listing~\ref{lst:example:external}.
% ------------------------------------------------------------
\begin{lstlisting}[caption={Extensible logic program},label={lst:example:external},language=clingo]
#external e(X) : f(X), X < 2.
f(1..2).
a(X) :- f(X), e(X).
b(X) :- f(X), not e(X).
\end{lstlisting}
% ------------------------------------------------------------

For grounding an external declaration as in \eqref{eq:external:declaration}, we treat it as a rule
\(
a \leftarrow B,\varepsilon
\)
where $\varepsilon$ is a distinguished ground atom marking rules from \lstinline{#external} declarations.
Formally,
given an extensible program $\RRG$, we define the collection $\QRG$ of rules corresponding to \lstinline{#external} declarations as follows.
\begin{align*}
  \QRG & {} = \{a\leftarrow B,\varepsilon\mid({\texttt{\#external}}~a : B)\in \RRG\}\\
  \RRG'& {} = \{a\leftarrow B\in\RRG\}
\end{align*}
With these, the ground instantiation % $\ground{\RRG\cup \QRG}$
of an extensible logic program $\RRG$ relative to an atom base $C$ is defined as a ground logic program~$\PRG$ associated with a set~$E$ of ground atoms,
where
\begin{align}
\label{grd:P}  \PRG & {} = \{r\in\cground{\RRG'\cup \QRG}{C\cup\{\varepsilon\}} \mid \varepsilon\notin\body{r}\}\\
\label{grd:E}  E & {} = \{\head{r}\mid r\in\cground{\RRG'\cup \QRG}{C\cup\{\varepsilon\}}, \varepsilon\in\body{r}\}
\end{align}
For simplicity, we refer to $\PRG$ and $E$ as a \emph{logic program with externals},
and we drop the reference to $\RRG$ and $C$ whenever clear from the context.
Note that after grounding, the special atom $\varepsilon$ appears neither in $\PRG$ nor $E$.
In fact, $\PRG$ is a logic program over $C\cup E\cup\Head{\PRG}$.

Given the set $E$ of externals,
$\PRG$ can alternatively be defined as $\cground{\RRG'}{C\cup E}$.

Grounding the program in Listing~\ref{lst:example:external} relative to an empty atom base yields
the below program with a single external atom, viz.~\lstinline{e(1)}\/:
\begin{lstlisting}[language=clingo]
f(1). f(2).
a(1) :- f(1), e(1).
b(1) :- f(1), not e(1).
b(2) :- f(2).
\end{lstlisting}
Note how externals influence the result of grounding.
While occurrences of \lstinline{e(1)} remain untouched, the atom \lstinline{e(2)} is unavailable and thus set to false according to the condition in~\eqref{grd:G}.
In practice, even more simplifications are applied during grounding.
For instance, in \clingo,
the established truth of \lstinline{f(1)} and \lstinline{f(2)} leads to their removal in the bodies in Line~2--4.

Logic programs with externals constitute a major building block of multi-shot solving.
Hence, before addressing their composition within a more elaborate formal framework,
let us provide some semantic underpinnings.
The stable models of such programs are defined relative to a truth assignment on the external atoms.
For a program $\PRG$ with externals $E$,
we define the set $I = E\setminus\Head{\PRG}$ as input atoms of $\PRG$.
That is, input atoms are externals that are not overridden by rules in $\PRG$.
Then, given $\PRG$ along with a partial assignment $V=( V^t,V^u )$ over $I$,
we define the stable models of $\PRG$ w.r.t.\ $V$ as the ones of
\(
\PRG\cup(\{{a\leftarrow{}} \mid a\in V^t\}\cup\{{\{a\}\leftarrow{}} \mid a \in V^u\})
\)
to capture the extension of $\PRG$ with respect to a truth assignment to the input atoms in $I$.
Note that the externals in $E$ remain implicit in the domain of~$V$.
For instance,
the above program $\PRG$ with externals $E=\{\text{\lstinline{e(1)}}\!\}$
has a stable model including \lstinline{a(1)} but excluding \lstinline{b(1)} w.r.t.\ assignment $(\{\texttt{e(1)}\},\emptyset)$,
and vice versa with $(\emptyset,\emptyset)$.

Further examples involving external atoms are given in
Listings~\ref{lst:example:script},
\ref{fig:iclingo:python},
\ref{lst:exactly:one},
\ref{lst:queens},
and~\ref{lst:lp:targets}
below.

\subsection{Composing logic programs with externals}
\label{sec:programs:composition}

The assembly of (ground) subprograms can be characterized by means of module theory.
Program states are captured by modules
whose input and output atoms provide the respective atom base.
Successive grounding instructions result in modules to be joined with the modules of the corresponding program states.

Given an atom base $C$,
a (non-ground) extensible program \RRG\ yields the module
\begin{equation}\label{eq:ground:module}
\inst{\module{\RRG}}{C}
=
(\PRG,(C\cup E)\setminus\Head{\PRG},\Head{\PRG})
\end{equation}
via the ground program $\PRG$ with externals $E$ obtained by grounding \RRG\ relative to $C$.%
\footnote{Note that $E\setminus\Head{\PRG}$ consists of atoms stemming from \lstinline{#external} declarations
  that have not been ``overwritten'' by any rules in \PRG.}
For example,
grounding the extensible program in Listing~\ref{lst:example:external} w.r.t.\ the empty atom base
results in the module in \eqref{eq:example:module}.
\begin{equation}
  \label{eq:example:module}
  \left(
    \left\{
      \begin{array}{l}
        f(1) \leftarrow {}               \\
        f(2) \leftarrow {}               \\
        a(1) \leftarrow f(1), e(1)       \\
        b(1) \leftarrow f(1), \naf{e(1)} \\
        b(2) \leftarrow f(2)
      \end{array}
    \right\}
    ,
    \left\{
      e(1)
    \right\}
    ,
    \left\{
      \begin{array}[l]{l}
        f(1), f(2), {}\\
        a(1), {}\\
        b(1), b(2)
      \end{array}
    \right\}
  \right)
\end{equation}

Given the induction of modules from extensible programs w.r.t.\ to atom bases in \eqref{eq:ground:module},
we define successive program states in the following way.

The initial program state is given by the empty module
\(
\module{\PRG}_0 = (\emptyset,\emptyset,\emptyset)
\).

The program state succeeding a module~$\module{\PRG}_i$ is captured by the module
\begin{equation}\label{eq:module:composition}
\module{\PRG}_{i+1}
=
\module{\PRG}_i
\sqcup
\inst{\module{\RRG}_{i+1}}{\inp{\module{\PRG}_i}\cup\out{\module{\PRG}_i}}
\end{equation}
where
\(
\inst{\module{\RRG}_{i+1}}{\inp{\module{\PRG}_i}\cup\out{\module{\PRG}_i}}
\)
gives the result of grounding an extensible program \RRG\ relative to the atom base $\inp{\module{\PRG}_i}\cup\out{\module{\PRG}_i}$
as defined in~\eqref{eq:ground:module}.
Note that \prog{\inst{\module{\RRG}_{i+1}}{\inp{\module{\PRG}_i}\cup\out{\module{\PRG}_i}}} is a logic program over
\(
\inp{\module{\PRG}_{i+1}}\cup\out{\module{\PRG}_{i+1}}
\)
with externals $E$ such that
\[
\inp{\module{\PRG}_{i+1}}\cup\out{\module{\PRG}_{i+1}}
\ = \
(\inp{\module{\PRG}_i}\cup\out{\module{\PRG}_i})
\cup E
\cup\Head{\prog{\inst{\module{\RRG}_{i+1}}{\inp{\module{\PRG}_i}\cup\out{\module{\PRG}_i}}}}).
\]
This reflects the atom base of programs with externals discussed after equations~\eqref{grd:P} and~\eqref{grd:E}.
From a practical point of view,
the modules~$\module{\PRG}_i$ and~$\module{\PRG}_{i+1}$ represent the program state of \clingo{}
before and after the $(i{+}1)$-st execution of a \lstinline{ground} command in a \lstinline{main} routine.
And
\(
\inst{\module{\RRG}_{i+1}}{\inp{\module{\PRG}_i}\cup\out{\module{\PRG}_i}}
\)
captures the result of the $(i{+}1)$-st \lstinline{ground} command applied to an extensible program
relative to the atom base provided by $\module{\PRG}_i$.
Notably, the join leading to $\module{\PRG}_{i+1}$ can be undefined in case the constituent modules are non-compositional.
At system level, compositionality is only partially checked.
\clingo, or more precisely \clasp, respectively, issues an error message when atoms become redefined but
no cycle check over modules is done.

Interestingly,
input atoms induced by externals can also be used to incorporate future information.
To see this, consider the following rules (extracted from Listing~\ref{lst:exactly:one}):%
\footnote{Similar to the $n$-Queens problem, 
  putting a queen (\texttt{q}) on position $i$ attacks (\textit{a}) positions $1,\dots,i-1$, and no other \texttt{q}ueen may be put there.}
\begin{lstlisting}
#program s(i).
#external a(i).

{ q(i) }.
a(i-1) :- q(i).
a(i-1) :- a(i).
       :- a(i), q(i).
\end{lstlisting}
These rules give rise to the module $\module{S}_i(C)$ given as follows.
\begin{align}\label{eq:example:module:two}
  \left(
  \left\{
  \begin{array}{ccl}
    \{ q(i  ) \} &\leftarrow&\\
       a(i-1)    &\leftarrow& q(i)  \\
       a(i-1)    &\leftarrow& a(i)\\
                 &\leftarrow& a(i),q(i)
  \end{array}
  \right\}
  ,
  (C\cup\{a(i)\})\setminus\{a(i-1), q(i)\}
  ,
  \{a(i-1), q(i)\}
  \right)  
\end{align}
We observe that the input atom $a(i)$ of $\module{S}_i(C)$ is defined in $\module{S}_{i+1}(C)$.
Proceeding as in~\eqref{eq:module:composition} by letting \module{R} be \module{S},
each module $\module{P}_i$ has a single input atom $a(i)$ and output atoms
\(
\{a(0),\dots,a(i-1),q(1),\dots,q(i)\}
\);
it yields $i+1$ stable models, either the empty one or one of the form
\(
\{a(0),\dots,a(j-1),q(j)\}
\)
for $1\leq j\leq i$.
In the latter models, the truth of an atom like $a(0)$ relies on that of $q(j)$, 
occurring in a subsequently joined module whenever $j \geq 2$.

The next result provides a characterization of sequences of program states in terms of their constituent subprograms and their associated external atoms.
The well-definedness of a sequence depends upon their compositionality, as defined in Section~\ref{sec:background}.
% ------------------------------------------------------------
\begin{proposition}
Let $(\RRG_i)_{i>0}$ be a sequence of (non-ground) extensible programs,
and let
$\PRG_{i+1}$ be the ground program with externals $E_{i+1}$ obtained from $\RRG_{i+1}$ and $\inp{\module{\PRG}_i}\cup\out{\module{\PRG}_i}$ for $i\geq 0$,
where
\(
\module{\PRG}_0 = (\emptyset,\emptyset,\emptyset)
\) and %,
$\module{\PRG}_{i+1}$ is defined as in~\eqref{eq:module:composition},
and
$\module{\RRG}_{i+1}$ is defined as in~\eqref{eq:ground:module}
for $i\geq 0$.

If
\(
\module{\PRG}_i
\)
and
\(
\inst{\module{\RRG}_{i+1}}{\inp{\module{\PRG}_i}\cup\out{\module{\PRG}_i}}
% \bigsqcup_{i\geq 0}\module{\PRG}_i
\)
are compositional for some $j\geq 0$ and all $j> i\geq 0$,
then
\begin{enumerate}
\item
\(
\prog{\module{\PRG}_j}
=
\bigcup_{i>0}^j \PRG_i
\)
\item
\(
\;\inp{\module{\PRG}_j}
=
\bigcup_{i>0}^j E_i\setminus\bigcup_{i>0}^j \Head{\PRG_i}
\)
\item
\(
\out{\module{\PRG}_j}
=
\bigcup_{i>0}^j \Head{\PRG_i}
\)
\end{enumerate}
\end{proposition}
% ------------------------------------------------------------

\begin{proof}
For
\(
\module{\PRG}_0 = (\emptyset,\emptyset,\emptyset)
\),
we have that
\(
\prog{\module{\PRG}_0}=\inp{\module{\PRG}_0}=\out{\module{\PRG}_0}=\emptyset
\),
and assume that
\(
\prog{\module{\PRG}_j}
=
\bigcup_{i>0}^j \PRG_i
\),
\(
\inp{\module{\PRG}_j}
=
\bigcup_{i>0}^j E_i\setminus\bigcup_{i>0}^j \Head{\PRG_i}
\), and
\(
\out{\module{\PRG}_j}
=
\bigcup_{i>0}^j \Head{\PRG_i}
\)
for some $j\geq 0$.
According to \eqref{grd:P}--\eqref{eq:ground:module} and~\eqref{eq:module:composition},
\(
\inst{\module{\RRG}_{j+1}}{\inp{\module{\PRG}_j}\cup\out{\module{\PRG}_j}}
\)
is the module
\(
(\PRG_{j+1},\linebreak[1](\inp{\module{\PRG}_j}\cup\out{\module{\PRG}_j}\cup E_{j+1})\setminus\Head{\PRG_{j+1}},\linebreak[1]\Head{\PRG_{j+1}})
\), where
\[
\begin{array}{@{}l@{}l@{}}
\PRG_{j+1} = {} &
\{r\in\cground{\begin{array}[t]{@{}l@{}}\{a\leftarrow B\in\RRG_{j+1}\}\cup
{} \\
                \{a\leftarrow B,\varepsilon\mid(\texttt{\#external}~a : B)\in \RRG_{j+1}\}}{\inp{\module{\PRG}_j}\cup\out{\module{\PRG}_j}\cup\{\varepsilon\}} \mid \varepsilon\notin\body{r}\}
               \end{array}
\\
E_{j+1} = {} &
\{\head{r}\mid r\in\cground{\begin{array}[t]{@{}l@{}}\{a\leftarrow B\in\RRG_{j+1}\}\cup
{} \\
                \{a\leftarrow B,\varepsilon\mid(\texttt{\#external}~a : B)\in \RRG_{j+1}\}}{\inp{\module{\PRG}_j}\cup\out{\module{\PRG}_j}\cup\{\varepsilon\}}, \varepsilon\in\body{r}\}
                            \end{array}
\end{array}
\]
Provided that $\module{\PRG}_j$ and
\(
\inst{\module{\RRG}_{j+1}}{\inp{\module{\PRG}_j}\cup\out{\module{\PRG}_j}}
\)
are compositional,
their join is defined as
\[
\module{\PRG}_{j+1}
=
\left(
\begin{array}{l}
 \bigcup_{i>0}^{j+1} \PRG_i,
{} \\
 \bigcup_{i>0}^j E_i\setminus\bigcup_{i>0}^{j+1} \Head{\PRG_i}
 \cup
% {} \\
 (\bigcup_{i>0}^{j+1} E_i\cup\bigcup_{i>0}^j \Head{\PRG_i})\setminus\bigcup_{i>0}^{j+1} \Head{\PRG_i},
{} \\
 \bigcup_{i>0}^{j+1} \Head{\PRG_i}
\end{array}
\right)
\]
That is,
\(
\prog{\module{\PRG}_{j+1}}
=
\bigcup_{i>0}^{j+1} \PRG_i
\),
\(
\inp{\module{\PRG}_{j+1}}
=
\bigcup_{i>0}^{j+1} E_i\setminus\bigcup_{i>0}^{j+1} \Head{\PRG_i}
\), and
\(
\out{\module{\PRG}_{j+1}}
=
\bigcup_{i>0}^{j+1} \Head{\PRG_i}
\).%
\end{proof}

The ground rules in
\(
\prog{\inst{\module{\RRG}_{i+1}}{\inp{\module{\PRG}_i}\cup\out{\module{\PRG}_i}}}
% \bigsqcup_{i\geq 0}\module{\PRG}_i
\)
are obtained by grounding $R_{i+1}$ relative to the atom base $\inp{\module{\PRG}_i}\cup\out{\module{\PRG}_i}$,
viz.\ the previously gathered input and output atoms.
Unlike this, the full ground program $\ground{\RRG_{i+1}}$ takes all ground terms into account;
this includes all integers.
Thus, for example, grounding the extensible program in Listing~\ref{lst:example:external}
in full yields an infinite module:
\begin{equation}
  \label{eq:example:full}
  \left(
    \left\{
      \begin{array}{l}
        f(1) \leftarrow {}\\  
        f(2) \leftarrow {}\\
        a(1) \leftarrow f(1), e(1) {}\\
        b(1) \leftarrow f(1), \naf{e(1)} {}\\
        a(2) \leftarrow f(2), e(2) {}\\
        b(2) \leftarrow f(2), \naf{e(2)} {}\\
        a(3) \leftarrow f(3), e(3) {}\\
        b(3) \leftarrow f(3), \naf{e(3)} {}\\
        \dots
      \end{array}
    \right\}
    ,
    \left\{
      e(1)
    \right\}
    ,
    \left\{
      \begin{array}[l]{l}
        f(1), f(2), {}\\
        a(1), a(2), a(3), \dots {}\\
        b(1), b(2), b(3), \dots
      \end{array}
    \right\}
  \right)
\end{equation}
Both this module and the one in~\eqref{eq:example:module} obtained by contextual grounding
possess two stable models:
$X_1=\{e(1),f(1),f(2),a(1),b(2)\}$ and $X_2=\{f(1),f(2),b(1),b(2)\}$.
However, when joined with the module
$(\{e(2)\leftarrow\},\emptyset,\{e(2)\})$,
the stable models of~\eqref{eq:example:module} turn into $X_1\cup\{e(2)\}$ and $X_2\cup\{e(2)\}$,
while~\eqref{eq:example:full} yields $X_1\setminus\{b(2)\}\cup\{e(2),a(2)\}$ and $X_2\setminus\{b(2)\}\cup\{e(2),a(2)\}$.
This mismatch results from the fact that the atom $e(2)$ was removed from~\eqref{eq:example:module} when grounding relative to atom base $\{e(1)\}$.
The subsequent definition of $e(2)$ in module $(\{e(2)\leftarrow\},\emptyset,\{e(2)\})$ is thus stripped of any logical relation to the rules in~\eqref{eq:example:module}.
Such differences can be eliminated by stipulating
\(
\Head{\ground{\RRG_{j+1}}}\cap\bigcup_{i>0}^j\atom{\ground{\RRG_i}}\subseteq\inp{\module{\PRG}_j}
\)
for a sequence $(\RRG_i)_{i>0}$ of extensible programs and all $j\geq 0$,
where the program state $\module{\PRG}_j$ is obtained through contextual grounding.
This condition rules out the join conducted in our example.
But even so, full and contextual grounding would still be imbalanced.
For example, the module in~\eqref{eq:example:module} can be joined with modules defining $a(3),b(3),\dots$,
while doing the same with~\eqref{eq:example:full} violates compositionality.
Not to mention that infinite ground programs as in~\eqref{eq:example:full} cannot be utilized in practice.
This discussion shows the influence of contextual grounding on inputs, outputs, and resulting ground programs.
Hence, some care has to be taken when writing interacting subprograms.
Actually, apart from the ones in Section~\ref{sec:queens},
all following modules obtained by grounding parametrized programs satisfy the above condition 
and have the same solutions no matter which form of grounding is used.

\subsection{State-based characterization of multi-shot solving}
\label{sec:semantics:operational}

For capturing multi-shot solving,
we must account for sequences of system states, involving information about the programs kept within the grounder and the solver.
To this end, we define a simple operational semantics based on system states and associated operations.

An ASP \emph{system state} is a triple
\(
( \boldsymbol{R},\module{\PRG},V )
\)
where
\begin{itemize}
\item $\boldsymbol{R}=\left(\RRG_c\right)_{c\in\mathcal{C}}$ is a collection of extensible (non-ground) logic programs,%
  \footnote{Note that $\RRG_c$ is merely an indexed set and thus different from $\RRG(c)$.}
\item $\module{\PRG}$ is a module,
\item $V=( V^t,V^u )$ is a three-valued assignment over $\inp{\module{\PRG}}$.
\end{itemize}
When solving with $\module{\PRG}$,
the input atoms in $\inp{\module{\PRG}}$ are taken to be false by default,
that is, $V^f=\inp{\module{\PRG}}\setminus (V^t\cup V^u)$.
This can still be altered by dedicated directives as illustrated below.

As informal examples for ASP system states,
consider the ones obtained from the program in Listing~\ref{lst:example:program}
after separately grounding subprograms \lstinline{base} and \lstinline{acid} (while replacing \lstinline{k} with \lstinline{42}),
respectively:
\begin{align}
\label{ex:state:one}&
(
(\RRG(\texttt{base}),\RRG(\texttt{acid})),
(\{a(1)\leftarrow{},a(2)\leftarrow{}\},\emptyset,\{a(1),a(2)\}),
(\emptyset,\emptyset)
)
\\
\label{ex:state:two}&
(
(\RRG(\texttt{base}),\RRG(\texttt{acid})),
(\{b(42)\leftarrow{}\},\emptyset,\{b(42)\}),
(\emptyset,\emptyset)
)
\end{align}
Given that the program in Listing~\ref{lst:example:program} has no external declarations,
no truth values can be assigned.
This is different in states obtained from the program in Listing~\ref{lst:example:external}.
Grounding this yields the state
\begin{align}
\label{ex:state:tri}&
(
(\RRG(\texttt{base})),
\module{R}_b,
(\emptyset,\emptyset)
)
\end{align}
where $\module{R}_b$ is the module given in \eqref{eq:example:module}.
Furthermore, we have set the external atom $e(1)$ to false.
The way such states are obtained from non-ground logic programs is made precise next.

ASP system states can be created and modified by the following operations.

Function \textit{create} partitions a program into subprograms.
\begin{description}
\item [$\mathit{create}(R): {} \mapsto ( \boldsymbol{R},\module{\PRG},V )$] \

  for a list $R$ of (non-ground) rules and declarations
  where
  \begin{itemize}
  \item $\boldsymbol{R}=\left(\RRG(c)\right)_{c\in\mathcal{C}}$
  \item $\module{\PRG} = (\emptyset,\emptyset,\emptyset)$
  \item $V=(\emptyset,\emptyset)$
  \end{itemize}
\end{description}
Each subprogram $\RRG(c)$ gathers all rules and non-program directives in the scope of $c$.

The respective subprograms can be extended by function \textit{add} with rules as well as external declarations.
\begin{description}
\item [$\mathit{add}(R): ( \boldsymbol{R}_1,\module{\PRG},V )\mapsto( \boldsymbol{R}_2,\module{\PRG},V )$] \

  for a list $R$ of (non-ground) rules and declarations
  where
  \begin{itemize}
  \item $\boldsymbol{R}_1=\left(\RRG_c\right)_{c\in\mathcal{C}}$ and
        $\boldsymbol{R}_2=\left(\RRG_c\cup\RRG(c)\right)_{c\in\mathcal{C}}$
  \end{itemize}
\end{description}
Obviously, we have
\(
\mathit{add}(R)(\mathit{create}(\emptyset))
=
\mathit{create}(R)
\).
Note that \textit{add} only affects non-ground subprograms and thus ignores
compositionality issues since they appear on the ground level.

Function \textit{ground} instantiates the designated subprograms in $\boldsymbol{R}$ and binds their parameters.
The resulting ground programs along with their external atoms are then joined with the ones captured in the current state ---
provided that they are compositional.
\begin{description}
\item [$\mathit{ground}((n,\boldsymbol{t}_n)_{n\in\mathcal{N}}): ( \boldsymbol{R},\module{\PRG}_1,V_1 )\mapsto(\boldsymbol{R},\module{\PRG}_2,V_2 )$] \

  for a collection $(n,\boldsymbol{t}_n)_{n\in\mathcal{N}}$
  of pairs of non-integer constants $\mathcal{N}\subseteq\mathcal{C}$ and term tuples $\boldsymbol{t}_n\in\mathcal{T}^{k_n}$ of arity $k_n$
  where
  \begin{itemize}
  \item
    $\module{\PRG}_2=\module{\PRG}_1\sqcup\inst{\module{R}}{\inp{\module{\PRG}_1}\cup\out{\module{\PRG}_1}}$
    \\
    and \inst{\module{R}}{\inp{\module{\PRG}_1}\cup\out{\module{\PRG}_1}} is the module
    obtained as in \eqref{eq:ground:module} from
    \begin{itemize}
    \item extensible program
      \(
      \bigcup_{n\in\mathcal{N}}R_n[p_1/t_1,\dots,p_{k_n}/t_{k_n}]
      \)
      where
      \(
      \boldsymbol{t}_n=(t_1,\dots,t_{k_n})
      \)
      and
    \item atom base $\inp{\module{\PRG}_1}\cup\out{\module{\PRG}_1}$
    \end{itemize}
    for
    $\left(\RRG_c\right)_{c\in\mathcal{C}}=\boldsymbol{R}$
    \item   $V_2^t\,=\{a\in\inp{\module{\PRG}_2}\mid V_1(a)=t\,\}$
    \item[] $V_2^u  =\{a\in\inp{\module{\PRG}_2}\mid V_1(a)=u  \}$
%    \item  $V_2^f=\inp{\module{\PRG}_2}\setminus (V_2^t\cup V_2^u)$
  \end{itemize}
\end{description}
A few more technical remarks are in order.
First,
note that a previous external status of an atom is eliminated once it becomes defined by a ground rule.
This is accomplished by module composition, namely, the elimination of output atoms from input atoms.
Second,
note that jointly grounded subprograms are treated as a single logic program.
In fact, while $\mathit{ground}((c,\boldsymbol{p}_c),(c,\boldsymbol{p}_c))(s)$ and $\mathit{ground}((c,\boldsymbol{p}_c))(s)$ yield the same result,
$\mathit{ground}((c,\boldsymbol{p}_c))(\mathit{ground}((c,\boldsymbol{p}_c))(s))$ leads to two non-compositional modules
whenever normal rules are contained in $R_c$.
Finally,
note that new inputs stemming from just added external declarations are set to false in view of $V_2^f=\inp{\module{\PRG}_2}\setminus (V_2^t\cup V_2^u)$.

The above functionality lets us now formally characterize the states in \eqref{ex:state:one} and \eqref{ex:state:two}.
By abbreviating the logic program in Listing~\ref{lst:example:program} with \RRG,
the system state in \eqref{ex:state:one} results from
\(
\mathit{ground}((\texttt{base},()))(\mathit{create}(R))
\),
while
\(
\mathit{ground}((\texttt{acid},(\texttt{42})))(\mathit{create}(R))
\)
yields \eqref{ex:state:two}.
Moreover, observe the difference between grounding subprogram \texttt{acid} before \texttt{base} and vice versa.
While the ground program comprised in
\[
\mathit{ground}((\texttt{base},()))(\mathit{ground}((\texttt{acid},(\texttt{42})))(\mathit{create}(R)))
\]
only consists of the facts $\{b(42)\leftarrow{}, a(1)\leftarrow{},a(2)\leftarrow{}\}$,
the one contained in
\[
\mathit{ground}((\texttt{acid},(\texttt{42})))(\mathit{ground}((\texttt{base},()))(\mathit{create}(R)))
\]
includes additionally the rules $\{c(1,42)\leftarrow a(1);\ c(2,42)\leftarrow a(2)\}$.
This difference is due to contextual grounding (cf.~Section~\ref{sec:grounding:contextual}).
While in the first case  rule $c(X,42)\leftarrow a(X)$ is grounded w.r.t.\ atom base $\{b(42)\}$,
it is grounded relative to $\{a(1),a(2),b(42)\}$ in the second case.
Such effects are obviously avoided when jointly grounding both subprograms, as in
\[
\mathit{ground}((\texttt{base},()),(\texttt{acid},(\texttt{42})))(\mathit{create}(R))
\ .
\]

The next function allows us to change the truth assignment of input atoms.
\begin{description}
\item [$\mathit{assignExternal}(a,v): ( \boldsymbol{R},\module{\PRG},V_1 )\mapsto(\boldsymbol{R},\module{\PRG},V_2 )$] \

  for a ground atom $a$ and $v\in\{t,u,f\}$ where
  \begin{itemize}
  \item if $v=t$
    \begin{itemize}
    \item $V_2^t\,=V_1^t\cup\{a\}$ if $a\in\inp{\module{\PRG}}$, and $V_2^t=V_1^t$ otherwise
    \item $V_2^u  =V_1^u\setminus\{a\}$
    \end{itemize}
  \item if $v=u$
    \begin{itemize}
    \item $V_2^t\,=V_1^t\,\setminus\{a\}$
    \item $V_2^u  =V_1^u\cup\{a\}$ if $a\in\inp{\module{\PRG}}$, and $V_2^u=V_1^u$ otherwise
    \end{itemize}
  \item if $v=f$
    \begin{itemize}
    \item $V_2^t\,=V_1^t\setminus\{a\}$
    \item $V_2^u  =V_1^u\setminus\{a\}$
    \end{itemize}
    % \item $V_2^f=\inp{\module{\PRG}_2}\setminus (V_2^t\cup V_2^u)$
  \end{itemize}
\end{description}
While the default truth value of input atoms is false, making them undefined results in a choice.
Note that \textit{assignExternal} only affects input atoms, that is, ``non-overwritten'' externals atoms.
If an atom is not external, then \textit{assignExternal} has no effect.

With this function,
we can now characterize the system state in \eqref{ex:state:tri}.
Abbreviating the program in Listing~\ref{lst:example:external} with \RRG,
system state \eqref{ex:state:tri} is issued by
\begin{equation}\label{ex:state:tri:ops}
\mathit{assignExternal}(\texttt{e(1)},f)(\mathit{ground}((\texttt{base},()))(\mathit{create}(R)))
\ .
\end{equation}
The resulting state is the same as the previous one, obtained from
\(
\mathit{ground}((\texttt{base},()))(\mathit{create}(R))
\),
since external atoms are assigned false by function \textit{ground}.

Function \textit{releaseExternal} removes the external status from an atom and sets it permanently to false,
otherwise this function has no effect.
\begin{description}
\item [$\mathit{releaseExternal}(a): ( \boldsymbol{R},\module{\PRG}_1,V_1 )\mapsto(\boldsymbol{R},\module{\PRG}_2 ,V_2 )$] \

  for a ground atom~$a$ where
  \begin{itemize}
  \item $\module{\PRG}_2=(\prog{\module{\PRG}_1},\inp{\module{\PRG}_1}\setminus\{a\},\out{\module{\PRG}_1}\cup\{a\})$ if $a\in\inp{\module{\PRG}_1}$,
    and $\module{\PRG}_2=\module{\PRG}_1$ otherwise
  \item $V_2^t\,=V_1^t\,\setminus\{a\}$
  \item $V_2^u  =V_1^u  \setminus\{a\}$
  \end{itemize}
\end{description}
Note that \textit{releaseExternal} only affects input atoms; defined atoms remain unaffected.
The addition of $a$ to the output makes sure that it can never be re-defined, neither by a rule nor an external declaration.
A released (input) atom is thus permanently set to false, since it is neither defined by any rule nor part of the input atoms,
and is also denied both statuses in the future.

The following properties shed some light on the interplay among the previous operations.
For an
ASP system state $s$ and $v,v'\in\{t,f,u\}$,
we have
\begin{enumerate}
\item $\mathit{releaseExternal}(a)(\mathit{releaseExternal}(a)(s))=\mathit{releaseExternal}(a)(s)$
\item $\mathit{releaseExternal}(a)(\mathit{assignExternal}(a,v)(s))=\mathit{releaseExternal}(a)(s)$
\item $\mathit{assignExternal}(a,v)(\mathit{releaseExternal}(a)(s))=\mathit{releaseExternal}(a)(s)$ %=s$
\item $\mathit{assignExternal}(a,v)(\mathit{assignExternal}(a,v')(s))=\mathit{assignExternal}(a,v)(s)$
\end{enumerate}

Finally, \textit{solve} leaves the system state intact and
outputs a possibly filtered set of stable models of the logic program with externals comprised in the current state
(cf.~Section~\ref{sec:programs:extensible}).
This set is general enough to define all basic reasoning modes of ASP.
\begin{description}
\item [$\mathit{solve}(( A^t,A^f )): ( \boldsymbol{R},\module{\PRG},V )\mapsto( \boldsymbol{R},\module{\PRG},V )$] \

  outputs the set
  \begin{align}\label{solve:models}
    \mathcal{X}_{\module{\PRG},V} =
    \{X\mid X\text{ is a stable model of }\prog{\module{\PRG}}\text{ w.r.t.\ }V\text{ such that }A^t\subseteq X\text{ and }A^f\cap X = \emptyset\}
  \end{align}
\end{description}
To be more precise,
a state like $( \boldsymbol{R},\module{\PRG},V )$ comprises the ground logic program \prog{\module{\PRG}} with external atoms \inp{\module{\PRG}}.
The latter constitutes the domain of the partial assignment $V$.
Recall from Section~\ref{sec:programs:extensible} that the stable models of \prog{\module{\PRG}} w.r.t.\ $V$ are given by the stable models
of the program
\(
\prog{\module{\PRG}}\cup\{{a\leftarrow{}}\mid a\in V^t\}\cup\{{\{a\}\leftarrow{}}\mid a\in V^u\}
\).
In addition to the assignment $V$ on input atoms,
we consider another partial assignment $( A^t,A^f )$ over an arbitrary set of atoms for filtering stable models;
they are commonly referred to as assumptions.%
\footnote{In \clingo, or more precisely in \clasp,
  such {assumptions} are the principal parameter to the underlying \lstinline{solve} function
  (see below).
  The term assumption traces back to~\citeN{eensor03b}; it was used in ASP by~\citeN{gekakaosscth08a}.}
Note the difference among input atoms and (filtering) assumptions.
While a true input atom amounts to a fact, a true assumption acts as an integrity constraint.%
\footnote{That is, the difference between `$a\leftarrow$' and `$\leftarrow\naf{a}$'.}
Thus, a true assumption must not be unfounded, while a true external atom is exempt from this condition.
Also, undefined input atoms are regarded as false, while undefined assumptions remain neutral.
Finally, at the solver level, input atoms are a transient part of the representation,
while assumptions only affect the assignment of a single search process.

For illustration, observe that applying $\mathit{solve}()$ to the system state in \eqref{ex:state:tri} leaves the state unaffected and
outputs a single stable model containing \lstinline{b(1)}.
Unlike this, no model is obtained from $\mathit{solve}((\emptyset,\{\texttt{b(1)}\}))(\eqref{ex:state:tri})$.
For a complement,
$\mathit{solve}()(\mathit{assignExternal}(\texttt{e(1)},u)(\eqref{ex:state:tri}))$ outputs two models,
one with \lstinline{a(1)} and another with \lstinline{b(1)}.

From the viewpoint of operational semantics,
a multi-shot ASP solving process can be associated with a sequence of operations
\(
( o_k )_{k\in K}
\),
which induce a sequence
\(
( \boldsymbol{R}_k,\module{\PRG}_k,V_k )_{k\in K}
\)
of ASP system states where
\begin{enumerate}
\item $o_0=\mathit{create}(R)$ for some logic program $R$ % and $o_k\neq\mathit{create}(\cdot)$ for $k> 0$
\item $( \boldsymbol{R}_0,\module{\PRG}_0,V_0 )=o_0$
\item $( \boldsymbol{R}_k,\module{\PRG}_k,V_k )=o_k(( \boldsymbol{R}_{k-1},\module{\PRG}_{k-1},V_{k-1}))$ for $k> 0$
\end{enumerate}
Note that only $o_0$ creates states while all others map states to states.

For capturing the result of multi-shot solving in terms of stable models,
we consider the sequence of sets of stable models obtained at each solving step.
More precisely, given a sequence of operations and system states as above,
a multi-shot solving process can be associated with the sequence
\(
( \mathcal{X}_{\module{\PRG}_j,V_j})_{j\in K, o_j=\mathit{solve}(( A^t_j,A^f_j ))}
\)
of sets of stable models, where $\mathcal{X}_{\module{\PRG},V}$ is defined w.r.t.\ $(A^t,A^f)$ as in~\eqref{solve:models}.

All of the above state operations have almost literal counterparts in \clingo's APIs.
For instance, the \lstinline{Control} class of the Python API for capturing system states provides the methods
\lstinline{__init__},
\lstinline{add},
\lstinline{ground},
\lstinline{assign_external},
\lstinline{release_external},
and
\lstinline{solve}.%
\footnote{For a complete listing of functions and classes available in \lstinline{clingo}'s Python API,\par see \url{https://potassco.org/clingo/python-api/current/clingo.html}}

\subsection{Example}
\label{sec:multi:shot:example}

Let us demonstrate the above apparatus via the authentic \clingo{} program in Listing~\ref{lst:example:script}.
% ------------------------------------------------------------
\begin{lstlisting}[language=clingo,caption={Example with \texttt{\#external} and \texttt{\#program} declarations controlled by a \texttt{main} routine in Python (\lstinline{simple.lp})},label={lst:example:script}]
#external p(1;2;3).
p(0) :- p(3).
p(0) :- not p(0).

#program succ(n).
#external p(n+3).
p(n) :- p(n+3).
p(n) :- not p(n+1), not p(n+2).

#script(python)
from clingo import Function
def main(prg):
    prg.ground([("base", [])])
    prg.assign_external(Function("p", [3]), True)
    prg.solve()
    prg.assign_external(Function("p", [3]), False)
    prg.solve()
    prg.ground([("succ", [1]),("succ", [2])])
    prg.solve()
    prg.ground([("succ", [3])])
    prg.solve()
#end.
\end{lstlisting}
% ------------------------------------------------------------
This program consists of two subprograms, viz.\ \lstinline{base} and \lstinline{succ} given in Line~1--3 and 5--8, respectively.
Note that once the rule in Line~3 is internalized no stable models are obtained whenever its body is satisfied.
Since we use the \lstinline{main} routine in Line~10--22 within a \lstinline{#script} environment,
an initial \clingo{} object is created for us and bound to variable \lstinline{prg} (cf.~Line~12).
This amounts to an implicit call of
\(
\mathit{create}(R)
\),
where $R$ is the list of (non-ground) rules and declarations in Line~1--8 in Listing~\ref{lst:example:script}.%
\footnote{Further \clingo{} objects could be created with
\(
\mathit{create}(\emptyset)
\)
and then further augmented and manipulated.}

The initial \clingo{} object gathers all rules and external declarations in the scope of the subprograms \lstinline{base} and \lstinline{succ};
its state is captured by
\[
(\boldsymbol{\RRG}_0,\module{\PRG}_0,V_0 )
=
((\RRG(\mathtt{base}),\RRG(\mathtt{succ})),(\emptyset,\emptyset,\emptyset),(\emptyset,\emptyset))
\ .
\]
where $\RRG(\mathtt{base})$ and $\RRG(\mathtt{succ})$ consist of the non-ground rules and external declarations in Line~1--3 and 5--8, respectively.
Empty subprograms are omitted.

The initial program state
induces the atom base
\(
\inp{\module{\PRG}_0}\cup\out{\module{\PRG}_0}=\emptyset
\).

The \lstinline{ground} instruction in Line~13 takes the extensible logic program
\(
\RRG(\mathtt{base})
\)
along with the empty base of atoms
and yields the ground program $\PRG_1$ with externals $E_1$,
where
\begin{align*}
  \PRG_1 & {} = \{p(0)\leftarrow p(3);\ p(0)\leftarrow \naf{p(0)}\}\\
     E_1 & {} = \{p(1),p(2),p(3)\} \ .
\end{align*}
This results in the module
\(
\inst{\module{\RRG}_1}{\emptyset}
=
(\PRG_1,E_1,\{p(0)\})
\),
whose join with $\module{\PRG}_0$ yields
\[
\module{\PRG}_1
=
\module{\PRG}_0 \sqcup \inst{\module{\RRG}_1}{\emptyset}
=
(
\{
p(0)\leftarrow p(3);\
p(0)\leftarrow \naf{p(0)}
\},
\{p(1),p(2),p(3)\},
\{p(0)\}
)
\ .
\]
We then obtain the system state
\[
(\boldsymbol{\RRG}_1,\module{\PRG}_1,V_1)
=
(\boldsymbol{\RRG}_0,\module{\PRG}_0 \sqcup \inst{\module{\RRG}_1}{\emptyset},V_0)
\ .
\]

While the input atoms $p(1)$, $p(2)$, and $p(3)$ are assigned (by default) to false by $V_1$,
the instruction in Line~14 switches the value of $p(3)$ to true.
And we obtain the system state
\[
(\boldsymbol{\RRG}_2,\module{\PRG}_2,V_2)
=
(\boldsymbol{\RRG}_0,\module{\PRG}_1,(\{p(3)\},\emptyset))
\ .
\]
Applying the \lstinline{solve} instruction in Line~15 leaves the state intact
and outputs the stable model $\{p(0),p(3)\}$ of~$\module{\PRG}_2$ w.r.t.\ $V_2$.
Note that making $p(3)$ true leads to the derivation of $p(0)$, which blocks the rule in Line~3.

Next,
the instruction in Line~16 turns $p(3)$ back to false,
which puts the ASP system into the state
\[
(\boldsymbol{\RRG}_3,\module{\PRG}_3,V_3)
=
(\boldsymbol{\RRG}_0,\module{\PRG}_1,(\emptyset,\emptyset))
\ .
\]
The last change withdraws the derivation of $p(0)$ and no stable model is obtained from $\module{\PRG}_3$ w.r.t.\ $V_3$ in Line~17.

The \lstinline{ground} instruction in Line~18 instantiates the rules and external declarations of subprogram \lstinline{succ(n)} in Line~5--8 twice.
Once the parameter \lstinline{n} is instantiated with~\lstinline{1} and once with~\lstinline{2}.
This yields the extensible logic program
\(
\RRG_4=\RRG(\mathtt{succ})[\mathtt{n}/\mathtt{1}]\cup\RRG(\mathtt{succ})[\mathtt{n}/\mathtt{2}]
\).
This program is then grounded relative to
\(
\inp{\module{\PRG}_3}\cup\out{\module{\PRG}_3}=\{p(1),p(2),p(3)\}\cup\{p(0)\}
\),
which results in the following ground program with externals and resulting module:
\begin{align*}
  \PRG_4 & {} =
        \left\{
        \begin{array}{@{}l@{}}
          p(1)\leftarrow p(4);\ p(1)\leftarrow \naf{p(2)},\naf{p(3)};
          \\
          p(2)\leftarrow p(5);\ p(2)\leftarrow \naf{p(3)},\naf{p(4)}
        \end{array}
  \right\}\\
  E_4 & {} = \{p(4),p(5)\}\\
  \text{and } \
  \inst{\module{\RRG}_4}{\inp{\module{\PRG}_3}\cup\out{\module{\PRG}_3}}
  &=
  \left(
    \PRG_4
    % \left\{
    %   \begin{array}{@{}l@{}}
    %     p(1)\leftarrow p(4);\ p(1)\leftarrow \naf{p(2)},\naf{p(3)};
    %     \\
    %     p(2)\leftarrow p(5);\ p(2)\leftarrow \naf{p(3)},\naf{p(4)}
    %   \end{array}
    % \right\}
    ,
    \left\{
      \begin{array}{@{}l@{}}
        p(0),p(4),
        \\
        p(3),p(5)
      \end{array}
    \right\}
    ,
    \left\{
      \begin{array}{@{}l@{}}
        p(1),\\p(2)
      \end{array}
    \right\}
  \right)
\end{align*}

Joining the latter with the program module $\module{\PRG}_3$ of the previous system state yields $\module{\PRG}_4=\module{\PRG}_3\sqcup\inst{\module{\RRG}_4}{\inp{\module{\PRG}_3}\cup\out{\module{\PRG}_3}}$,
or more precisely:
\[
\module{\PRG}_4
=
\left(
\left\{
\begin{array}{@{}l@{}}
p(0)\leftarrow p(3);\quad   p(1)\leftarrow p(4);\ p(1)\leftarrow \naf{p(2)},\naf{p(3)};
\\
p(0)\leftarrow \naf{p(0)};\ p(2)\leftarrow p(5);\ p(2)\leftarrow \naf{p(3)},\naf{p(4)}
\end{array}
\right\}
,
\left\{
\begin{array}{@{}l@{}}
\phantom{p(3),}\, p(4),
\\
p(3),p(5)
\end{array}
\right\}
,
\left\{
\begin{array}{@{}l@{}}
p(0),p(1),
\\
\phantom{p(0),}\,p(2)
\end{array}
\right\}
\right)
\]
This puts the ASP system into the state
\[
(\boldsymbol{\RRG}_4,\module{\PRG}_4,V_4 )
=
(\boldsymbol{\RRG}_0,\module{\PRG}_3\sqcup\inst{\module{\RRG}_4}{\inp{\module{\PRG}_3}\cup\out{\module{\PRG}_3}},V_3)
\ .
\]

The subsequent \lstinline{solve} command in Line~19 leaves the state intact
but returns no stable models for~$\module{\PRG}_4$ w.r.t.\ $V_4$.

Then, \clingo{} proceeds in Line~20 with the \lstinline{ground} instruction
instantiating $\RRG(\mathtt{succ})[\mathtt{n}/\mathtt{3}]$ relative to the atom base
${\inp{\module{\PRG}_4}\cup\out{\module{\PRG}_4}}=\{p(0),p(1),p(2),p(3),p(4),p(5)\}$.
This results in the following ground program with externals and induced module:
\begin{align*}
  \PRG_5 & {} = \{p(3)\leftarrow p(6);\ p(3)\leftarrow \naf{p(4)},\naf{p(5)}\}\\
  E_5 & {} = \{p(6)\}\\
\text{and } \
\inst{\module{\RRG}_5}{\inp{\module{\PRG}_4}\cup\out{\module{\PRG}_4}}&{}
=
\left(
\PRG_5%\{p(3)\leftarrow p(6);\ p(3)\leftarrow \naf{p(4)},\naf{p(5)}\}
,
\left\{
\begin{array}{@{}l@{}}
p(0),p(1),p(2),
\\
p(4),p(5),p(6)
\end{array}
\right\}
,
  \left\{
p(3)
\right\}
\right)
\end{align*}
With the latter, we obtain the system state
\[
(\boldsymbol{\RRG}_5,\module{\PRG}_5,V_5 )
=
(\boldsymbol{\RRG}_0,\module{\PRG}_4\sqcup\inst{\module{\RRG}_5}{\inp{\module{\PRG}_4}\cup\out{\module{\PRG}_4}},V_3)
\]
where
\[
\module{\PRG}_5
=
\left(
\left\{
\begin{array}{@{}l@{}}
p(0)\leftarrow p(3);\quad   p(1)\leftarrow p(4);\ p(1)\leftarrow \naf{p(2)},\naf{p(3)};
\\
p(0)\leftarrow \naf{p(0)};\ p(2)\leftarrow p(5);\ p(2)\leftarrow \naf{p(3)},\naf{p(4)};
\\
\phantom{p(0)\leftarrow \naf{p(0)};\ }\,p(3)\leftarrow p(6);\ p(3)\leftarrow \naf{p(4)},\naf{p(5)}
\end{array}
\right\}
,
\left\{
\begin{array}{@{}l@{}}
p(4),
\\
p(5),
\\
p(6)
\end{array}
\right\}
,
\left\{
\begin{array}{@{}l@{}}
p(0),p(1),
\\
\qquad\ p(2),
\\
\qquad\ p(3)
\end{array}
\right\}
\right)
\ .
\]

Finally,
the \lstinline{solve} command in Line~21
yields the stable model
$\{p(0),p(3)\}$ of module $\module{\PRG}_5$ w.r.t.\ assignment $(\emptyset,\emptyset)$.

The result of the ASP solving process induced by program \texttt{simple.lp} from Listing~\ref{lst:example:script} is given in Listing~\ref{lst:example:run}.
The parameter \texttt{0} instructs \clingo{} to compute all stable models upon each invocation of \lstinline{solve}.%
\footnote{In fact, \clingo's API allows for changing solver configurations in between successive solver calls.}
% ------------------------------------------------------------
\begin{lstlisting}[numbers=none,basicstyle=\ttfamily\small,caption={Running the program in Listing~\ref{lst:example:script} with \clingo},label={lst:example:run}]
$ clingo simple.lp 0
clingo version 4.5.4
Reading from simple.lp
Solving...
Answer: 1
p(3) p(0)
Solving...
Solving...
Solving...
Answer: 1
p(3) p(0)
SATISFIABLE

Models     : 2     
Calls      : 4
Time       : 0.008s (Solving: 0.00s 1st Model: 0.00s Unsat: 0.00s)
CPU Time   : 0.000s
\end{lstlisting}
% ------------------------------------------------------------
Each such invocation is indicated by `{\small\texttt{Solving...}}'.
We see that stable models are only obtained for the first and last invocation.
Semantically, our ASP solving process thus results in a sequence of four sets of stable models,
namely
\(
(\{\{p(0),p(3)\}\},\emptyset,\emptyset,\{\{p(0),p(3)\}\})
\).

The above example illustrates the customized selection of (non-ground) subprograms to instantiate
upon \lstinline{ground} commands.
For a convenient declaration of input atoms from other subprogram instances,
schematic \lstinline{#external} declarations are embedded into the grounding process.
Given that they do not contribute ground rules, but merely qualify (undefined)
atoms that should be exempted from simplifications, \lstinline{#external} declarations
only contribute to the signature of subprograms' ground instances.
Hence,
it is advisable to condition them by domain predicates%
\footnote{Domain and built-in predicates have unique extensions that can be evaluated entirely by means of grounding.}
\cite{lparseManual} only,
as this precludes any interferences between signatures and grounder implementations.
As long as input atoms remain undefined, their truth values
can be freely picked and modified in-between \lstinline{solve} commands
via \lstinline{assign_external} instructions.
This allows for configuring the inputs to modules
representing system states in order to select among their stable models.
Unlike that,
the predecessors \iclingo{} and \oclingo{} of \clingo{}
always assigned input atoms to false,
so that the addition of rules was necessary to
accomplish switching truth values as in Line~14 and~16 above.
However,
for a well-defined semantics,
\clingo{} like its predecessors builds on the assumption that
modules resulting from subprogram instantiation are compositional,
which essentially requires definitions of atoms and
mutual positive dependencies to be local to evolving ground programs
(cf.~\cite{gekakaosscth08a}).

%%% Local Variables:
%%% mode: latex
%%% TeX-master: "paper"
%%% End:

\section{Using multi-shot solving in practice}\label{sec:practice}

After fixing the formal foundations of multi-shot solving and sketching the corresponding \clingo{} constructs,
let us now illustrate their usage in several case studies.

%%% Local Variables: 
%%% mode: latex
%%% TeX-master: "paper"
%%% End: 

\subsection{Incremental ASP solving}
\label{sec:incremental}

As mentioned, the new \clingo{} series fully supersedes its special-purpose predecessors \iclingo{} and \oclingo.
To illustrate this,
we give below a \python\ implementation of \iclingo's control loop,
corresponding to the one shipped with \clingo.%
\footnote{The source code is also available in \clingo's examples.
The code for incremental solving is in \url{https://github.com/potassco/clingo/tree/master/examples/clingo/iclingo} and the Towers of Hanoi example in \url{https://github.com/potassco/clingo/tree/master/examples/gringo/toh}.}
Roughly speaking,
\iclingo{} offers a step-oriented, incremental approach to ASP that avoids redundancies by gradually processing the extensions to a problem
rather than repeatedly re-processing the entire extended problem (as in iterative deepening search).
To this end, a program is partitioned into a
base part, describing static knowledge independent of the step parameter~\lstinline{t},
a cumulative part, capturing knowledge accumulating with increasing~\lstinline{t},
and
a volatile part specific for each value of~\lstinline{t}.
These parts are delineated in \iclingo{} by the special-purpose directives \lstinline{#base}, `\lstinline{#cumulative t}', and `\lstinline{#volatile t}'.
In \clingo, all three parts are captured by \lstinline{#program} declarations
along with \lstinline{#external} atoms for handling volatile rules.
More precisely,
our exemplar relies upon subprograms named \lstinline{base}, \lstinline{step}, and \lstinline{check}
along with external atoms of form \lstinline{query(t)}.

We illustrate this approach by adapting the Towers of Hanoi encoding by~\citeN{gekakasc12a} in Listing~\ref{fig:toh:enc}.
% --------------------------------------------------------------------------------------------------------------------------------------------
\begin{lstlisting}[float=tb,literate={\%\%}{}{0},escapeinside={\#(}{\#)},language=clingo,caption={Towers of Hanoi incremental encoding (\lstinline{tohE.lp})},label={fig:toh:enc}]
#program base.%%#(\label{fig:toh:enc:base}#)
on(D,P,0) :- init_on(D,P).%%#(\label{fig:toh:enc:static}#)

#program step(t).%%#(\label{fig:toh:enc:step:begin}#)
1 { move(D,P,t) : disk(D), peg(P) } 1.

move(D,t)        :- move(D,P,t).
on(D,P,t)        :- move(D,P,t).
on(D,P,t)        :- on(D,P,t-1), not move(D,t).
blocked(D-1,P,t) :- on(D,P,t-1).
blocked(D-1,P,t) :- blocked(D,P,t), disk(D).

:- move(D,P,t), blocked(D-1,P,t).
:- move(D,t), on(D,P,t-1), blocked(D,P,t).
:- disk(D), not 1 { on(D,P,t) } 1.%%#(\label{fig:toh:enc:step:end}#)

#program check(t).
:- goal_on(D,P), not on(D,P,t), query(t).%%#(\label{fig:toh:enc:goal}#)

#show move/3.
\end{lstlisting}
% --------------------------------------------------------------------------------------------------------------------------------------------
\begin{lstlisting}[float=tb,literate={\%\%}{}{0},escapeinside={\#(}{\#)},language=clingo,caption={Towers of Hanoi instance (\lstinline{tohI.lp})},label={fig:toh:ins}]
peg(a;b;c).
disk(1..4).
init_on(1..4,a).
goal_on(1..4,c).

\end{lstlisting}
% --------------------------------------------------------------------------------------------------------------------------------------------
%
The problem instance in Listing~\ref{fig:toh:ins} as well as Line~\ref{fig:toh:enc:static} in~\ref{fig:toh:enc} constitute static knowledge
and thus belong to the \lstinline{base} program.
The transition function is described in the subprogram \lstinline{step} in Line~\ref{fig:toh:enc:step:begin}--\ref{fig:toh:enc:step:end}
of Listing~\ref{fig:toh:enc}.
Finally, the query is expressed in Line~\ref{fig:toh:enc:goal};
its volatility is realized by making the actual goal condition `\lstinline{goal_on(D,P), not on(D,P,t)}'
subject to the truth assignment to the external atom \lstinline{query(t)}.
For convenience,
this atom is predefined in Line~\ref{fig:iclingo:python:query} in Listing~\ref{fig:iclingo:python} as part of the \lstinline{check} program (cf.~Line~\ref{fig:iclingo:python:check}).
Hence, subprogram \lstinline{check} consists of a user- and predefined part.
Since the encoding of the Towers of Hanoi problem is fairly standard, we refer the interested reader to the literature~\cite{gekakasc12a}
and devote ourselves in the sequel to its solution by means of multi-shot solving.

Grounding and solving of the program in Listing~\ref{fig:toh:ins} and~\ref{fig:toh:enc} is controlled by the \python{}
script in Listing~\ref{fig:iclingo:python}.
%
% --------------------------------------------------------------------------------------------------------------------------------------------
\begin{lstlisting}[float=tbp,literate={\%\%}{}{0},escapeinside={\#(}{\#)},language=clingo,basicstyle=\ttfamily\small,caption={\python\ script implementing \iclingo{} functionality in \clingo\ (\lstinline{inc.lp})},label=fig:iclingo:python]
#script (python)

from clingo import Function

def get(val, default):#(\label{fig:iclingo:python:const:begin}#)
    return val if val != None else default

def main(prg):
    imin   = get(prg.get_const("imin"), 1)#(\label{fig:iclingo:python:imin}#)
    imax   = prg.get_const("imax")
    istop  = get(prg.get_const("istop"), "SAT")#(\label{fig:iclingo:python:const:end}\label{fig:iclingo:python:istop}#)

    step, ret = 0, None#(\label{fig:iclingo:python:vars}#)
    while ((imax is None or step < imax) and#(\label{fig:iclingo:python:loop:begin}#)
           (step == 0   or step < imin or (
              (istop == "SAT"     and not ret.satisfiable) or
              (istop == "UNSAT"   and not ret.unsatisfiable) or
              (istop == "UNKNOWN" and not ret.unknown)))):
        parts = []
        parts.append(("check", [step]))
        if step > 0:
            prg.release_external(Function("query", [step-1]))
            parts.append(("step", [step]))
            prg.cleanup()#(\label{fig:iclingo:python:cleanup}#)
        else:
            parts.append(("base", []))
        prg.ground(parts)
        prg.assign_external(Function("query", [step]), True)#(\label{fig:iclingo:python:assign:query}#)
        ret, step = prg.solve(), step+1#(\label{fig:iclingo:python:loop:end}\label{fig:iclingo:python:solve}#)
#end.

#program check(t).%%#(\label{fig:iclingo:python:check}#)
#external query(t).%%#(\label{fig:iclingo:python:query}#)
\end{lstlisting}
% --------------------------------------------------------------------------------------------------------------------------------------------
%
Lines~\ref{fig:iclingo:python:const:begin}--\ref{fig:iclingo:python:const:end} fix the values of the constants \lstinline{imin}, \lstinline{imax}, and \lstinline{istop}.
In fact, the setting in Line~\ref{fig:iclingo:python:imin} and~\ref{fig:iclingo:python:istop} relieves us from adding
`\lstinline{-c imin=0 -c istop="SAT"}'
when calling \clingo.
All three constants mimic command line options in \iclingo.
\lstinline{imin} and \lstinline{imax} prescribe a least and largest number of iterations, respectively;
\lstinline{istop} gives a termination criterion.
The initial values of variables \lstinline{step} and \lstinline{ret} are set in Line~\ref{fig:iclingo:python:vars}.
The value of \lstinline{step} is used to instantiate the parametrized subprograms
and \lstinline{ret} comprises the solving result.
Together, the previous five variables control the loop in Lines~\ref{fig:iclingo:python:loop:begin}--\ref{fig:iclingo:python:loop:end}.

The subprograms grounded at each iteration are accumulated in the list \lstinline{parts}.
Each of its entries is a pair consisting of a subprogram name along with its list of actual parameters.
In the very first iteration, the subprograms \lstinline{base} and \lstinline{check(0)} are grounded.
Note that this involves the declaration of the external atom \lstinline{query(0)} and the assignment of its default value false.
The latter is changed in Line~\ref{fig:iclingo:python:assign:query} to true in order to activate the actual query.
The \lstinline{solve} call in Line~\ref{fig:iclingo:python:solve} then amounts to checking whether the goal situation is already satisfied in the initial state.
As well, the value of \lstinline{step} is incremented to \lstinline{1}.

As long as the termination condition remains unfulfilled,
each following iteration takes the respective value of variable \lstinline{step}
to replace the parameter in subprograms \lstinline{step} and \lstinline{check}
during grounding.
In addition,
the current external atom \lstinline{query(t)} is set to true,
while the previous one is permanently set to false.
This disables the corresponding instance of the integrity constraint in Line~\ref{fig:toh:enc:goal} of Listing~\ref{fig:toh:enc} before it is replaced in the next iteration.
In this way,
the query condition only applies to the current horizon.

An interesting feature is given in Line~\ref{fig:iclingo:python:cleanup}.
As its name suggests, this function cleans up atom bases used during grounding.
That is, whenever the truth value of an atom is ultimately determined by the solver,
it is communicated to the grounder where it can be used for simplifications
in subsequent grounding steps.
The call in Line~\ref{fig:iclingo:python:cleanup} effectively removes atoms from the current atom base
(and marks some atoms as facts, which might lead to further simplifications).

The result of each call to \lstinline{solve} is printed by \clingo.
In our example, the solver is called 16 times before a plan of length 15 is found:
%
% --------------------------------------------------------------------------------------------------------------------------------------------
\begin{lstlisting}[mathescape=true,numbers=none,basicstyle=\ttfamily\small,caption={Running the programs in Listing~\ref{fig:toh:enc} and~\ref{fig:toh:ins} with \clingo},label={lst:toh:run}]
$\mbox{\$}$ clingo inc.lp tohE.lp tohI.lp 0
clingo version 4.5.4
Reading from inc.lp ...
Solving...
$\textit{[\/\dots]}$
Solving...
Answer: 1
move(4,b,1)  move(3,c,2)  move(4,c,3)  move(2,b,4)  move(4,a,5)  \
move(3,b,6)  move(4,b,7)  move(1,c,8)  move(4,c,9)  move(3,a,10) \
move(4,a,11) move(2,c,12) move(4,b,13) move(3,c,14) move(4,c,15)
SATISFIABLE

Models     : 1
Calls      : 16
Time       : 0.020s (Solving: 0.00s 1st Model: 0.00s Unsat: 0.00s)
CPU Time   : 0.020s
\end{lstlisting}
% --------------------------------------------------------------------------------------------------------------------------------------------

For a complement,
we give in Figure~\ref{fig:toh:operations} a trace of the Python script in terms of the operations defined in Section~\ref{sec:semantics:operational}.
% ------------------------------------------------------------
\begin{figure}
  \centering
  \[
  \begin{array}{llr}
    \mathtt{step}&\mathrm{Operation}                                         & \mathrm{Line} \\
                 &\mathit{create}(\mathtt{TOH})                              & 8             \\
                0&\mathit{ground}(((\mathtt{base},()),(\mathtt{check},(0)))) &27             \\
                 &\mathit{assignExternal}(\mathtt{query(0)},\mathtt{t})      &28             \\
                 &\mathit{solve}((\emptyset,\emptyset))                      &29             \\
                1&\mathit{releaseExternal}(\mathtt{query(0)})                &22             \\
                 &\mathit{ground}(((\mathtt{step},(1)),(\mathtt{check},(1))))&27             \\
                 &\mathit{assignExternal}(\mathtt{query(1)},\mathtt{t})      &28             \\
                 &\mathit{solve}((\emptyset,\emptyset))                      &29             \\
                 &\qquad\vdots                                               &               \\
                k&\mathit{releaseExternal}(\mathtt{query(k{-}1)})            &22             \\
                 &\mathit{ground}(((\mathtt{step},(k)),(\mathtt{check},(k))))&27             \\
                 &\mathit{assignExternal}(\mathtt{query(k)},\mathtt{t})      &28             \\
                 &\mathit{solve}((\emptyset,\emptyset))                      &29
  \end{array}
  \]
  \caption{Trace of Listing~\ref{lst:toh:run} in terms of operations}
  \label{fig:toh:operations}
\end{figure}
% ------------------------------------------------------------
%
We let \texttt{TOH} stand for the combination of programs \lstinline{tohI.lp} and \lstinline{tohE.lp} in Listing~\ref{fig:toh:enc} and~\ref{fig:toh:ins}.
Without setting any constants in Listing~\ref{fig:iclingo:python},
the sequence stops at the first $k \geq 0$ for which $\mathit{solve}((\emptyset,\emptyset))$ yields a stable model.
Each $k$-th invocation of $\mathit{solve}((\emptyset,\emptyset))$ is applied to a system state consisting of
\begin{enumerate}
\item the non-ground programs $R(\mathtt{base})$, $R(\mathtt{check})$, and $R(\mathtt{step})$,
\item the module obtained by
  \begin{enumerate}
  \item composing the ground subprograms of \texttt{base}, \texttt{check(0)},
    \par \texttt{check(l)}, and \texttt{step(l)} for $1\leq \texttt{l}\leq k$,
  \end{enumerate}
  having
  \begin{enumerate}\addtocounter{enumii}{1}
  \item the single input atom \texttt{query(k)},
    and
  \item output atoms stemming from
    \begin{enumerate}
    \item all ground rule heads in the subprograms and
    \item all released variables \texttt{query(l)} for $1\leq l\leq k$,
    \end{enumerate}
  \end{enumerate}
  and
\item a partial assignment mapping \texttt{query(k)} to true. % , and undefined otherwise.
\end{enumerate}
Note that all released atoms \texttt{query(l)} are undefined and set to false under stable models semantics.
Hence, among all instances of the integrity constraint in Line~\ref{fig:toh:enc:goal} in Listing~\ref{fig:toh:enc},
only the $k$-th one is effective.

%%% Local Variables:
%%% mode: latex
%%% TeX-master: "paper"
%%% End:

\subsection{\texorpdfstring{$n$}{n}-Queens problem}
\label{sec:queens}

In this section, we consider the well-known $n$-Queens problem.
However, in contrast to the classical setting,
we aim at solving series of problems of increasing size.

\subsubsection{Encoding incremental cardinality constraints }
\label{sec:cardinality:constraints}

The $n$-Queens problem can be expressed in terms of cardinality constraints,
that is,
there is exactly one queen per row and column, and
there is at most one queen per diagonal.
Hence,
for addressing this problem incrementally,
we have to encode such constraints in an incremental way.%
\footnote{The source code can also be found in \clingo's examples: \url{https://github.com/potassco/clingo/tree/master/examples/clingo/incqueens}}
To this end,
let us elaborate our encoding technique in a slightly simpler setting.
Let $1,\dots,n$ be a sequence of adjacent positions, such as a row, column, or diagonal,
and let $q_1, \dots, q_n$ be atoms indicating whether a queen is on position $1,\dots, n$
of such a sequence, respectively.%
\footnote{Such sequences are successively build for rows, columns, and diagonals via predicate \texttt{target}/6 in Listing~\ref{lst:queens}.}

We begin with a simple way to encode at-most-one constraints for an increasing set of positions $n$.
The corresponding program, $Q_i^{\leq1}$, is given in Listing~\ref{lst:at:most:one}.
\begin{lstlisting}[label=lst:at:most:one,caption={Incremental encoding of at-most-one constraints},language=clingo]
#program step(i).
{ q(i) }.
a(i) :- q(i-1).
a(i) :- a(i-1).
:- a(i), q(i).
\end{lstlisting}
We use \texttt{q(i)} to represent $q_i$ as well as auxiliary variables of form \texttt{a(i)} 
to indicate that position \texttt{i} is attacked by a queen on a position $j\leq i$.
The idea is to join the instantiation of $Q_n^{\leq1}$ with the previous program modules whenever a new position $n$ is added.
With this addition a queen may be put on position $n$ in Line~2.
Position $n$ is attacked if either the directly adjacent position or another connected position is occupied by a queen (Line~3 and~4).
Finally, a queen must not be placed on an attacked position (Line~5).

Let us make this precise by means of the operations introduced in Section~\ref{sec:semantics:operational}.
At first, $\mathit{create}(Q_i^{\leq1})$ yields a state comprising $R(\mathtt{step})$, an empty module, and an empty assignment.
Applying $\mathit{ground}((\mathtt{step}, (1)))$ to the resulting state yields the module
\begin{align*}
\mathbb{P}_1 = (\{\{q_1\}\leftarrow\},\emptyset,\{q_1\})
\end{align*}
and leaves $R(\mathtt{step})$ as well as the assignment intact.%
\footnote{Since neither is changed in the sequel, we concentrate on the evolution of the program module.}
Note that grounding $Q_1^{\leq1}$ relative to the empty atom base produces no instances of the rules in lines~3 to~5.

Applying $\mathit{ground}((\mathtt{step}, (2)))$ to the resulting state yields the module
\begin{align*}
  \mathbb{P}_2 
  &= 
  (\{\{q_1\}\leftarrow\},\emptyset,\{q_1\}) 
  \sqcup
  \left(\left\{\begin{aligned}
                 \{q_2\}&\leftarrow\\
                   a_2  &\leftarrow q_1\\
                        &\leftarrow a_2, q_2
               \end{aligned}
        \right\},
        \{q_1\},\{q_2,a_2\}
  \right)
                                     \\
  &=
  \left(\left\{\begin{aligned}
                 \{q_1\}&\leftarrow\\
                 \{q_2\}&\leftarrow\\
                   a_2  &\leftarrow q_1\\
                        &\leftarrow a_2, q_2
               \end{aligned}
        \right\},
        \emptyset,\{q_1,q_2,a_2\}
  \right)
\end{align*}
Grounding $Q_2^{\leq1}$ relative to the output atoms $\{q_1\}$ of $\mathbb{P}_1$ produces no instance of the rule in Line~4.

Each subsequent application of $\mathit{ground}((\mathtt{step},(n)))$ for $n\geq3$ yields the ground program in~\eqref{prg:queens:at:most:one:n}.
\begin{align}
  \{ q_1 \} &\leftarrow       &     &                   &     &                   & &                      \label{prg:queens:at:most:one:i} \\
  \{ q_2 \} &\leftarrow       & a_2 &\leftarrow q_1     &     &                   & & \leftarrow a_2, q_2  \label{prg:queens:at:most:one:ii} \\
  \{ q_3 \} &\leftarrow       & a_3 &\leftarrow q_2     & a_3 &\leftarrow a_2     & & \leftarrow a_3, q_3  \label{prg:queens:at:most:one:iii} \\
  {}        &\vdots           &     & \vdots            &     &\vdots             & & \vdots               \nonumber \\
  \{ q_n \} &\leftarrow       & a_n &\leftarrow q_{n-1} & a_n &\leftarrow a_{n-1} & & \leftarrow a_{n}, q_{n}\label{prg:queens:at:most:one:n}
\end{align}
Accordingly,
the corresponding join 
\(
\mathbb{P}_n = \mathbb{Q}_1^{\leq1} \sqcup \dots \sqcup \mathbb{Q}_n^{\leq1}
\)
comprises the union of the programs in~\eqref{prg:queens:at:most:one:i}, \eqref{prg:queens:at:most:one:ii}, 
and \eqref{prg:queens:at:most:one:iii} to~\eqref{prg:queens:at:most:one:n};
it has no inputs but outputs $\{q_1,\dots,q_n\} \cup \{a_2,\dots,a_n\}$.
% ------------------------------------------------------------
Then,
$X$ is a stable model of $\mathbb{P}_n$ iff $X=\emptyset$ or $X=\{q_i, a_{i+1}, \dots, a_n\}$ for some $1 \leq i \leq n$.
% ------------------------------------------------------------
%
This shows that $\mathbb{P}_n$ captures the set of all subsets of $\{q_1, \dots, q_n \}$ containing at most one $q_i$.

Let us now turn to an incremental encoding delineating all singletons in $\{q_1, \dots, q_n \}$.
Unlike above, the program, viz.\ $Q_i^{=1}$, in Listing~\ref{lst:exactly:one} uses external atoms to capture attacks from prospective board positions.
\begin{lstlisting}[label=lst:exactly:one,caption={Incremental encoding of exactly-one constraints},language=clingo]
#program step(i).
#external a(i).
{ q(i) }.
a(i-1) :- q(i).
a(i-1) :- a(i).
       :- a(i), q(i).
       :- not a(1), not q(1), i=1.
\end{lstlisting}
As in Listing~\ref{lst:at:most:one},
each instantiation of $Q_i^{=1}$ allows for placing a queen at position \texttt{i} or not.
Unlike there, however, attacks are now propagated in the opposite direction,
either by placing a queen at position \texttt{i} or an attack from a position beyond $i$.
The latter is indicated by the external atom \texttt{a(i)}, which becomes defined in $Q_{i+1}^{=1}$.
As in Listing~\ref{lst:at:most:one}, Line~6 denies an installation of a queen at \texttt{i} while it is attacked.

Applying $\mathit{ground}((\mathtt{step}, (1)))$ to the state resulting from $\mathit{create}(Q_i^{=1})$ yields the module
\begin{align*}
\mathbb{P}_1 = (\{\{q_1\}\leftarrow,\ \leftarrow a_1,q_1,\ \leftarrow \naf{a_1},\naf{q_1}\},\{a_1\},\{q_1\})
\end{align*}
No ground rule was produced from Line~4 and~5.
Note that the external declaration led to the input atom $a_1$.
Each subsequent application of $\mathit{ground}((\mathtt{step},(n)))$ for $n\geq2$ yields the ground program in~\eqref{prg:queens:exactly:one:n}.
\begin{align}\label{prg:queens:exactly:one:i}
  \{ q_1 \} &\leftarrow &         &                 &         &               & & \leftarrow a_1, q_1      & & \leftarrow {} \sim a_1, {} \sim q_1 \\
  \{ q_2 \} &\leftarrow & a_1     &\leftarrow q_2   & a_1     &\leftarrow a_2 & & \leftarrow a_2, q_2      & & \label{prg:queens:exactly:one:ii}   \\
  \{ q_3 \} &\leftarrow & a_2     &\leftarrow q_3   & a_2     &\leftarrow a_3 & & \leftarrow a_3, q_3      & &                           \nonumber \\
  {}        &\vdots     &         & \vdots          &         &\vdots         & & \vdots                   & &                           \nonumber \\
  \{ q_n \} &\leftarrow & a_{n-1} &\leftarrow q_{n} & a_{n-1} &\leftarrow a_n & & \leftarrow a_{n}, q_{n}  & &\label{prg:queens:exactly:one:n}
\end{align}
Accordingly,
the corresponding join 
\(
\mathbb{P}_n = \mathbb{Q}_1^{=1} \sqcup \dots \sqcup \mathbb{Q}_n^{=1}
\)
comprises the union of the programs in~\eqref{prg:queens:exactly:one:i}, and \eqref{prg:queens:exactly:one:ii} to~\eqref{prg:queens:exactly:one:n};
it has input $\{a_n\}$ and outputs $\{q_1,\dots,q_n\} \cup \{a_1,\dots,a_{n-1}\}$.
%
% ------------------------------------------------------------
Then,
$X$ is a stable model of $\mathbb{P}_n$ iff $X=\{a_{1}, \dots, a_{i-1}, q_i\}$ for some $1 \leq i \leq n$.
% ------------------------------------------------------------
%
That is, the stable models of $\mathbb{P}_n$ are in a one-to-one correspondence to one element subsets of $\{q_1, \dots, q_n \}$.

Interestingly,
the last encoding can be turned into one for an at-most one-constraint by omitting Line~7,
and into an at-least-one constraint by removing Line~6.
Note that the integrity constraint in Line 7 cannot simply be added to the encoding in Listing~\ref{lst:at:most:one} because it encodes the attack direction the other way round.
One could add `\texttt{:- not a(i), not q(i), query(i).}' subject to the query atom.
But this has the disadvantage that the constraint would have to be retracted whenever the query atom becomes permanently false.
The encoding in Listing~\ref{lst:exactly:one} ensures that all constraints in the solver (including learnt constraints) can be reused in successive solving steps.

Finally, note that by using \texttt{step}, both encodings can by used with the built-in incremental mode,
described in the previous section.

\subsubsection{An incremental encoding}
\label{sec:incremental:queens}

In what follows,
we use the above encoding schemes to model the $n$-Queens problem.
As mentioned,
we aim at solving series of differently sized boards.
Given that larger boards subsume smaller ones, an evolving problem
specification can reuse ground rules from previous~\clingo{} states
when the size increases.
To this end, we view the increment of~$n$ to~$n+1$ as the addition of
one more row and column.
The basic idea of our incremental encoding is to interconnect the previous and added
board cells so that each of them has a unique predecessor or successor
in either of the four attack directions of queens, respectively.
Each such connection scheme amounts to a sequence of adjacent positions,
as used above in Section~\ref{sec:cardinality:constraints}.
The four schemes obtained in our setting 
are depicted in Figure~\ref{fig:queens}\subref{fig:queens:b}--\subref{fig:queens:v}.
Direct links are indicated by arrows to target cells with a (white or black) circle.
Paths represent the respective ways of attack across several board extensions.
They concretise the sequences discussed in the previous section.
Attacks from prospective board positions are indicated by black circles
(and implemented as external atoms),
the ones from the board by white ones.
Figure~\ref{fig:queens}\subref{fig:queens:b} illustrates the scheme for backward diagonals.
It connects cells of the uppermost previous row to corresponding attacked cells in a new column;
the latter are in turn linked to the new cells they attack in the row above.
For instance,
position $(2,2)$ is linked to $(3,1)$ which is itself linked to $(1,3)$.
Note that this scheme ensures that,
starting from the middle of any backward diagonal,
all cells that are successively added (and belong to the same backward diagonal) are on a path.
Such a path follows the board evolution and is directed from
previous to newly added cells,
where white circles in arrow targets indicate the
presence of attacking cells on the board when their respective target cells are added.
This orientation is analogous to the above encoding of at-most-one constraints.
The schemes for attacks along forward diagonals, horizontal rows, and
vertical columns are shown in Figure~\ref{fig:queens}\subref{fig:queens:f},
\subref{fig:queens:h}, and~\subref{fig:queens:v}.
Notably, the latter two (partially) link new cells to previous ones,
in which case the targets are highlighted by black circles.
At the level of modules,
links from cells that may be added later on give rise to input atoms.
% The underlying strategy is to make the information whether a row or column
% contains some queen available at the very first position.

As mentioned,
the $n$-Queens problem can be expressed in terms of cardinality constraints,
requiring that there is 
exactly one queen per row and column and
at most one per diagonal.
The idea is to use the incremental encodings from Section~\ref{sec:cardinality:constraints} in combination with the four connection schemes
depicted in Figure~\ref{fig:queens}\subref{fig:queens:b}--\subref{fig:queens:v}.
The first encoding, capturing at-most-one constraints, is used for each diagonal,
and the second one, handling exactly-one constraints, for each row and column.
The resulting \clingo{} encoding is given in Listing~\ref{lst:queens}.
Let us first outline its structure in relation to the ideas presented in Section~\ref{sec:cardinality:constraints}.
The rules in lines~7-13 gather linked positions for at-most- and exactly-one constraints in the predicate \texttt{target}/6.
The choice rule in Line~15 places queens on the new column and row.
The rules in lines~17 and~18 determine which cells are attacked.
The rule in line~20 ensures the at-most-one condition for rows, columns, and diagonals.
And finally the rules in lines 22-23 ensure the at-least-one condition for rows and columns.

Let us make this precise in what follows.
%
% ------------------------------------------------------------
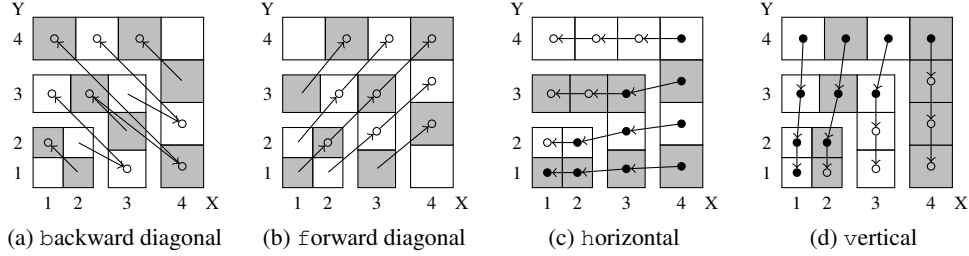
\begin{figure}[t]
  \centering
  \captionsetup[subfigure]{justification=centering,subrefformat=parens}
  \begin{subfigure}[b]{.24\linewidth}
    \centering
\begin{tikzpicture}[
      x=.2cm,y=.2cm,
      axis/.style={inner sep=0,font=\scriptsize},
      tile1/.style={rectangle,inner sep=0,minimum size=.4cm,draw=black},
      tile2/.style={rectangle,inner sep=0,minimum size=.5cm,draw=black},
      tile3/.style={rectangle,inner sep=0,minimum size=.5666666666cm,draw=black},
      dot/.style={ellipse,inner sep=0,minimum size=.1cm,draw=black},
      empty/.style={inner sep=0,minimum size=0}, 
      edge/.style={->}]
  \draw
    (1,1) 
      node[tile1] { }
    +(-2,0)
      node[axis] { 1 }
    +(0,-2)
      node[axis] { 1 }
    ++(2,0) 
      node[tile1,fill=lightgray] { }
      node[empty] (n21) {  }
    +(0,-2)
      node[axis] { 2 }
    ++(0,2) 
      node[empty] (n22) {  }
      node[tile1] { }
    ++(-2,0) 
      node[tile1,fill=lightgray] { }
      node[dot] (n12) { }
    +(-2,0)
      node[axis] { 2 }
    ;
  \draw
    (6.25,1.25)
      node[dot] (n31) { }
      node[tile2] { }
    +(0,-2.25)
      node[axis] { 3 }
    ++(0,2.5)
      node[tile2,fill=lightgray] { }
      node[empty] (n32) { }
    ++(0,2.5)
      node[empty] (n33) { }
      node[tile2] { }
    ++(-2.5,0)
      node[tile2,fill=lightgray] { }
      node[dot] (n23) { }
    ++(-2.5,0)
      node[dot] (n13) { }
      node[tile2] { }
    +(-2.25,0)
      node[axis] { 3 }
    ;
  \draw
    (9.9166666666,1.4166666666)
      node[tile3,fill=lightgray] { }
      node[dot] (n41) { }
    +(0,-2.4166666666)
      node[axis] { 4 }
    +(2,-2.4166666666)
      node[axis] { X }
    ++(0,2.8333333333)
      node[dot] (n42) { }
      node[tile3] { }
    ++(0,2.8333333333)
      node[tile3,fill=lightgray] { }
      node[empty] (n43) { }
    ++(0,2.8333333333)
      node[empty] (n44) { }
      node[tile3] { }
    ++(-2.8333333333,0)
      node[tile3,fill=lightgray] { }
      node[dot] (n34) { }
    ++(-2.8333333333,0)
      node[dot] (n24) { }
      node[tile3] { }
    ++(-2.8333333333,0)
      node[tile3,fill=lightgray] { }
      node[dot] (n14) { }
    +(-2.4166666666,0)
      node[axis] { 4 }
    +(-2.4166666666,2)
      node[axis] { Y }
    ;
  \path[edge]
    (n21) edge (n12)

    (n22) edge (n31)
    (n31) edge (n13)
    (n32) edge (n23)

    (n23) edge (n41)
    (n41) edge (n14)
    (n33) edge (n42)
    (n42) edge (n24)
    (n43) edge (n34);
\end{tikzpicture}
\caption{\lstinline{b}ackward diagonal}
\label{fig:queens:b}
\end{subfigure}
\begin{subfigure}[b]{.24\linewidth}
\centering
\begin{tikzpicture}[
      x=.2cm,y=.2cm,
      axis/.style={inner sep=0,font=\scriptsize},
      tile1/.style={rectangle,inner sep=0,minimum size=.4cm,draw=black},
      tile2/.style={rectangle,inner sep=0,minimum size=.5cm,draw=black},
      tile3/.style={rectangle,inner sep=0,minimum size=.5666666666cm,draw=black},
      dot/.style={ellipse,inner sep=0,minimum size=.1cm,draw=black},
      empty/.style={inner sep=0,minimum size=0}, 
      edge/.style={->}]
  \draw
    (1,1) 
      node[tile1,fill=lightgray] { }
      node[empty] (n11) {  }
    +(-2,0)
      node[axis] { 1 }
    +(0,-2)
      node[axis] { 1 }
    ++(2,0) 
      node[empty] (n21) {  }
      node[tile1] { }
    +(0,-2)
      node[axis] { 2 }
    ++(0,2) 
      node[tile1,fill=lightgray] { }
      node[dot] (n22) {  }
    ++(-2,0) 
      node[empty] (n12) {  }
      node[tile1] { }
    +(-2,0)
      node[axis] { 2 }
    ;
  \draw
    (6.25,1.25)
      node[tile2,fill=lightgray] { }
      node[empty] (n31) { }
    +(0,-2.25)
      node[axis] { 3 }
    ++(0,2.5)
      node[dot] (n32) { }
      node[tile2] { }
    ++(0,2.5)
      node[tile2,fill=lightgray] { }
      node[dot] (n33) { }
    ++(-2.5,0)
      node[dot] (n23) { }
      node[tile2] { }
    ++(-2.5,0)
      node[tile2,fill=lightgray] { }
      node[empty] (n13) { }
    +(-2.25,0)
      node[axis] { 3 }
;
  \draw
    (9.9166666666,1.4166666666)
      node[empty] (n41) { }
      node[tile3] { }
    +(0,-2.4166666666)
      node[axis] { 4 }
    +(2,-2.4166666666)
      node[axis] { X }
    ++(0,2.8333333333)
      node[tile3,fill=lightgray] { }
      node[dot] (n42) { }
    ++(0,2.8333333333)
      node[dot] (n43) { }
      node[tile3] { }
    ++(0,2.8333333333)
      node[tile3,fill=lightgray] { }
      node[dot] (n44) { }
    ++(-2.8333333333,0)
      node[dot] (n34) { }
      node[tile3] { }
    ++(-2.8333333333,0)
      node[tile3,fill=lightgray] { }
      node[dot] (n24) { }
    ++(-2.8333333333,0)
      node[empty] (n14) { }
      node[tile3] { }
    +(-2.4166666666,0)
      node[axis] { 4 }
    +(-2.4166666666,2)
      node[axis] { Y }
    ;
  \path[edge]
    (n11) edge (n22)

    (n12) edge (n23)
    (n22) edge (n33)
    (n21) edge (n32)

    (n13) edge (n24)
    (n23) edge (n34)
    (n31) edge (n42)
    (n32) edge (n43)
    (n33) edge (n44);

\end{tikzpicture}
\caption{\lstinline{f}orward diagonal}
\label{fig:queens:f}
\end{subfigure}
\begin{subfigure}[b]{.24\linewidth}
\centering
\begin{tikzpicture}[
      x=.2cm,y=.2cm,
      axis/.style={inner sep=0,font=\scriptsize},
      tile1/.style={rectangle,inner sep=0,minimum size=.4cm,draw=black},
      tile2/.style={rectangle,inner sep=0,minimum size=.5cm,draw=black},
      tile3/.style={rectangle,inner sep=0,minimum size=.5666666666cm,draw=black},
      dot/.style={ellipse,inner sep=0,minimum size=.1cm,draw=black},
      ext/.style={ellipse,inner sep=0,minimum size=.1cm,draw=black,fill},
      empty/.style={inner sep=0,minimum size=0}, 
      edge/.style={->}]
  \draw
    (1,1) 
      node[tile1,fill=lightgray] { }
      node[ext] (n11) {  }
    +(-2,0)
      node[axis] { 1 }
    +(0,-2)
      node[axis] { 1 }
    ++(2,0) 
      node[tile1,fill=lightgray] { }
      node[ext] (n21) {  }
    +(0,-2)
      node[axis] { 2 }
    ++(0,2) 
      node[ext] (n22) {  }
      node[tile1] { }
    ++(-2,0) 
      node[dot] (n12) {  }
      node[tile1] { }
    +(-2,0)
      node[axis] { 2 }
    ;
  \draw
    (6.25,1.25)
      node[tile2,fill=lightgray] { }
      node[ext] (n31) { }
    +(0,-2.25)
      node[axis] { 3 }
    ++(0,2.5)
      node[ext] (n32) { }
      node[tile2] { }
    ++(0,2.5)
      node[tile2,fill=lightgray] { }
      node[ext] (n33) { }
    ++(-2.5,0)
      node[tile2,fill=lightgray] { }
      node[dot] (n23) { }
    ++(-2.5,0)
      node[tile2,fill=lightgray] { }
      node[dot] (n13) { }
    +(-2.25,0)
      node[axis] { 3 }
    ;
  \draw
    (9.9166666666,1.4166666666)
      node[tile3,fill=lightgray] { }
      node[ext] (n41) { }
    +(0,-2.4166666666)
      node[axis] { 4 }
    +(2,-2.4166666666)
      node[axis] { X }
    ++(0,2.8333333333)
      node[ext] (n42) { }
      node[tile3] { }
    ++(0,2.8333333333)
      node[tile3,fill=lightgray] { }
      node[ext] (n43) { }
    ++(0,2.8333333333)
      node[ext] (n44) { }
      node[tile3] { }
    ++(-2.8333333333,0)
      node[dot] (n34) { }
      node[tile3] { }
    ++(-2.8333333333,0)
      node[dot] (n24) { }
      node[tile3] { }
    ++(-2.8333333333,0)
      node[dot] (n14) { }
      node[tile3] { }
    +(-2.4166666666,0)
      node[axis] { 4 }
    +(-2.4166666666,2)
      node[axis] { Y }
    ;
  \path[edge]
    (n44) edge (n34)
    (n34) edge (n24)
    (n24) edge (n14)

    (n43) edge (n33)
    (n33) edge (n23)
    (n23) edge (n13)

    (n42) edge (n32)
    (n32) edge (n22)
    (n22) edge (n12)

    (n41) edge (n31)
    (n31) edge (n21)
    (n21) edge (n11)
    ;

\end{tikzpicture}
\caption{\lstinline{h}orizontal}
\label{fig:queens:h}
\end{subfigure}
\begin{subfigure}[b]{.24\linewidth}
\centering
\begin{tikzpicture}[
      x=.2cm,y=.2cm,
      axis/.style={inner sep=0,font=\scriptsize},
      tile1/.style={rectangle,inner sep=0,minimum size=.4cm,draw=black},
      tile2/.style={rectangle,inner sep=0,minimum size=.5cm,draw=black},
      tile3/.style={rectangle,inner sep=0,minimum size=.5666666666cm,draw=black},
      dot/.style={ellipse,inner sep=0,minimum size=.1cm,draw=black},
      ext/.style={ellipse,inner sep=0,minimum size=.1cm,draw=black,fill},
      empty/.style={inner sep=0,minimum size=0}, 
      edge/.style={->}]
  \draw
    (1,1) 
      node[ext] (n11) {  }
      node[tile1] { }
    +(-2,0)
      node[axis] { 1 }
    +(0,-2)
      node[axis] { 1 }
    ++(2,0) 
      node[tile1,fill=lightgray] { }
      node[dot] (n21) {  }
    +(0,-2)
      node[axis] { 2 }
    ++(0,2) 
      node[tile1,fill=lightgray] { }
      node[ext] (n22) {  }
    ++(-2,0) 
      node[ext] (n12) {  }
      node[tile1] { }
    +(-2,0)
      node[axis] { 2 }
    ;
  \draw
    (6.25,1.25)
      node[dot] (n31) { }
      node[tile2] { }
    +(0,-2.25)
      node[axis] { 3 }
    ++(0,2.5)
      node[dot] (n32) { }
      node[tile2] { }
    ++(0,2.5)
      node[ext] (n33) { }
      node[tile2] { }
    ++(-2.5,0)
      node[tile2,fill=lightgray] { }
      node[ext] (n23) { }
    ++(-2.5,0)
      node[ext] (n13) { }
      node[tile2] { }
    +(-2.25,0)
      node[axis] { 3 }
    ;
  \draw
    (9.9166666666,1.4166666666)
      node[tile3,fill=lightgray] { }
      node[dot] (n41) { }
    +(0,-2.4166666666)
      node[axis] { 4 }
    +(2,-2.4166666666)
      node[axis] { X }
    ++(0,2.8333333333)
      node[tile3,fill=lightgray] { }
      node[dot] (n42) { }
    ++(0,2.8333333333)
      node[tile3,fill=lightgray] { }
      node[dot] (n43) { }
    ++(0,2.8333333333)
      node[tile3,fill=lightgray] { }
      node[ext] (n44) { }
    ++(-2.8333333333,0)
      node[ext] (n34) { }
      node[tile3] { }
    ++(-2.8333333333,0)
      node[tile3,fill=lightgray] { }
      node[ext] (n24) { }
    ++(-2.8333333333,0)
      node[ext] (n14) { }
      node[tile3] { }
    +(-2.4166666666,0)
      node[axis] { 4 }
    +(-2.4166666666,2)
      node[axis] { Y }
    ;
  \path[edge]
    (n44) edge (n43)
    (n43) edge (n42)
    (n42) edge (n41)

    (n34) edge (n33)
    (n33) edge (n32)
    (n32) edge (n31)
                   
    (n24) edge (n23)
    (n23) edge (n22)
    (n22) edge (n21)
                   
    (n14) edge (n13)
    (n13) edge (n12)
    (n12) edge (n11)
    ;

\end{tikzpicture}
\caption{\lstinline{v}ertical}
\label{fig:queens:v}
\end{subfigure}

%\negsubfig
\caption{Attack target links among cells of successive $n$-Queens boards up to size~$4$}
\label{fig:queens}
\end{figure}
%
%%% Local Variables: 
%%% mode: latex
%%% TeX-master: "paper"
%%% End: 

%
\begin{lstlisting}[float=t,basicstyle=\ttfamily\small,linerange={1-35},caption={\clingo{} program for successive $n$-Queens solving (\lstinline{queens.lp})},label=lst:queens,language=clingo]
#show queen/2.

#program board(n).
#external attack(n,1..n,h).
#external attack(1..n,n,v).

target(n,  X,  X,  n,  b,n) :- X = 1..n-1.              % diagonal b
target(Y,  n-1,n,  Y-1,b,n) :- Y = 2..n-1.              % diagonal b
target(X,  n-1,X+1,n,  f,n) :- X = 1..n-1.              % diagonal f
target(n-1,Y,  n,  Y+1,f,n) :- Y = 1..n-2.              % diagonal f
target(X,  n,  X-1,n,  h,n) :- X = 2..n.                % horizontal
target(n,  Y,  n-1,Y,  h,n) :- Y = 1..n-1.              % horizontal
target(Y,  X,  Y,  X-1,v,n) :- target(X,Y,X-1,Y,h,n).   % vertical

{ queen(1..n,n); queen(n,1..n-1) }.

attack(X',Y',D) :- target(X,Y,X',Y',D,n), queen(X,Y).
attack(X',Y',D) :- target(X,Y,X',Y',D,n), attack(X,Y,D).

:- target(X,Y,X',Y',D,n), attack(X',Y',D), queen(X',Y').

:- not queen(1,n), not attack(1,n,h).
:- not queen(n,1), not attack(n,1,v).

#script(python)
def main(prg):
    n = 0
    for bound in prg.get_const("calls").arguments:
        while n < bound.arguments[1].number:
            n += 1
            prg.ground([("board", [n])])
            if n >= bound.arguments[0].number:
                print('SIZE {0}'.format(n))
                prg.solve()
#end.
\end{lstlisting}
% ------------------------------------------------------------
%
After declaring \lstinline{queen/2} as the output predicate to be displayed,
the (sub)program \lstinline{board(n)} provides rules for extending a
board of size $n{-}1$ to~$n\geq 1$.
To this end, the \lstinline{#external} directives in Line~4 and~5 declare atoms
representing horizontal and vertical attacks on cells in the $n$-th column or row as inputs, respectively.
This is analogous to the use of external atoms in Listing~\ref{lst:exactly:one}.
Such atoms match the targets of arrows leading to cells with black circles
in Figure~\ref{fig:queens}\subref{fig:queens:h} and~\subref{fig:queens:v}.
For instance,
\lstinline{attack(2,1,h)} and \lstinline{attack(2,2,h)}
as well as
\lstinline{attack(1,2,v)} and \lstinline{attack(2,2,v)}
are the inputs to \lstinline{board(n}$/$\lstinline{2)}.
These external atoms express that cells at the horizontal and vertical borders
can become targets of attacks once the board is extended beyond size~$2$.
The instances of \lstinline{target(X,Y,X',Y',D,n)} specified in Line~7--13
provide links from cells $($\lstinline{X}$,$\lstinline{Y}$)$ to targets
$($\lstinline{X'}$,$\lstinline{Y'}$)$ along with directions~\lstinline{D}
leading from or to some newly added cell in the $n$-th column or row.
These instances correspond to arrows shown in
Figure~\ref{fig:queens}\subref{fig:queens:b}--\subref{fig:queens:v},
yet omitting those to border cells
such that \lstinline{attack(X',Y',D)} is declared as input in Line~4 and~5,
also highlighted by black circles
in Figure~\ref{fig:queens}\subref{fig:queens:h} and~\subref{fig:queens:v}.
Queens at newly added cells % $($\lstinline{X}$,$\lstinline{Y}$)$
in the $n$-th column or row are enabled via the choice rule in Line~15,
and the links provided by instances of \lstinline{target(X,Y,X',Y',D,n)}
are utilized in Line~17 and~18 for deriving \lstinline{attack(X',Y',D)}
in view of a queen at cell $($\lstinline{X}$,$\lstinline{Y}$)$ or
any of its predecessors in the direction indicated by~\lstinline{D}.
For instance, the following ground rules,
simplified by dropping atoms of the domain predicate
\lstinline{target/6},
capture horizontal attacks
along the first row of a board of size~$4$:
\begin{lstlisting}[xleftmargin=2\parindent,numbers=none,basicstyle=\ttfamily\small,language=clingo]
attack(1,1,h) :- queen(2,1).   attack(1,1,h) :- attack(2,1,h).
attack(2,1,h) :- queen(3,1).   attack(2,1,h) :- attack(3,1,h).
attack(3,1,h) :- queen(4,1).   attack(3,1,h) :- attack(4,1,h).
\end{lstlisting}
Note that a queen represented by an instance of \lstinline{queen(X,1)},
for $2\leq{}$\lstinline{X}${}\leq 4$, propagates to cells on its left
via an implication chain deriving \lstinline{attack(X',1,h)} for every
$1\leq{}$\lstinline{X'}${}<{}$\lstinline{X}.
Moreover, the fact that the cell at $($\lstinline{4}$,$\lstinline{1}$)$
can be attacked from the right when increasing the board size is reflected
by the input atom \lstinline{attack(4,1,h)} declared in
\lstinline{board(n}$/$\lstinline{4)}.
Given that attacks are propagated analogously for other rows and directions,
instances of the integrity constraint in Line~20
prohibit a queen at cell $($\lstinline{X'}$,$\lstinline{Y'}$)$
whenever \lstinline{attack(X',Y',D)} signals
that some predecessor in either direction~\lstinline{D}
has a queen already.
The integrity constraints in Line~22 and~23 additionally require that each row
and column contains some queen.
In view of the orientations of horizontal and vertical links, as
displayed in Figure~\ref{fig:queens}\subref{fig:queens:h} and~\subref{fig:queens:v},
non-emptiness can be recognized from a queen at or an attack propagated to
the first position in a row or column, no matter to which size the board is
extended later on.
The incremental development of these rules is illustrated in Figure~\ref{fig:inc:row}
for the first row of a board (abbreviating predicates by their first letter).
% ------------------------------------------------------------
\begin{sidewaysfigure}
\vspace{10cm}%\centering
\newcommand{\attack}{a}
\newcommand{\queen}{q}
\newcommand{\target}{t}
\newcommand{\external}{external}
\newcommand{\VS}{\\\phantom{()}}
{\footnotesize\renewcommand{\arraystretch}{3}%
\begin{tabular}{lllllll}
  Line~4                                                 & Line~12                          & Line~15                   & Line~17                                                                                                  & Line~18                                                                                                     & Line~20                                                                                                 & Line~22 \\
  \texttt{\shortstack[l]{\#\external\\~\attack(1,1,h).}} &                                  & \texttt{\shortstack[l]{\{\queen(1,1)\}.\VS}} &                                                                                                          &                                                                                                             &                                                                                                         & \texttt{\shortstack[l]{:-~not~\queen(1,1),~\\~~~not~\attack(1,1,h).}} \\
  \texttt{\shortstack[l]{\#\external\\~\attack(2,1,h).\VS}} & \texttt{\shortstack[l]{\target(2,1,1,1,h,2).\VS\VS}}   & \texttt{\shortstack[l]{\{\queen(2,1)\}.\VS\VS}} & \texttt{\shortstack[l]{\attack(1,1,h)~:-\\~{\color{lightgray}\target(2,1,1,1,h,2)},\\~\queen(2,1).}}     & \texttt{\shortstack[l]{\attack(1,1,h)~:-\\~{\color{lightgray}\target(2,1,1,1,h,2)},~\\\attack(2,1,h).}}     & \texttt{\shortstack[l]{:-~{\color{lightgray}\target(2,1,1,1,h,2)},\\~~~\attack(1,1,h),~\queen(1,1).\VS}}   & \\
  \vdots                                                 & \vdots                           & \vdots                    & \vdots                                                                                                   & \vdots                                                                                                      & \vdots                                                                                                  & \\
  \texttt{\shortstack[l]{\#\external\\~\attack(n,1,h).\VS}} & \texttt{\shortstack[l]{\target(n,1,n-1,1,h,n).\VS\VS}} & \texttt{\shortstack[l]{\{\queen(n,1)\}.\VS\VS}} & \texttt{\shortstack[l]{\attack(n-1,1,h)~:-\\~{\color{lightgray}\target(n,1,n-1,1,h,n)},\\~\queen(n,1).}} & \texttt{\shortstack[l]{\attack(n-1,1,h)~:-\\~{\color{lightgray}\target(n,1,n-1,1,h,n)},\\~\attack(n,1,h).}} & \texttt{\shortstack[l]{:-~{\color{lightgray}\target(n,1,n-1,1,h,n)},\\~~~\attack(n,1,h),~\queen(n,1).\VS}} & \\
\end{tabular}}  
\caption{Incremental development of rules regarding the first row of a board}
\label{fig:inc:row}
\end{sidewaysfigure}
%%% Local Variables:
%%% mode: latex
%%% TeX-master: "paper"
%%% End:

% ------------------------------------------------------------
It is interesting to observe that the rules generated at each step in lines~15 to~22
correspond to the ones in~\eqref{prg:queens:exactly:one:i} to~\eqref{prg:queens:exactly:one:n}.

Importantly, instantiations of \lstinline{board(n)}
with different integers for~\lstinline{n} define distinct
(ground) atoms, and the non-circularity of paths according to the
connection schemes in Figure~\ref{fig:queens}\subref{fig:queens:b}--\subref{fig:queens:v}
excludes
mutual positive dependencies (between instances of \lstinline{attack(X',Y',D)}).
Hence, the modules induced by different instantiations of \lstinline{board(n)}
are compositional and can be joined to successively increase the board size.

The \python{} \lstinline{main} routine in Line~25--35
of Listing~\ref{lst:queens} controls
the successive grounding and solving of a series of boards.
To this end, an ordered list of non-overlapping integer intervals is to be provided on
the command-line.
For example, \lstinline{ -c calls="list((1,1),(3,5),(8,9))"}
leads to successively solving the $n$-Queens problem for board sizes 1, 3, 4, 5, 8, and 9.
As long as the upper limit of some interval is yet unreached,
the board size is incremented by one in Line~30 and,
in view of the
\lstinline{ground} instruction in Line~31, taken as a term
for instantiating \lstinline{board(n)}. % with.
However, solving is only invoked in Line~34 if the current size lies
within the interval of interest.
Provided that this is the case for any particular~$n\geq 1$,
the sequence of issued \lstinline{ground} instructions makes sure
that the current \clingo{} state corresponds to the module
obtained by instantiating and joining the subprograms
\lstinline{board(n}$/i$\lstinline{)}, for $1\leq i\leq n$,
in increasing order.
Since all ground rules accumulated in such a state are relevant
(and not superseded by permanently falsifying the body) for
$n$-Queens solving,
there is no redundancy in instantiating
\lstinline{board(n}$/i$\lstinline{)} for each $1\leq i\leq n$,
even when the provided integer intervals do not include~$i$
and $i$-Queens solving is skipped.
For instance, \lstinline{ -c calls="list((1,1),(3,5),(8,9))"}
specifies a series of six boards to solve, while the
subprogram \lstinline{board(n)} is successively
instantiated with nine different terms for parameter~\lstinline{n}.
In fact, the \lstinline{main} routine in Line~25--35
% $n$-Queens solving
automates the
assembly of subprograms needed to process an arbitrary
yet increasing sequence of board sizes.

The result of running the program in Listing~\ref{lst:queens} with \clingo{} is given in Listing~\ref{lst:queen:run}.
% ------------------------------------------------------------
\begin{lstlisting}[numbers=none,basicstyle=\ttfamily\small,caption={Running the program in Listing~\ref{lst:queens} with \clingo},label={lst:queen:run}]
$ clingo queens.lp -c calls="list((1,1),(3,5),(8,9))"
clingo version 4.5.4
Reading from queen.alt.lp
SIZE 1
Solving...
Answer: 1
queen(1,1)
SIZE 3
Solving...
SIZE 4
Solving...
Answer: 1
queen(2,1) queen(1,3) queen(4,2) queen(3,4)
SIZE 5
Solving...
Answer: 1
queen(2,1) queen(3,3) queen(1,4) queen(5,2) queen(4,5)
SIZE 8
Solving...
Answer: 1
queen(2,1) queen(1,4) queen(5,2) queen(3,5) queen(7,3) \
queen(6,7) queen(8,6) queen(4,8)
SIZE 9
Solving...
Answer: 1
queen(3,2) queen(1,3) queen(6,4) queen(5,6) queen(7,1) \
queen(2,7) queen(8,5) queen(4,8) queen(9,9)
SATISFIABLE

Models     : 5+    
Calls      : 6
Time       : 0.013s (Solving: 0.00s 1st Model: 0.00s Unsat: 0.00s)
CPU Time   : 0.010s
\end{lstlisting}
% ------------------------------------------------------------

%%% Local Variables:
%%% mode: latex
%%% TeX-master: "paper"
%%% End:

\subsection{Ricochet Robots}
\label{sec:robots}

In practice, ASP systems are embedded in encompassing software environments and thus need means for interaction.
Multi-shot ASP solvers can address this by allowing a reactive procedure to loop on solving while acquiring changes to the problem specification.
In this section,
we want to illustrate this by modeling the popular board game of \emph{Ricochet Robots}.
Our particular focus lies on capturing the underlying round playing through the procedural-declarative interplay
offered by \clingo.

\emph{Ricochet Robots} is a board game for multiple players designed by Alex Randolph.%
\footnote{\url{http://en.wikipedia.org/wiki/Ricochet_Robot}}
A board consists of 16$\times$16 fields arranged in a grid structure having barriers between various
neighboring fields (see Figure~\ref{fig:rr:goal13} and~\ref{fig:rr:goal4}).
Four differently colored robots roam across the board along either horizontally or vertically
accessible fields, respectively.
Each robot can thus move in four directions.
A robot cannot stop its move until it either hits a barrier or another robot.
The goal is to place a designated robot on a target location with a shortest sequence of moves.
Often this involves moving several robots to establish temporary barriers.
The game is played in rounds.
At each round, a chip with a colored symbol indicating the target location is drawn.
Then, the specific goal is to move the robot with the same color on this location.
The player who reaches the goal with the fewest number of robot moves wins the chip.
The next round is then played from the end configuration of the previous round.
At the end, the player with most chips wins the game.

\subsubsection{Encoding {Ricochet Robots}}

The following encoding%
\footnote{%
Alternative ASP encodings of the game were studied by~\citeN{gejokaobsascsc13a},
and used for comparing various ASP solving techniques.
More disparate encodings resulted from the ASP competition in 2013,
where \emph{Ricochet Robots} was included in the modeling track.
ASP encodings and instances of \emph{Ricochet Robots} are available at~\url{https://potassco.org/doc/apps/2016/09/20/ricochet-robots.html}.
There is also visualizer for the problem among the clingo examples: \url{https://github.com/potassco/clingo/tree/master/examples/clingo/robots}}
and fact format follow the ones of~\citeN{gejokaobsascsc13a}.

An authentic board configuration of \emph{Ricochet Robots} is shown in Figure~\ref{fig:rr:goal13}
and represented as facts in Listing~\ref{lst:lp:board}.
% ------------------------------------------------------------
\begin{figure}[ht]
  \centering
  \includegraphics[width=0.45\textwidth]{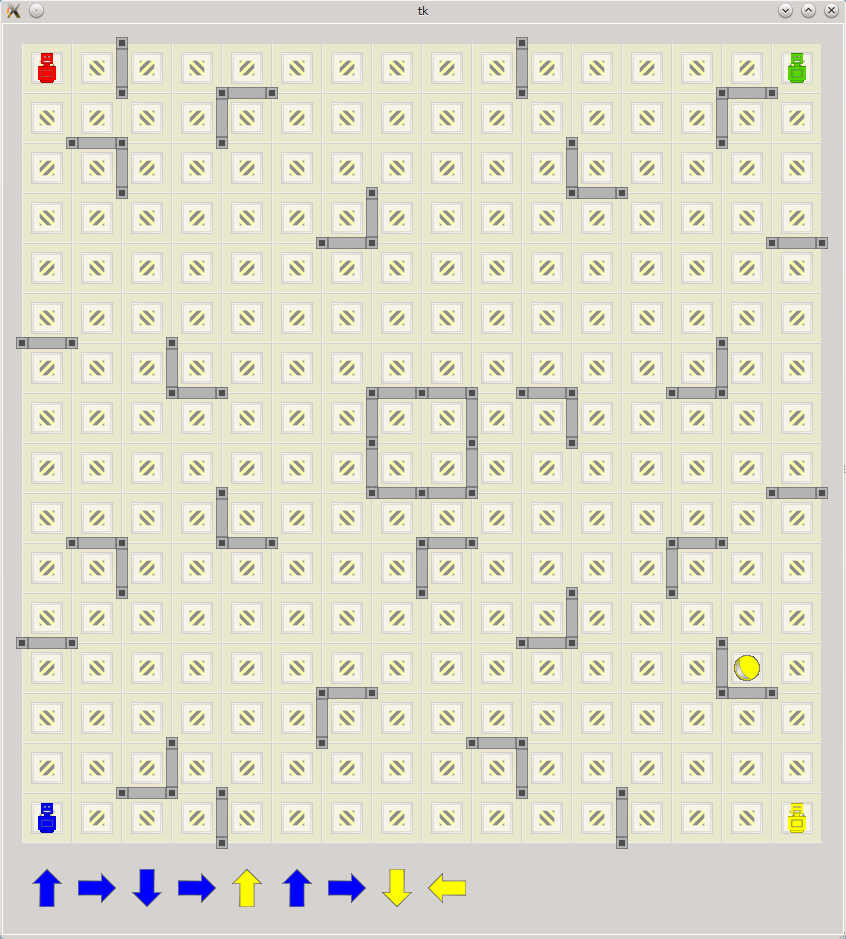}
  \hspace{0.08\textwidth}
  \includegraphics[width=0.45\textwidth]{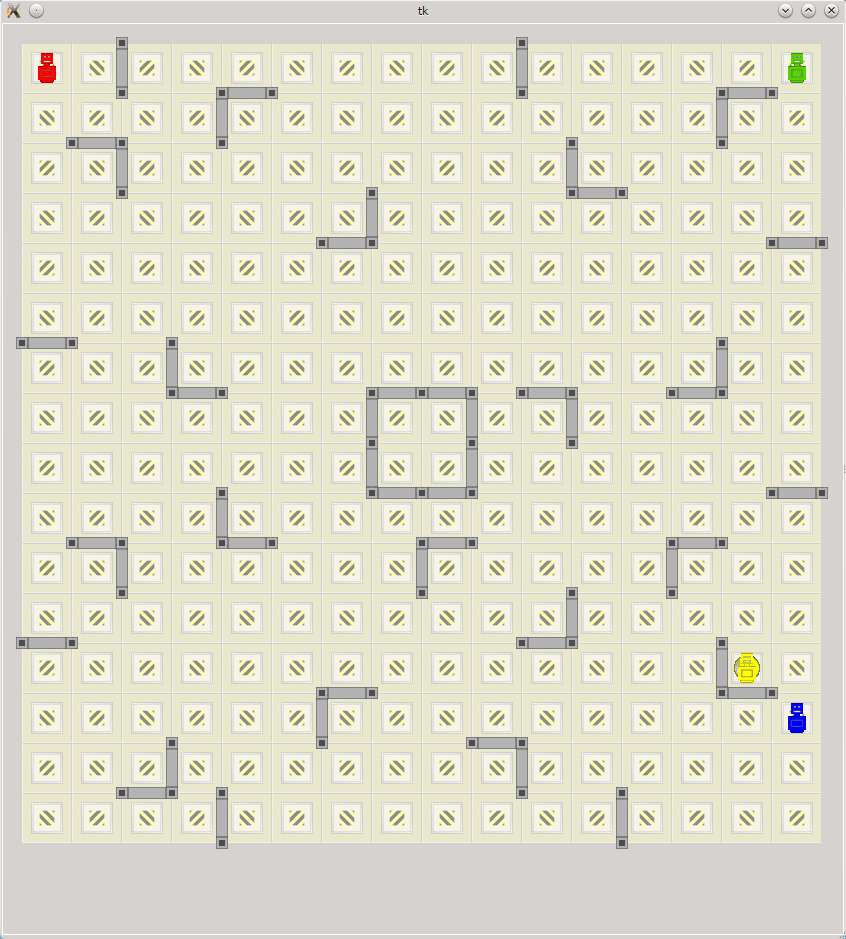}
  \caption{Visualization of solving \lstinline{goal(13)} from initially cornered robots}
  \label{fig:rr:goal13}
\end{figure}%
% ------------------------------------------------------------
%
The dimension of the board is fixed to $16$ in Line~1.
As put forward by~\citeN{gejokaobsascsc13a},
barriers are indicated by atoms with predicate \lstinline{barrier}/4.
The first two arguments give the field position and the last two
the orientation of the barrier, which is mostly east (1,0) or south (0,1).%
\footnote{Symmetric barriers are handled by predicate \lstinline{stop}/4 in Line~4 and~5 of Listing~\ref{lst:lp:ricochet}.}
For instance, the atom \lstinline{barrier(2,1,1,0)} in Line~3 represents the vertical wall
between the fields (2,1) and (3,1), and \lstinline{barrier(5,1,0,1)}
stands for the horizontal wall separating (5,1) from (5,2).
% ------------------------------------------------------------
\begin{lstlisting}[basicstyle=\ttfamily\small,caption=The Board (\texttt{board.lp}),label=lst:lp:board,language=clingo]
dim(1..16).

barrier( 2, 1, 1,0).  barrier(13,11, 1,0).  barrier( 9, 7,0, 1).
barrier(10, 1, 1,0).  barrier(11,12, 1,0).  barrier(11, 7,0, 1).
barrier( 4, 2, 1,0).  barrier(14,13, 1,0).  barrier(14, 7,0, 1).
barrier(14, 2, 1,0).  barrier( 6,14, 1,0).  barrier(16, 9,0, 1).
barrier( 2, 3, 1,0).  barrier( 3,15, 1,0).  barrier( 2,10,0, 1).
barrier(11, 3, 1,0).  barrier(10,15, 1,0).  barrier( 5,10,0, 1).
barrier( 7, 4, 1,0).  barrier( 4,16, 1,0).  barrier( 8,10,0,-1).
barrier( 3, 7, 1,0).  barrier(12,16, 1,0).  barrier( 9,10,0,-1).
barrier(14, 7, 1,0).  barrier( 5, 1,0, 1).  barrier( 9,10,0, 1).
barrier( 7, 8, 1,0).  barrier(15, 1,0, 1).  barrier(14,10,0, 1).
barrier(10, 8,-1,0).  barrier( 2, 2,0, 1).  barrier( 1,12,0, 1).
barrier(11, 8, 1,0).  barrier(12, 3,0, 1).  barrier(11,12,0, 1).
barrier( 7, 9, 1,0).  barrier( 7, 4,0, 1).  barrier( 7,13,0, 1).
barrier(10, 9,-1,0).  barrier(16, 4,0, 1).  barrier(15,13,0, 1).
barrier( 4,10, 1,0).  barrier( 1, 6,0, 1).  barrier(10,14,0, 1).
barrier( 2,11, 1,0).  barrier( 4, 7,0, 1).  barrier( 3,15,0, 1).
barrier( 8,11, 1,0).  barrier( 8, 7,0, 1).  
\end{lstlisting}
% ------------------------------------------------------------

Listing~\ref{lst:lp:targets} gives the sixteen possible target locations printed on the game's carton board
(cf.~Line~3 to~18).
Each robot has four possible target locations, expressed by the ternary predicate \lstinline{target}.
Such a target is put in place via the unary predicate~\lstinline{goal} that associates a number with each location.
The external declaration in Line~1 paves the way for fixing the target location from outside the solving process.
For instance, setting \lstinline{goal(13)} to true makes position \lstinline{(15,13)} a target location for the \lstinline{yellow} robot.
% ------------------------------------------------------------
\begin{lstlisting}[basicstyle=\ttfamily\small,caption=Robots and targets (\texttt{targets.lp}),label=lst:lp:targets,language=clingo]
#external goal(1..16).

target(red,    5, 2) :- goal(1).       % red moon
target(red,   15, 2) :- goal(2).       % red triangle
target(green,  2, 3) :- goal(3).       % green triangle
target(blue,  12, 3) :- goal(4).       % blue star
target(yellow, 7, 4) :- goal(5).       % yellow star
target(blue,   4, 7) :- goal(6).       % blue saturn
target(green, 14, 7) :- goal(7).       % green moon
target(yellow,11, 8) :- goal(8).       % yellow saturn
target(yellow, 5,10) :- goal(9).       % yellow moon
target(green,  2,11) :- goal(10).      % green star
target(red,   14,11) :- goal(11).      % red star
target(green, 11,12) :- goal(12).      % green saturn
target(yellow,15,13) :- goal(13).      % yellow star
target(blue,   7,14) :- goal(14).      % blue star
target(red,    3,15) :- goal(15).      % red saturn
target(blue,  10,15) :- goal(16).      % blue moon

robot(red;green;blue;yellow).
#external pos((red;green;blue;yellow),1..16,1..16).

\end{lstlisting}
% ------------------------------------------------------------
Similarly,
the initial robot positions can be set externally, as declared in Line~21.
That is, each robot can be put at 256 different locations.
On the left hand side of Figure~\ref{fig:rr:goal13},
we cornered all robots by setting
\lstinline{pos(red,1,1)}, \lstinline{pos(blue,1,16)}, \lstinline{pos(green,16,1)}, and \lstinline{pos(yellow,16,16)}
to true.

Finally, the encoding in Listing~\ref{lst:lp:ricochet} gives a non-incremental encoding with a fixed horizon,
following the one by~\citeN[Listing 2]{gejokaobsascsc13a}.
% ------------------------------------------------------------
\begin{lstlisting}[basicstyle=\ttfamily\small,caption=Simple encoding for \emph{Ricochet Robots} (\texttt{ricochet.lp}),label=lst:lp:ricochet,language=clingo]
time(1..horizon).
dir(-1,0;1,0;0,-1;0,1).

stop( DX, DY,X,   Y   ) :- barrier(X,Y,DX,DY).
stop(-DX,-DY,X+DX,Y+DY) :- stop(DX,DY,X,Y).

pos(R,X,Y,0) :- pos(R,X,Y).

1 { move(R,DX,DY,T) : robot(R), dir(DX,DY) } 1 :- time(T).
move(R,T) :- move(R,_,_,T).

halt(DX,DY,X-DX,Y-DY,T) :- pos(_,X,Y,T), dir(DX,DY),
     dim(X-DX), dim(Y-DY), not stop(-DX,-DY,X,Y), T < horizon.

goto(R,DX,DY,X,Y,T) :- pos(R,X,Y,T), dir(DX,DY), T < horizon.
goto(R,DX,DY,X+DX,Y+DY,T) :- goto(R,DX,DY,X,Y,T),
     dim(X+DX), dim(Y+DY), not stop(DX,DY,X,Y), not halt(DX,DY,X,Y,T).

pos(R,X,Y,T) :- move(R,DX,DY,T), goto(R,DX,DY,X,Y,T-1),
                not goto(R,DX,DY,X+DX,Y+DY,T-1).
pos(R,X,Y,T) :- pos(R,X,Y,T-1), time(T), not move(R,T).

:- target(R,X,Y), not pos(R,X,Y,horizon).

#show move/4.
\end{lstlisting}
% ------------------------------------------------------------
%
The first lines in Listing~\ref{lst:lp:ricochet} furnish domain definitions,
fixing
the sequence of time steps (\lstinline{time}/1)%
\footnote{The initial time point \lstinline{0} is handled explicitly.}
and two-dimensional representations of the four possible directions (\lstinline{dir}/2).
The constant \lstinline{horizon} is expected to be provided via \clingo\ option \lstinline{-c}
(e.g.\ `\lstinline{-c horizon=20}').
Predicate \lstinline{stop}/4 is the symmetric version of \lstinline{barrier}/4 from above
and identifies all blocked field transitions.
The initial robot positions are fixed in Line~7 (in view of external input).

At each time step, some robot is moved in a direction (cf.~Line~9).
Such a \lstinline{move} can be regarded as the composition of successive field transitions,
captured by predicate \lstinline{goto}/6 (in Line~15--17).
To this end, predicate \lstinline{halt}/5 provides
temporary barriers due to robots' positions before the \lstinline{move}.
To be more precise, a robot moving in
direction \lstinline{(DX,DY)}
% east (1,0), west (-1,0),
% south (0,1), or north (0,-1) direction
must halt at field \lstinline{(X-DX,Y-DY)} when
some (other) robot is located at \lstinline{(X,Y)}, and
an instance of \lstinline{halt(DX,DY,X-DX,Y-DY,T)}
may provide information relevant to the \lstinline{move} at step \lstinline{T+1}
if there is no barrier between \lstinline{(X-DX,Y-DY)} and \lstinline{(X,Y)}.
% the dynamic and robot-specific extension of the
% domain predicate \lstinline{stop/4};
% it is obtained by adding robots as barriers.
%
Given this,
the definition of \lstinline{goto}/6 starts at a robot's position (in Line~15)
and continues in
direction \lstinline{(DX,DY)} (in Line~16--17) unless a barrier, a robot,
or the board's border is encountered.
As this definition tolerates board traversals of length zero,
\lstinline{goto}/6 is guaranteed to yield a successor position
for any \lstinline{move} of a robot~\lstinline{R} in direction
\lstinline{(DX,DY)}, so that
the rule in Line~19--20 captures the effect of \lstinline{move(R,DX,DY,T)}.
Moreover,
the frame axiom in Line~21 preserves the positions of unmoved robots,
relying on the projection \lstinline{move}/2 (cf.~Line~10).

Finally, we stipulate in Line~23 that a robot \lstinline{R} must be at its target position
\lstinline{(X,Y)} at the last time point \lstinline{horizon}.
Adding directive `\lstinline{#show move/4.}' further allows for
projecting stable models onto the extension of the \lstinline{move}/4 predicate.

The encoding in Listing~\ref{lst:lp:ricochet} allows us to decide whether a plan of
length \lstinline{horizon} exists.
For computing a shortest plan, we may augment our decision encoding with an optimization directive.
This can be accomplished by adding the part in Listing~\ref{lst:lp:optimization}.
% ------------------------------------------------------------
\begin{lstlisting}[basicstyle=\ttfamily\small,caption=Encoding part for optimization (\texttt{optimization.lp}),label=lst:lp:optimization,firstnumber=27,language=clingo]
goon(T) :- target(R,X,Y), T = 0..horizon, not pos(R,X,Y,T).

:- move(R,DX,DY,T-1), time(T), not goon(T-1), not move(R,DX,DY,T).

#minimize{ 1,T : goon(T) }.

\end{lstlisting}
% ------------------------------------------------------------
The rule in Line~27 indicates whether some goal condition is (not) established at a time point.
Once the goal is established, the additional integrity constraint in Line~29
ensures that it remains satisfied by enforcing that
the goal-achieving move is repeated at later steps
(without altering robots' positions).
Note that the \lstinline{#minimize} directive
in Line~31 aims at few instances of \lstinline{goon}/1,
corresponding to an early establishment of the goal,
while further repetitions of the goal-achieving move are ignored.
Our extended encoding allows for computing a shortest plan of length bounded by
\lstinline{horizon}.
If there is no such plan, the problem can be posed again with an enlarged \lstinline{horizon}.
For computing a shortest plan in an unbounded fashion,
we can take advantage of incremental ASP solving, as illustrated in Section~\ref{sec:incremental}.

Apart from the two external directives that allow us to vary initial robot and target positions,
the four programs constitute an ordinary ASP formalization of a \emph{Ricochet Robots} instance.
To illustrate this,
let us override the external directives by adding facts accounting for the robot and target positions
on the left hand side of Figure~\ref{fig:rr:goal13}.
The corresponding call of \clingo\ is shown in Listing~\ref{lst:sh:one}.%
\footnote{Note that rather than using input redirection, we also could have passed the five facts via a file.}
% ------------------------------------------------------------
\begin{lstlisting}[basicstyle=\ttfamily\small,caption=One-shot solving with \clingo,label=lst:sh:one]
$ clingo board.lp targets.lp ricochet.lp optimization.lp \
          -c horizon=10                                  \
          <(echo "pos(red,1,1).   pos(green,16,1).       \
                  pos(blue,1,16). pos(yellow,16,16).     \
                  goal(13).")
\end{lstlisting}
% ------------------------------------------------------------
\begin{lstlisting}[basicstyle=\ttfamily\small,caption=Stable model projected onto the extension of the \lstinline{move}/4 predicate,label=lst:sm:one]
move(blue,0,-1,1)     move(blue,1,0,2)     move(blue,0,1,3)    \
move(blue,1,0,4)      move(yellow,0,-1,5)  move(blue,0,-1,6)   \
move(blue,1,0,7)      move(yellow,0,1,8)   move(yellow,-1,0,9) \
move(yellow,-1,0,10)
\end{lstlisting}
% ------------------------------------------------------------
The resulting one-shot solving process yields a(n optimal) stable model containing the extension of the \lstinline{move}/4 predicate given in Listing~\ref{lst:sm:one}.
The \lstinline{move} atoms in Line~1--4 of Listing~\ref{lst:sm:one} correspond to the plan indicated by the colored arrows at the bottom of the left hand side of Figure~\ref{fig:rr:goal13}.
That is, the blue robot starts by going north, east, south, and east,
then the yellow one goes north,
the blue one resumes and goes north and east,
before finally the yellow robot goes south (bouncing off the blue one)
and lands on the target by going west.
This leads to the situation depicted on the right hand side of Figure~\ref{fig:rr:goal13}.
Note that the tenth move (in Line~4) is redundant since it merely replicates the previous one
because the goal was already reached after nine steps.

\subsubsection{Playing in rounds}

\emph{Ricochet Robots} is played in rounds.
Hence, the next goal must be reached with robots placed at the positions resulting from the previous round.
For example, when pursuing \lstinline{goal(4)} in the next round,
the robots must start from the end positions given on the right hand side of Figure~\ref{fig:rr:goal13}.
The resulting configuration is shown on the left hand side of Figure~\ref{fig:rr:goal4}.
% ------------------------------------------------------------
\begin{figure}[ht]
  \centering
  \includegraphics[width=0.45\textwidth]{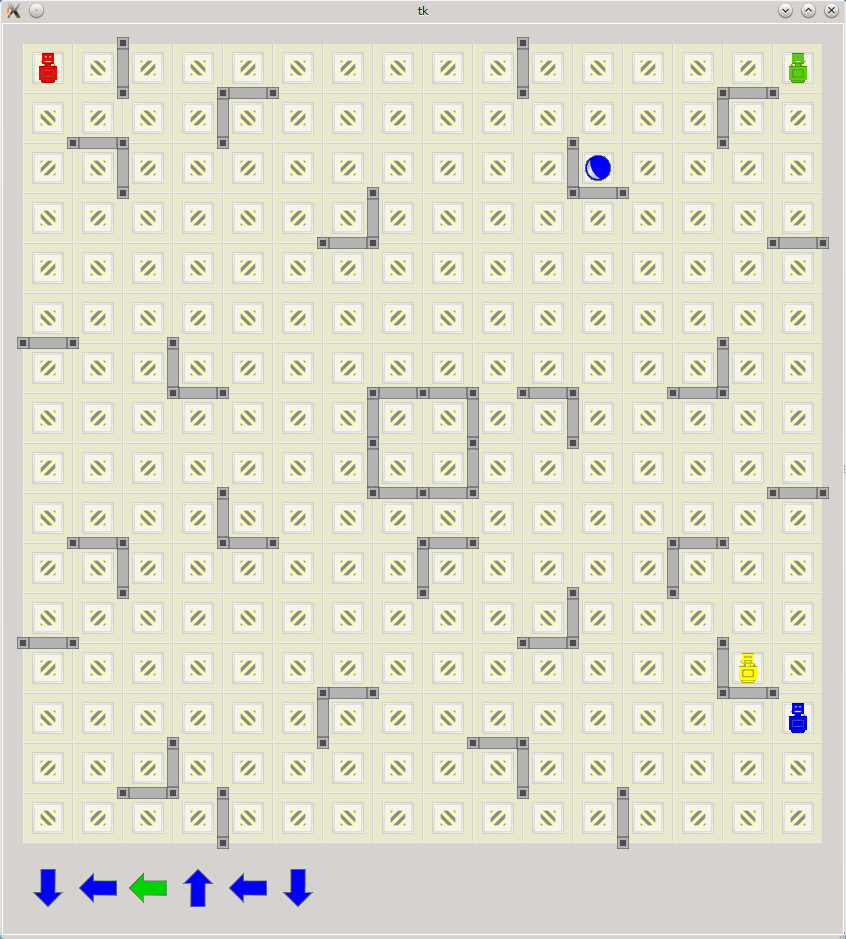}
  \hspace{0.08\textwidth}
  \includegraphics[width=0.45\textwidth]{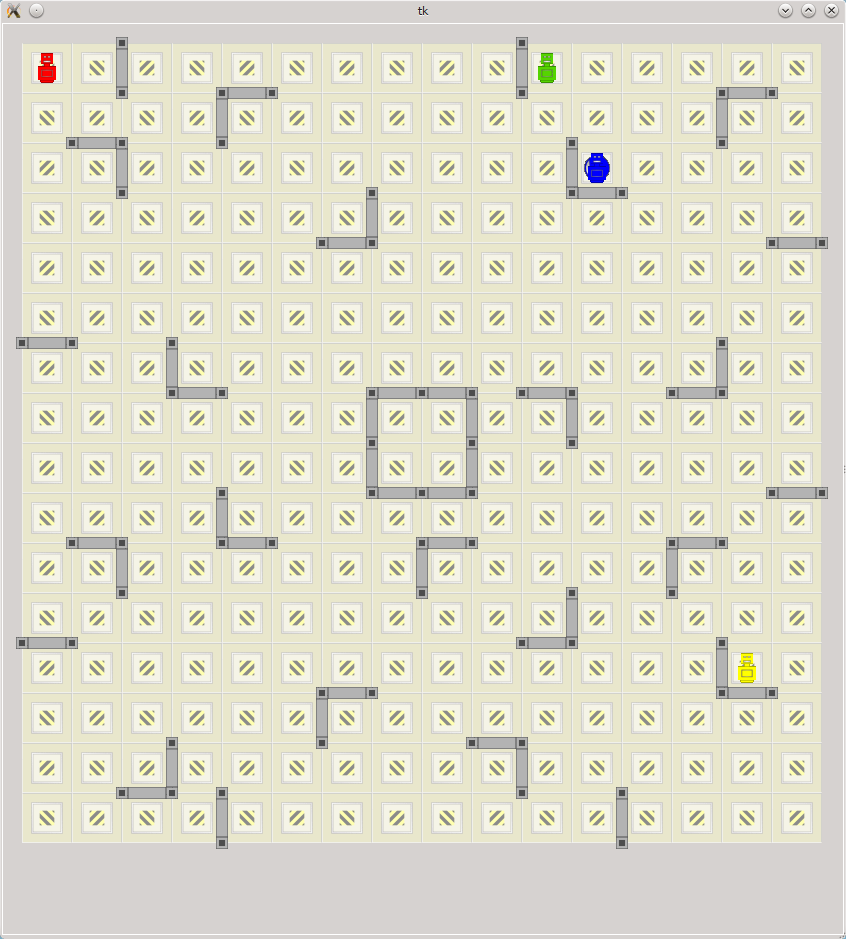}
  \caption{Visualization of solving \lstinline{goal(4)} from robot positions after having solved \lstinline{goal(13)}}
  \label{fig:rr:goal4}
\end{figure}
% ------------------------------------------------------------
For one-shot solving,
we would re-launch \clingo\ from scratch as shown in Listing~\ref{lst:sh:one},
yet by accounting for the new target and robot positions by
replacing Line~3--5 of Listing~\ref{lst:sh:one} by the following ones.
% ------------------------------------------------------------
\begin{lstlisting}[basicstyle=\ttfamily\small,firstnumber=3]
     <(echo "pos(red,1,1).   pos(green,16,1).    \
             pos(blue,16,10). pos(yellow,15,13). \
             goal(4)."                           )
\end{lstlisting}
% ------------------------------------------------------------

Unlike this,
our multi-shot approach to playing in rounds relies upon a single%
\footnote{In general, multiple such control objects can be created and made to interact via Python.}
operational \clingo\ control object that we use in a simple loop:
\begin{enumerate}
\item Create an operational control object (containing a grounder and a solver object)
\item Load and ground the programs in Listing~\ref{lst:lp:board},~\ref{lst:lp:targets},~\ref{lst:lp:ricochet},
  and optionally~\ref{lst:lp:optimization}
  \\
  (relative to some fixed \lstinline{horizon})
  within the control object
\item While there is a goal, do the following
  \begin{enumerate}
  \item Enforce the initial robot positions
  \item Enforce the current goal
  \item Solve the logic program contained in the control object
  \end{enumerate}
\end{enumerate}
The  control loop is implemented in Python by means of \clingo's Python API.
This module provides grounding and solving functionalities.%
\footnote{An analogous module is available for Lua.}
As mentioned in Section~\ref{sec:background},
both modules support (almost) literal counterparts to `Create', `Load', `Ground', and `Solve'.
The ``enforcement'' of robot and target positions is more complex,
as it involves changing the truth values of externally controlled atoms
(mimicking the insertion and deletion of atoms, respectively).

The resulting Python program is given in Listing~\ref{lst:py:ricochet}.
% ----------------------------------------------------------------------
\begin{lstlisting}[float=tp,basicstyle=\ttfamily\small,caption=The Ricochet Robot Player (\texttt{ricochet.py}),label=lst:py:ricochet,language=clingo]
from clingo import Control, Model, Function

class Player:
    def __init__(self, horizon, positions, files):
        self.last_positions = positions
        self.last_solution = None
        self.undo_external = []
        self.horizon = horizon
        self.ctl = Control(
            ['-c', 'horizon={0}'.format(self.horizon)])
        for x in files:
            self.ctl.load(x)
        self.ctl.ground([("base", [])])

    def solve(self, goal):
        for x in self.undo_external:
            self.ctl.assign_external(x, False)
        self.undo_external = []
        for x in self.last_positions + [goal]:
            self.ctl.assign_external(x, True)
            self.undo_external.append(x)
        self.last_solution = None
        self.ctl.solve(on_model=self.on_model)
        return self.last_solution

    def on_model(self, model):
        self.last_solution = model.symbols(atoms=True)
        self.last_positions = []
        for atom in model.symbols(atoms=True):
            if (atom.name == "pos" and
                    len(atom.arguments) == 4 and
                    atom.arguments[3].number == self.horizon):
                self.last_positions.append(
                    Function("pos", atom.arguments[:-1]))

horizon   = 15
encodings = ["board.lp", "targets.lp", "ricochet.lp", "optimization.lp"]
positions = [Function("pos", [Function("red"),     1,  1]),
             Function("pos", [Function("blue"),    1, 16]),
             Function("pos", [Function("green"),  16,  1]),
             Function("pos", [Function("yellow"), 16, 16])]
sequence  = [Function("goal", [13]),
             Function("goal", [4]),
             Function("goal", [7])]

player = Player(horizon, positions, encodings)
for goal in sequence:
    print player.solve(goal)
\end{lstlisting}
% ----------------------------------------------------------------------
Line~1 imports the \lstinline{clingo} module.
We are only using three classes from the module,
which we directly pull into the global namespace
to avoid qualification with ``\lstinline{clingo.}'' and so to keep the code compact.
%\footnotemark[\ref{fn:gringo:clingo}]

Line~3--34 show the \lstinline{Player} class.
This class encapsulates all state information including \clingo's \lstinline{Control} object
that in turn holds the state of the underlying grounder and solver.
In the \lstinline{Player}'s \lstinline{__init__} function (similar to a constructor in other object-oriented languages) the following member variables are initialized:
\begin{description}
  \item[{\ttfamily last\_positions}]
    This variable is initialized upon construction with the starting positions of the robots.
    During the progression of the game, this variable holds the initial starting positions of the robots for each turn.
  \item[{\ttfamily last\_solution}]
    This variable holds the last solution of a search call.
    %We use this member variable
    %because the gringo module offers (among others) an interface using callbacks,
    %which does not provide direct means to return values from within a callback.
    %A callback can be any python function though.
    %This allows us to pass a closure having a reference to the player object,
    %which in turn allows us to access the [{\ttfamily last\_solution}] member variable
    %- and hence to keep our data well organized.
  \item[{\ttfamily undo\_external}]
    We want to successively solve a sequence of goals.
    In each step, a goal has to be reached from different starting positions.
    This variable holds a list containing the current goal and starting positions
    that have to be cleared upon the next step.
  \item[{\ttfamily horizon}]
    We are using a bounded encoding.
    This (Python) variable holds the maximum number of moves to find a solution for a given step.
  \item[{\ttfamily ctl}]
    This variable holds the actual object providing an interface to the grounder and solver.
    It holds all state information necessary for multi-shot solving
    along with heuristic information gathered during solving.
\end{description}
As shown in Line~4--13, the constructor takes the \lstinline{horizon}, initial robot \lstinline{positions}, and the \lstinline{files} containing the various logic programs.
\clingo's \lstinline{Control} object is created in Line~9--10 by passing the option \lstinline{-c} to replace the logic program constant \lstinline{horizon}
by the value of the Python variable \lstinline{horizon} during grounding.
Finally, the constructor loads all \lstinline{files} and grounds the entire logic program in Line~11--13.
Recall from Section~\ref{sec:background}  that all rules outside the scope of \lstinline{#program} directives belong to the \lstinline{base} program.
Note also that this is the only time grounding happens because the encoding is bounded.
All following solving steps are configured exclusively via manipulating external atoms.

The \lstinline{solve} method in Line~15--24 starts with
initializing the search for the solution to the new \lstinline{goal}.
To this end,
it first undos in Line~16--17 the previous goal and starting positions stored in \lstinline{undo_external}
by assigning \lstinline{False} to the respective atoms.
% and it records this in Line~18.
In the following lines~19 to~21, the next step is initialized by
assigning \lstinline{True} to the current \lstinline{goal} along with the last robot positions;
these are also stored in \lstinline{undo_external} so that they can be taken back afterwards.
Finally, the  \lstinline{solve} method calls \clingo's \lstinline{ctl.solve} to initiate the search.
The result is captured in variable \lstinline{last_solution}.
% which is thus cleared in Line~20.
Note that the call to \lstinline{ctl.solve} takes \lstinline{ctl.on_model} as (keyword) argument,
which is called whenever a model is found.
In other words, \lstinline{on_model} acts as a callback for intercepting models.
Finally, variable \lstinline{last_solution} is returned at the end of the method.

The last function of the \lstinline{Player} class is the \lstinline{on_model} callback.
As mentioned, it intercepts the (final) \lstinline{model}s computed by the solver,
which can then be inspected via the functions of the \lstinline{Model} class.
At first, it stores the shown atoms in variable \lstinline{last_solution} in Line~27.%
\footnote{In view of `\lstinline{#show move/4.}' in Listing~\ref{lst:lp:ricochet}, this only involves instances of \lstinline{move/4}, while all true atoms are included via the argument \lstinline{Model.ATOMS} in Line~29.}
The remainder of the \lstinline{on_model} callback extracts the final robot positions from the stable model.
For that, it loops in Line~29--34 over the full set of atoms in the \lstinline{model} and checks whether their signatures match.
That is,
if an atom is formed from predicate \lstinline{pos/4} and its fourth argument equals the \lstinline{horizon},
then it is appended to the list of \lstinline{last_positions} after stripping its time step from its arguments.

As an example,
consider \lstinline{pos(yellow,15,13,20)},
say
the final position of the yellow robot on the right hand side of Figure~\ref{fig:rr:goal13} at an \lstinline{horizon} of 20.
This leads to the addition of \lstinline{pos(yellow,15,13)} to the \lstinline{last_positions}.
Note that \lstinline{pos(yellow,15,13)} is declared an external atom in Line~21 of Listing~\ref{lst:lp:targets}.
For playing the next round, we can thus make it \lstinline{True} in Line~20 of Listing~\ref{lst:py:ricochet}.
And when solving, the rule in Line~7 of Listing~\ref{lst:lp:ricochet} allows us to derive \lstinline{pos(yellow,15,13,0)}
and makes it the new starting position of the yellow robot,
as shown on the left hand side of Figure~\ref{fig:rr:goal4}.

Line~36--44 show the code for configuring   the player.
They set the search \lstinline{horizon},
the \lstinline{encodings} to solve with, and
the initial \lstinline{positions} in form of \lstinline{clingo} terms.
Furthermore, we fix a \lstinline{sequence} of goals in Line~42--44.
In a more realistic setting,
either some user interaction or
a random sequence might be generated to emulate arbitrary draws.

% ------------------------------------------------------------
\begin{lstlisting}[basicstyle=\ttfamily\small,caption=Multi-shot solving with \clingo's Python API,label=lst:sh:multi,escapechar=?]
$ python ricochet.py
[move(red,0,1,1), move(red,1,0,2), move(red,0,1,3), ...]
[move(blue,0,-1,1), move(blue,1,0,2), move(blue,0,1,3), ...]
[move(green,0,1,1), move(green,1,0,2), move(green,1,0,3), ...]
\end{lstlisting}
% ------------------------------------------------------------
Finally, Line~46--48 implement the search for sequences of moves that solve the configuration given above.
For each \lstinline{goal} in the \lstinline{sequence}, a solution is plainly printed, as engaged in Line~48.
The three lists in Listing~\ref{lst:sh:multi} represent solutions to the three goals in Line~42--44.
The \clingo\ library does not foresee any output, which must thus be handled by the scripting language.
Note also that the first list represents an alternative solution to the one given in Listing~\ref{lst:sm:one}.

%%% Local Variables:
%%% mode: latex
%%% TeX-master: "paper"
%%% End:

\subsection{Optimization}
\label{sec:optimization}

Another innovative feature of \clingo{} is its incremental optimization.
This allows for adapting objective functions along the evolution of a program at hand.
A simple example is the search for shortest plans when increasing the horizon in non-consecutive steps.
To see this,
recall that literals in minimize statements (and analogously weak constraints) are supplied with a sequence of terms of the form $w@p,\vec{t}$,
where $w$ and $p$ are integers providing a weight and a priority level and $\vec{t}$ is a sequence of terms (cf.~\cite{aspcore2}).
As an example, consider the subprogram:%
\footnote{The same applies to a weak constraint of form `\texttt{:\texttildelow~}\lstinline{move(X,Y,W,P,t). [W@P,X,Y,t]}'.}
\begin{lstlisting}[numbers=none,language=clingo]
#program cumulativeObjective(t).
#minimize{ W@P,X,Y,t : move(X,Y,W,P,t) }.
\end{lstlisting}
When grounding and solving \lstinline{cumulativeObjective(t)} for successive values of \lstinline{t},
the solver's objective function (per priority level \lstinline{P}) is 
gradually extended with new atoms over \lstinline{move/5}, 
and all previous ones are kept.

For enabling the removal of literals from objective functions,
we can use externals:
\begin{lstlisting}[numbers=none,language=clingo]
#program volatileObjective(t).
#external activateObjective(t).
#minimize{ W@P,X,Y,t : move(X,Y,W,P,t), activateObjective(t) }.
\end{lstlisting}  %or :~ move(X,Y,W,P,t), activateObjective(t).[W@P,X,Y,t]
The subprogram \lstinline{volatileObjective(t)} behaves like \lstinline{cumulativeObjective(t)} as long as the external atom \lstinline{activateObjective(t)} is true.
Once it is set to false, all atoms over \lstinline{move/5} with the
corresponding term for~\lstinline{t} are dismissed from objective functions.

%%% Local Variables: 
%%% mode: latex
%%% TeX-master: "paper"
%%% End: 

\section{Application program interfaces}
\label{sec:api}

This section provides some further selected functionalities of \clingo's APIs;
detailed descriptions can be found at~\url{potassco.org}.
Currently, \clingo{} provides APIs in \C, \cpp, \lua, and \python,
all sharing the same functionality.
A tutorial on using the \python{} API for multi-shot and theory solving was given by~\citeN{kascwa17a}.

The theory reasoning capabilities of \clingo{} are described by~\citeN{gekakaosscwa16a}.
In brief,
\clingo{} provides generic means for incorporating theory reasoning.
They span from theory grammars for seamlessly extending its input language with theory expressions to 
a simple interface for integrating theory propagators into its solver component.
Multi-shot solving for selected theories is described by~\citeN{bakaossc16a} and~\citeN{jakaosscscwa17a}.

The central role in multi-shot solving is played by control objects capturing the system states of grounders and solvers,
as introduced in Section~\ref{sec:semantics:operational}.
While the control object created by invoking \clingo{} from the command line is passed as argument to the \lstinline{main} routine,
further such objects can be created with the constructor \lstinline{Control}.
Examples for both settings can be found in Line~8 of Listing~\ref{fig:iclingo:python} and Line~9 and~10 of Listing~\ref{lst:py:ricochet}, respectively.
In general,
this allows multiple independent \clingo{} objects to coexist and communicate with each other.

\clingo's interface can be structured into three parts.

\paragraph{Parsing.}
A simple but very useful feature of \clingo's API is that it allows us to
leverage the parser of its grounding component \gringo{} to obtain an abstract syntax tree (AST)
of the non-ground program.
More precisely,
the interface allows for both obtaining an AST of a full-fledged logic program and adding such a program in form of an AST.
This provides an easy way of applying program transformations on the non-ground level without any burden of parsing input programs in their full generality.
This feature is exploited by the systems \asprin~\cite{brderosc15b} for expressing preferences and \anthem\footnote{\url{https://github.com/potassco/anthem}} for formula extraction.

\paragraph{Grounding.}
A basic functionality of \clingo{} objects is to incrementally augment non-ground programs by loading programs from file or adding them in string form.
An example of the former can be found in Line~12 of Listing~\ref{lst:py:ricochet};
the latter is accomplished by the counterpart of \textit{add}, defined in Section~\ref{sec:semantics:operational}.
Notably, this functionality allows for adding dynamically generated programs.

More fine-grained control is provided by the low level part of the interface.
This allows, for instance, for inspecting the result of grounding and adding ground rules in the intermediate ASP format \aspif{} (cf.~\cite{kascwa17a}).
In this way, one can iterate over all ground atoms and inspect them individually, 
or implement eager constraint translations by adding new (theory) atoms on demand.
Also, symbols can be injected during grounding via external functions,
similar to value invention~\cite{cacoia07a}.

\paragraph{Solving.}
The principal \lstinline{solve} method can be shaped in various ways.
For instance, we have seen in Listing~\ref{lst:py:ricochet} how \lstinline{on_model} can act as a callback for intercepting models.
More precisely,
for each stable model found during a call to \lstinline{solve(on_model=f)},
a model object is passed to function \lstinline{f},
whose implementation can then access and inspect the model.
An example consists of the addition of constraints whenever a model is found,%
\footnote{A more sophisticated way is to use theory propagators adding constraints not just when a model is found but also during the solver's propagation; 
  see~\cite{gekakaosscwa16a} for details.}
as in optimization tasks, or final tests on model candidates.
Similarly,
\lstinline{solve} can by supplied with assumptions, as detailed in~\eqref{solve:models}.
For instance,
the call \lstinline{solve(assumptions=[(Function("a"), True)])} only admits stable models containing atom \lstinline{a}.
Moreover,
\clingo{} provides an asynchronous interface, which is particularly useful in reactive settings.
Here, solving is done in the background and interruptible at any time.
For example, this allows to accommodate scenarios where agents have to stay responsive even though solving has not yet finished.
Finally, it is worth mentioning that dedicated parts of the API allow for configurating search and extracting solver statistics.
In combination with the aforementioned \lstinline{on_model} callback this allows for re-configuring search in view of the statistics gathered during
the search for the last model.

%%% Local Variables:
%%% mode: latex
%%% TeX-master: "paper"
%%% End:

\section{Experiments}
\label{sec:experiments}

The computational advantage of multi-shot solving lies in its avoidance of redundancies otherwise caused by relaunching grounder and solver programs.
Since the substantial savings on grounding intense benchmarks have already been demonstrated~\cite{gekakaosscth08a},
we focus our empirical analysis on the impact of our approach on solving.
In particular, we want to investigate in how far multi-shot solving can
benefit from the learning capacities of modern ASP solvers and the reuse of already gathered heuristic scores.
To this end,
we empirically evaluate the impact of \clingo's multi-shot solving capacities
on three planning benchmarks:%
\footnote{The benchmarks are available at \url{http://www.cs.uni-potsdam.de/wv/clingo/benchmark-2017-05-26.tar.xz}}
Towers of Hanoi (cf.\ Section~\ref{sec:incremental}),
Ricochet Robots (cf.\ Section~\ref{sec:robots}), and
ASP encodings of PDDL problems~\cite{digelurosc17a}.%
\footnote{PDDL stands for Planning Domain Definition Language and should indicate that our benchmarks were obtained by translating benchmarks
  originally specified in PDDL and used by the planning community.}
For either benchmark,
we let \clingo{} version~4.5.4 search for a shortest plan
by incrementally extending the horizon until the first plan is found.
In particular,
we consider multi-shot solving by means of \clingo's built-in incremental mode
(invoked by `\lstinline{#include <incmode>.}')
in four different settings:
\begin{itemize}
\item \emph{multi}: keeping recorded nogoods as well as heuristic values between solver calls
\item \emph{multi -heuristic}: keeping recorded nogoods, but not heuristic values, between solver calls
\item \emph{multi -nogoods}: keeping heuristic values, but not recorded nogoods, between solver calls
\item \emph{multi -heuristic -nogoods}: keeping neither recorded nogoods nor heuristic values between solver calls
\end{itemize}
We contrast these four settings to the traditional single-shot approach,
denoted by \emph{single},
where \clingo{} performs grounding and solving from scratch for each planning horizon.
The experiments were run sequentially on a Linux machine equipped with % two Quad-Core
Xeon E5520 2.27GHz processors,
limiting wall-clock time to 3000 seconds per run without imposing any (effective) memory limit.
\clingo{} was run in all experiments in its default configuration
except for the option \lstinline{--forget-on-step} that allows for
configuring the four settings above.

We portray our results in terms of cactus plots, in which
the x-axes list instances ordered by time (and conflicts) and
the y-axes reflect times and conflicts, respectively.
The magnitude of the latter are given on top of the y-axis.

\subsection{Towers of Hanoi}
\label{sec:experiments:hanoi}
\begin{figure}[t]
\begin{tikzpicture}
  \begin{axis}[
    width=\textwidth,
    height=200pt,
    axis x line=middle,
    axis y line=center,
    grid=major,
    scaled y ticks=base 10:-3,
    legend pos=north west,
    legend cell align=left,
    xlabel=Instances ordered by time,
    ylabel=Time in seconds,
    x label style={at={(axis description cs:0.5,-0.11)},anchor=north},
    y label style={at={(axis description cs:-0.07,.5)},rotate=90,anchor=south},
    ]
    \definecolor{current1}{RGB}{255,0,0}
    \addplot [color=current1] table [header=true,y=multi,x=number,col sep=comma] {plots/toh-time.csv};
    \definecolor{current2}{RGB}{0,255,0}
    \addplot [color=current2] table [header=true,y=multi -heuristic,x=number,col sep=comma] {plots/toh-time.csv};
    \definecolor{current3}{RGB}{0,0,255}
    \addplot [color=current3] table [header=true,y=multi -nogoods,x=number,col sep=comma] {plots/toh-time.csv};
    \definecolor{current4}{RGB}{0,255,255}
    \addplot [color=current4] table [header=true,y=multi -heuristic -nogoods,x=number,col sep=comma] {plots/toh-time.csv};
    \definecolor{current5}{RGB}{0,0,0}
    \addplot [color=current5] table [header=true,y=single,x=number,col sep=comma] {plots/toh-time.csv};
  \legend{multi,multi -heuristic,multi -nogoods,multi -heuristic -nogoods,single}
  \end{axis}
\end{tikzpicture}
\begin{tikzpicture}
  \begin{axis}[
    width=\textwidth,
    height=200pt,
    axis x line=middle,
    axis y line=center,
    grid=major,    
legend style={yshift=-1.7cm},
    legend pos=north west,
    legend cell align=left,
    xlabel=Instances ordered % lexicographically 
           by time and conflicts,
    ylabel=Number of conflicts,
    x label style={at={(axis description cs:0.5,-0.11)},anchor=north},
    y label style={at={(axis description cs:-0.07,.5)},rotate=90,anchor=south},
    ]
    \definecolor{current1}{RGB}{255,0,0}
    \addplot [color=current1] table [header=true,y=multi,x=number,col sep=comma] {plots/toh-conflicts.csv};
    \definecolor{current2}{RGB}{0,255,0}
    \addplot [color=current2] table [header=true,y=multi -heuristic,x=number,col sep=comma] {plots/toh-conflicts.csv};
    \definecolor{current3}{RGB}{0,0,255}
    \addplot [color=current3] table [header=true,y=multi -nogoods,x=number,col sep=comma] {plots/toh-conflicts.csv};
    \definecolor{current4}{RGB}{0,255,255}
    \addplot [color=current4] table [header=true,y=multi -heuristic -nogoods,x=number,col sep=comma] {plots/toh-conflicts.csv};
    \definecolor{current5}{RGB}{0,0,0}
    \addplot [color=current5] table [header=true,y=single,x=number,col sep=comma] {plots/toh-conflicts.csv};
  \legend{multi,multi -heuristic,multi -nogoods,multi -heuristic -nogoods,single}
  \end{axis}
\end{tikzpicture}
\caption{Cactus plots for Towers of Hanoi benchmark\label{fig:experiments:hanoi}}
\end{figure}
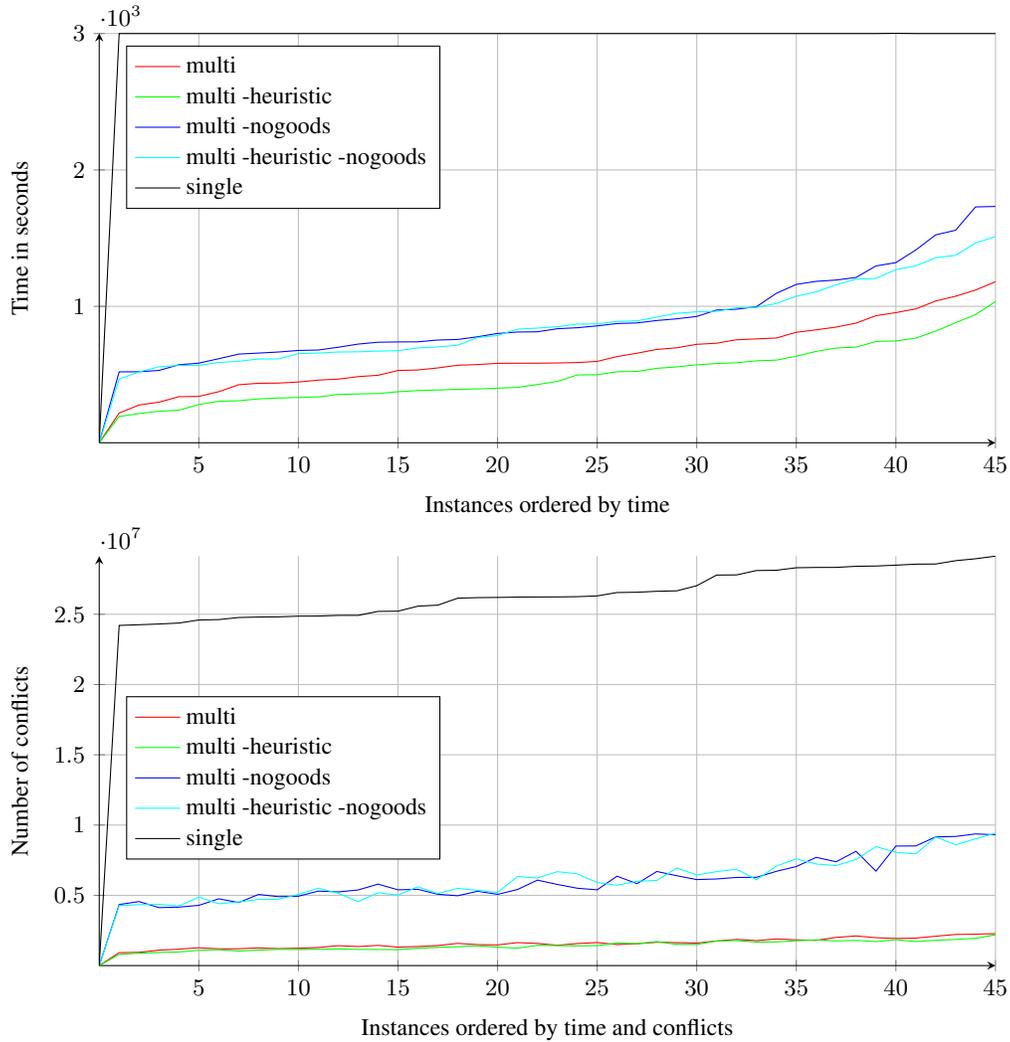

The upper plot in Figure~\ref{fig:experiments:hanoi} displays runtimes
for each \clingo{} setting in increasing order
over 45 instances of the Towers of Hanoi benchmark.
Most apparently, we observe that single-shot solving performed in the \emph{single} setting
fails to complete any of the instances within the allotted 3000 seconds.
This clearly shows that the redundancy of relaunching grounding and solving processes from scratch
for each horizon incurs non-negligible overhead here.

Somewhat unexpectedly, the multi-shot solving approaches in which recorded nogoods are discarded
between successive solver calls, viz.\ \emph{multi -nogoods} and \emph{multi -heuristic -nogoods},
perform much better than single-shot solving and complete all 45 instances within the
time limit.
On the one hand, this advantage is owed to incremental grounding, adding only new rule instances
when switching from one horizon to the next.
On the other hand, we verified that solving steps take the major share of runtime with each setting,
and the numbers of conflicts plotted in the lower part of Figure~\ref{fig:experiments:hanoi}
exhibit substantial search reductions in multi-shot solving, even when neither recorded nogoods
nor heuristic values are kept.
In fact,
at the implementation level \clingo{} asserts unary nogoods and
stores binary as well as ternary nogoods persistently in dedicated data structures.
Hence, such short nogoods remain available in  multi-shot solving regardless of settings and
explain the gap to iterated single-shot solving, where respective information
has to be repeatedly retrieved from search conflicts at each solving step.

Keeping also long nogoods between solver calls,
as done in the \emph{multi} and \emph{multi -heuristic} settings,
further reduces search conflicts and runtime
by a factor of about~5 or~1.5, respectively.
This indicates a trade-off between memory demands and search savings due to
recorded nogoods, where the savings outweigh the overhead on the Towers of Hanoi benchmark.
Unlike that, keeping heuristic values does not pay off here, and the two
\emph{-heuristic} settings save some fraction of runtime in comparison to their counterparts
passing such values between solver calls.
Although respective gaps are modest, we conclude that biasing search to proceed as in
previous solving steps does not provide shortcuts for problems with an extended planning horizon.

\subsection{Ricochet Robots}
\label{sec:experiments:robots}
\begin{figure}[t]
\begin{tikzpicture}
  \begin{axis}[
    width=\textwidth,
    height=200pt,
    axis x line=middle,
    axis y line=center,
    grid=major,
    scaled y ticks=base 10:-3,
    legend pos=north west,
    legend cell align=left,
    xlabel=Instances ordered by time,
    ylabel=Time in seconds,
    x label style={at={(axis description cs:0.5,-0.11)},anchor=north},
    y label style={at={(axis description cs:-0.07,.5)},rotate=90,anchor=south},
    ]
    \definecolor{current1}{RGB}{255,0,0}
    \addplot [color=current1] table [header=true,y=multi,x=number,col sep=comma] {plots/ricochet-time.csv};
    \definecolor{current2}{RGB}{0,255,0}
    \addplot [color=current2] table [header=true,y=multi -heuristic,x=number,col sep=comma] {plots/ricochet-time.csv};
    \definecolor{current3}{RGB}{0,0,255}
    \addplot [color=current3] table [header=true,y=multi -nogoods,x=number,col sep=comma] {plots/ricochet-time.csv};
    \definecolor{current4}{RGB}{0,255,255}
    \addplot [color=current4] table [header=true,y=multi -heuristic -nogoods,x=number,col sep=comma] {plots/ricochet-time.csv};
    \definecolor{current5}{RGB}{0,0,0}
    \addplot [color=current5] table [header=true,y=single,x=number,col sep=comma] {plots/ricochet-time.csv};
  \legend{multi,multi -heuristic,multi -nogoods,multi -heuristic -nogoods,single}
  \end{axis}
\end{tikzpicture}
\begin{tikzpicture}
  \begin{axis}[
    width=\textwidth,
    height=200pt,
    axis x line=middle,
    axis y line=center,
    grid=major,
    legend pos=north west,
    legend cell align=left,
    xlabel=Instances ordered % lexicographically 
           by time and conflicts,
    ylabel=Number of conflicts,
    x label style={at={(axis description cs:0.5,-0.11)},anchor=north},
    y label style={at={(axis description cs:-0.07,.5)},rotate=90,anchor=south},
    ]
    \definecolor{current1}{RGB}{255,0,0}
    \addplot [color=current1] table [header=true,y=multi,x=number,col sep=comma] {plots/ricochet-conflicts.csv};
    \definecolor{current2}{RGB}{0,255,0}
    \addplot [color=current2] table [header=true,y=multi -heuristic,x=number,col sep=comma] {plots/ricochet-conflicts.csv};
    \definecolor{current3}{RGB}{0,0,255}
    \addplot [color=current3] table [header=true,y=multi -nogoods,x=number,col sep=comma] {plots/ricochet-conflicts.csv};
    \definecolor{current4}{RGB}{0,255,255}
    \addplot [color=current4] table [header=true,y=multi -heuristic -nogoods,x=number,col sep=comma] {plots/ricochet-conflicts.csv};
    \definecolor{current5}{RGB}{0,0,0}
    \addplot [color=current5] table [header=true,y=single,x=number,col sep=comma] {plots/ricochet-conflicts.csv};
  \legend{multi,multi -heuristic,multi -nogoods,multi -heuristic -nogoods,single}
  \end{axis}
\end{tikzpicture}
\caption{Cactus plots for Ricochet Robots benchmark\label{fig:experiments:robots}}
\end{figure}
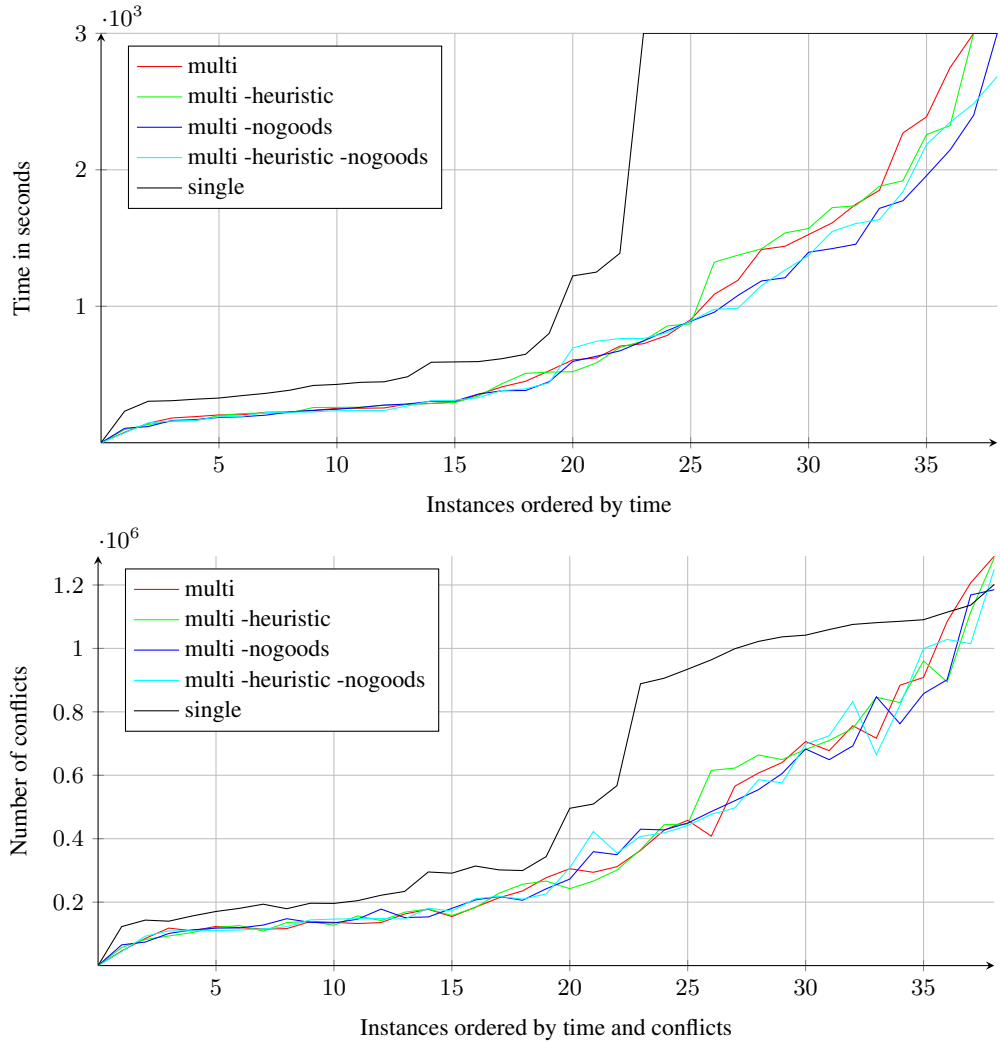

The behavior of single- and multi-shot solving approaches on 38 instances of the Ricochet Robots benchmark,
plotted in Figure~\ref{fig:experiments:robots}, parallels the previous observations.
While the \emph{single} setting is able to complete 22 of the instances within the given time,
the most successful multi-shot solving approach, \emph{multi -heuristic -nogoods}, solves all of them.
In contrast to Towers of Hanoi above,
keeping long nogoods does not pay off here,
as the lower plot in Figure~\ref{fig:experiments:robots} shows that they do not
significantly reduce the search conflicts at solving steps with an extended horizon.
Hence, the two \emph{-nogoods} settings are ahead in terms of runtime as well as solved instances
displayed in the upper part of Figure~\ref{fig:experiments:robots}.
Moreover, we see that keeping or discarding heuristic values, the latter denoted by \emph{-heuristic},
does not make much difference,
which again exposes that biasing the search process in view of previous solving steps does not necessarily
help for problem extensions.

\subsection{PDDL Problems}
\label{sec:experiments:pddl}
\begin{figure}[t]
\begin{tikzpicture}
  \begin{axis}[
    width=\textwidth,
    height=200pt,
    axis x line=middle,
    axis y line=center,
    grid=major,
    scaled y ticks=base 10:-3,
    legend pos=north west,
    legend cell align=left,
    xlabel=Instances ordered by time,
    ylabel=Time in seconds,
    x label style={at={(axis description cs:0.5,-0.11)},anchor=north},
    y label style={at={(axis description cs:-0.07,.5)},rotate=90,anchor=south},
    ]
    \definecolor{current1}{RGB}{255,0,0}
    \addplot [color=current1] table [header=true,y=multi,x=number,col sep=comma] {plots/pddl-time.csv};
    \definecolor{current2}{RGB}{0,255,0}
    \addplot [color=current2] table [header=true,y=multi -heuristic,x=number,col sep=comma] {plots/pddl-time.csv};
    \definecolor{current3}{RGB}{0,0,255}
    \addplot [color=current3] table [header=true,y=multi -nogoods,x=number,col sep=comma] {plots/pddl-time.csv};
    \definecolor{current4}{RGB}{0,255,255}
    \addplot [color=current4] table [header=true,y=multi -heuristic -nogoods,x=number,col sep=comma] {plots/pddl-time.csv};
    \definecolor{current5}{RGB}{0,0,0}
    \addplot [color=current5] table [header=true,y=single,x=number,col sep=comma] {plots/pddl-time.csv};
  \legend{multi,multi -heuristic,multi -nogoods,multi -heuristic -nogoods,single}
  \end{axis}
\end{tikzpicture}
\begin{tikzpicture}
  \begin{axis}[
    width=\textwidth,
    height=200pt,
    axis x line=middle,
    axis y line=center,
    grid=major,
    legend pos=north west,
    legend cell align=left,
    xlabel=Instances ordered % lexicographically 
           by time and conflicts,
    ylabel=Number of conflicts,
    x label style={at={(axis description cs:0.5,-0.11)},anchor=north},
    y label style={at={(axis description cs:-0.07,.5)},rotate=90,anchor=south},
    ]
    \definecolor{current1}{RGB}{255,0,0}
    \addplot [color=current1] table [header=true,y=multi,x=number,col sep=comma] {plots/pddl-conflicts.csv};
    \definecolor{current2}{RGB}{0,255,0}
    \addplot [color=current2] table [header=true,y=multi -heuristic,x=number,col sep=comma] {plots/pddl-conflicts.csv};
    \definecolor{current3}{RGB}{0,0,255}
    \addplot [color=current3] table [header=true,y=multi -nogoods,x=number,col sep=comma] {plots/pddl-conflicts.csv};
    \definecolor{current4}{RGB}{0,255,255}
    \addplot [color=current4] table [header=true,y=multi -heuristic -nogoods,x=number,col sep=comma] {plots/pddl-conflicts.csv};
    \definecolor{current5}{RGB}{0,0,0}
    \addplot [color=current5] table [header=true,y=single,x=number,col sep=comma] {plots/pddl-conflicts.csv};
  \legend{multi,multi -heuristic,multi -nogoods,multi -heuristic -nogoods,single}
  \end{axis}
\end{tikzpicture}
\caption{Cactus plots for PDDL benchmark\label{fig:experiments:pddl}}
\end{figure}
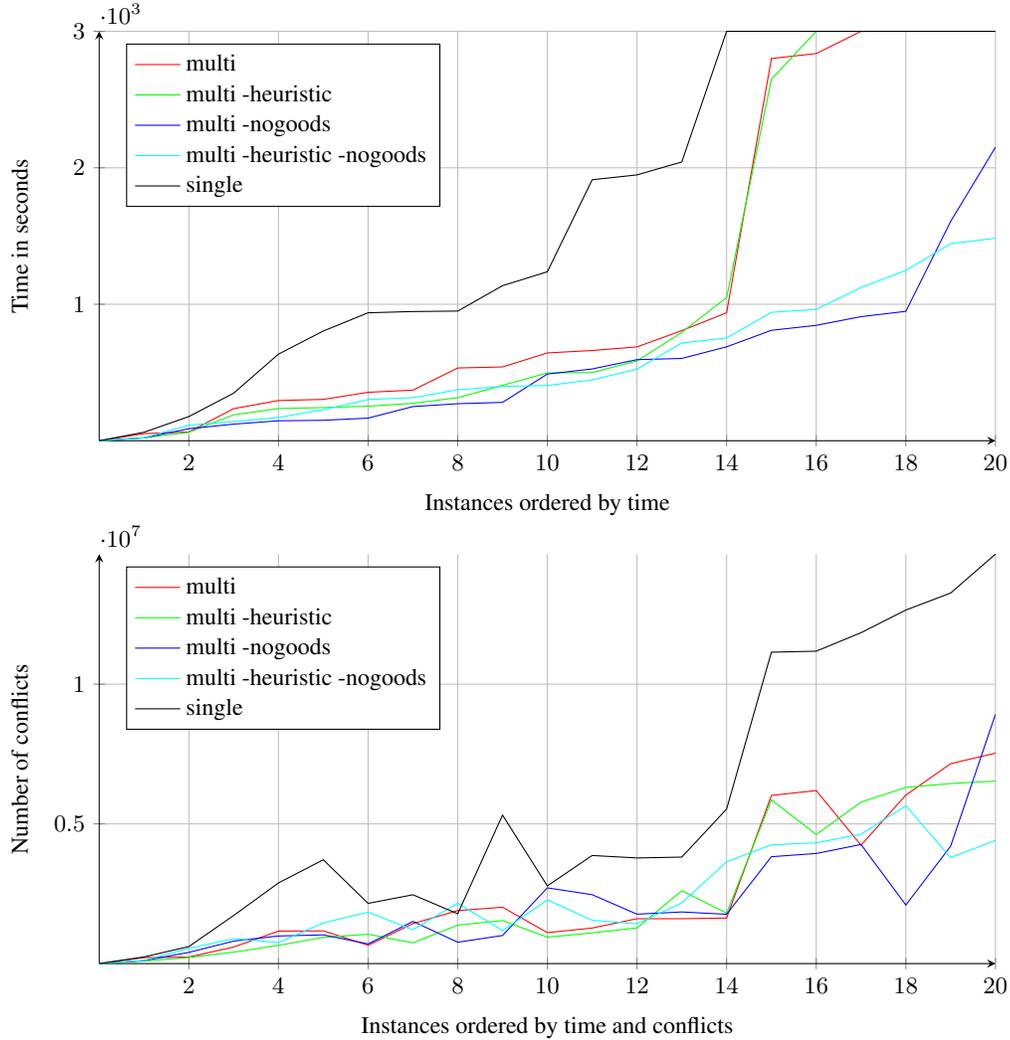

The performance results displayed in Figure~\ref{fig:experiments:pddl},
summarizing 20 instances obtained by translating planning problems from PDDL,
also exhibit a significant gap separating single-shot from multi-shot solving
in its four settings.
The behavior of the latter differs primarily w.r.t.\ the treatment of long nogoods,
where only the two \emph{-nogoods} settings that discard them between solver calls complete
all instances in time.
In fact,
comparing runtimes in the upper and numbers of conflicts in the lower part of Figure~\ref{fig:experiments:pddl},
it turns out that keeping such long nogoods incurs overhead without bringing about (consistent) search savings
in return.
The reduced numbers of conflicts relative to single-shot solving nevertheless indicate
a substantial added value of keeping some nogoods, in particular, short ones,
between solver calls.

\subsection{Observations}

On all benchmarks the number of conflicts is lower in the \emph{multi} settings compared to the \emph{single} one.
This is due to the reuse of nogoods learnt from previous solving steps.
In fact, the treatment of nogoods had the greatest influence on the runtime in the \emph{multi} settings.
As regards long nogoods, the sometimes unproportional overhead of keeping them,
observed on the Ricochet Robots benchmark and PDDL problems,
suggests to strive in the future for adaptive filtering mechanisms (beyond regular nogood deletion)
assessing the relevance of recorded nogoods from previous solving steps.
Unlike that, \clingo{} settings that differ in the treatment of heuristic values only,
i.e., pairs denoted with or without \emph{-heuristic},
exhibit comparable behavior on all three investigated benchmark problems,
thus indicating that recorded nogoods play a more important role upon successive solver calls.

Our benchmarks were not very memory demanding.
The peak memory consumption over all benchmarks was 187 MB when multi-shot solving
and only 150 MB in the single-shot setting.
Even though it might seem that the single-shot approach has a smaller memory footprint,
the multi-shot settings probably just held larger databases of learnt clauses
leading to a slightly higher memory consumption.
With the selected benchmarks, the memory used by the grounding component is negligible.

%%% Local Variables:
%%% mode: latex
%%% TeX-master: "paper"
%%% End:

\section{Related work}
\label{sec:related:work}

Although 
\clingo~\cite{gekakosc11a} already featured \lua{} as an embedded scripting language up to series~3, 
its usage was limited to (deterministic) computations during grounding;
neither were library functions furnished by \clingo~3.

Of particular interest is
\dlvhex~\cite{eifikrre12a}, an ASP system aiming at the integration of external computation sources.
For this purpose, \dlvhex{} relies on higher-order logic programs using external higher-order atoms for software interoperability.
Such external atoms should not be confused with \clingo's \lstinline{#external} directive because they are evaluated via
procedural means during solving.
Given this,
\dlvhex{} can be seen as an \emph{ASP modulo Theory} solver, similar to SAT modulo Theory solvers~\cite{niolti06a}.
In fact, \dlvhex{} is build upon \clingo{} and follows the design of the \emph{ASP modulo CSP} solver \clingcon~\cite{ostsch12a}
in communicating with external ``oracles'' through \clasp's post propagation mechanism.
In this way,
theory solvers are tightly integrated into the ASP system and
have access to the solver's partial assignments.
Unlike this, multi-shot solving only provides access to total (stable) assignments.
This is why \clingo{} also offers full-fledged theory reasoning capabilities, dealing with partial assignments~\cite{gekakaosscwa16a,kascwa17a}.
Clearly, the above considerations also apply to extensions of \dlvhex, such as \acthex~\cite{figeiaresc13a}.
Furthermore,
\jdlv~\cite{felegrri12a} encapsulates the \dlv{} system to facilitate
one-shot ASP solving in \java{} environments by providing means to
generate and process logic programs,
and to afterwards extract their stable models.
\embasp~\cite{fugezaancape16a} provides a more recent and more general environment for embedding ASP systems,
including \clingo{} and \dlv, into external systems.    
Meanwhile the ASP solver \wasp~\cite{aldoleri15a} also features a foreign language API, 
yet restricted to solving functionalities.
More precisely,
it provides low-level functionalities to customize heuristics and propagation~\cite{dorisc16a,dogalemurisc16a}.

The procedural attachment to the \idp{} system~\cite{powide11a,debobrde14a}
builds on interfaces to \cpp{} and \lua.
Like \clingo, it allows for 
evaluating functions during grounding,
calling the grounder and solver multiple times,
inspecting solutions, and
reacting to external input after search.
The emphasis, however, lies 
on high-level control blending in with \idp's modeling language,
while
\clingo{} offers more fine-grained control over the grounding and solving process,
particularly aiming at a flexible incremental assembly of programs from subprograms.

In SAT, incremental solver interfaces from low-level APIs are common practice.
Pioneering work was done in \minisat~\cite{eensor03a}, 
furnishing a \cpp{} interface for solving under assumptions.
In fact, the \clasp{} library underlying \clingo{} builds upon this functionality to implement incremental search 
(see~\cite{gekakaosscth08a}).
Given that SAT deals with propositional formulas only,
solvers and their APIs lack support for modeling languages and grounding.
Unlike this,
the SAT modulo Theory solver \zzz~\cite{dembjo08a} comes with a \python{} API that,
similar to \clingo,
provides a library for controlling the solver as well as 
language bindings for constraint handling.
In this way, \python{} can be used as a modeling language for \zzz.

%%% Local Variables: 
%%% mode: latex
%%% TeX-master: "paper"
%%% End: 

\section{Conclusion}\label{sec:discussion}

The \clingo{} system complements ASP's declarative input language by control capacities expressed either by embedded scripting languages
or by importing \clingo{} modules into imperative programs.
This is accomplished within a single integrated ASP grounding and solving process in which a logic program may evolve over time.
The addition, deletion, and replacement of programs is controlled procedurally by means of \clingo's API.
Applications that cannot be captured with the standard one-shot approach of ASP but that require evolving logic programs are manifold.
Examples include
unrolling a transition function as in planning,
interacting with an environment as in assisted living, robotics, or stream reasoning,
interacting with a user exploring a domain,
theory solving,
and advanced forms of search.
Addressing these demands by providing a high-level API yields a generic and transparent approach.
Unlike this, previous systems, like \iclingo\ and \oclingo, had a dedicated purpose involving rigid control procedures buried in monolithic programs.
Rather than that,
the basic technology of \clingo{} allows us to instantiate subprograms in-between solver invocations in a fully customizable way.
On the declarative side,
the availability of program parameters and the embedding of \lstinline{#external} directives
into the grounding process provide us with a great flexibility in modeling schematic subprograms.
In addition, the possibility of assigning input atoms facilitates
the implementation of applications such as query answering~\cite{geobsc13a} or sliding window reasoning~\cite{gegrkaobsasc12b}, 
as truth values can now be switched without modifying logic programs.

The semantic underpinnings of our framework in terms of module theory
capture the dynamic combination of logic programs in a generic way.
Although this eases the modular composition of data structures,
other choices are possible at the cost of higher maintenance.
Note that the difficulty of composing subprograms in ASP is due to its nonmonotonic nature;
this is much easier in monotonic approach such as SAT.
Finally,
it is interesting future work to investigate how dedicated change operations
that were so far only of theoretic interest, like
updating~\cite{alpeprpr02a},
forgetting~\cite{zhafoo06a},
revising~\cite{desctowo08a}, or
merging~\cite{desctowo09a} etc.,
can be put into practice within this framework.
A first attempt at capturing the update of logic programs was made by~\citeN{sablei17a}.

The input language of \clingo{} extends the \textit{ASP-Core-2} standard~\cite{aspcore2}
and has meanwhile been put in its entirety on solid semantic foundations~\cite{gehakalisc15a}.
Although we have presented \clingo{} for normal logic programs,
we mention that it accepts (extended) disjunctive logic programs
processed via the multi-threaded solving approach of \clasp~\cite{gekasc13a}.
Since \clingo{} embeds \clasp{} series~3, it moreover features domain-specific heuristics~\cite{gekaotroscwa13a} and
optimization using unsatisfiable cores~\cite{ankamasc12a}.
\clingo{} is freely available at~\url{potassco.org},
and its releases include many best practice examples
illustrating the aforementioned application scenarios.

Since its first release and accompanying publication~\cite{gekakasc14b},
\clingo's multi-shot solving framework has been used for implementing several ASP-based reasoning systems,
such as
\asprin~\cite{brderosc15a,brderosc15b},
\aspic~\cite{geobsc13a},
\rosoclingo~\cite{anrasasc15a},
and
\dflat~\cite{abblchduhewo14a};
various forms of aggregates were implemented with it by \citeN{alfage15a} and \citeN{alvleo15a}.
As well, \dlvhex~\cite{eifikrre12a} builds upon \clingo{} and its versatile API.
This already hints at the potential impact of \clingo's multi-shot ASP solving framework,
and we believe that it constitutes a step towards putting more and more applications into the reach of ASP.

Multi-shot ASP solving broadens the spectrum of applications of ASP.
This also brings about the new user profile of \emph{ASP engineering}
that combines ASP modeling with traditional programming for controlling an ASP solving process.
This may lead to generic advanced forms of ASP solving such as
the incremental approach in Section~\ref{sec:incremental} or be
restricted to customized settings as with the Ricochet Robots game in~Section~\ref{sec:robots}.
Multi-shot solving enables users to engineer such novel declarative systems on top of ASP.
We believe that this new engineering facet is crucial to putting ASP into practice.

%%% Local Variables:
%%% mode: latex
%%% TeX-master: "paper"
%%% End:

\paragraph{Acknowledgments.}

We are grateful to Evgenii Balai, Javier Romero, and Adam Smith
for many fruitful discussions on the paper.
This work was partially funded by DFG grant SCHA 550/9.

%%% Local Variables: 
%%% mode: latex
%%% TeX-master: "paper"
%%% End: 

\nocite{gekakasc14b,gekaobsc15a}
\bibliographystyle{acmtrans}
% \bibliography{lit,akku,procs,misc}

\begin{thebibliography}{}

\bibitem[\protect\citeauthoryear{Abiteboul, Hull, and Vianu}{Abiteboul
  et~al\mbox{.}}{1995}]{abhuvi95a}
{\sc Abiteboul, S.}, {\sc Hull, R.}, {\sc and} {\sc Vianu, V.} 1995.
\newblock {\em Foundations of Databases}.
\newblock Addison-Wesley.

\bibitem[\protect\citeauthoryear{Abseher, Bliem, Charwat, Dusberger, Hecher,
  and Woltran}{Abseher et~al\mbox{.}}{2014}]{abblchduhewo14a}
{\sc Abseher, M.}, {\sc Bliem, B.}, {\sc Charwat, G.}, {\sc Dusberger, F.},
  {\sc Hecher, M.}, {\sc and} {\sc Woltran, S.} 2014.
\newblock The {D-FLAT} system for dynamic programming on tree decompositions.
\newblock In {\em Proceedings of the Fourteenth European Conference on Logics
  in Artificial Intelligence (JELIA'14)}, {E.~Ferm\'{e}} {and} {J.~Leite}, Eds.
  Lecture Notes in Artificial Intelligence, vol. 8761. Springer-Verlag,
  558--572.

\bibitem[\protect\citeauthoryear{Alferes, Pereira, Przymusinska, and
  Przymusinski}{Alferes et~al\mbox{.}}{2002}]{alpeprpr02a}
{\sc Alferes, J.}, {\sc Pereira, L.}, {\sc Przymusinska, H.}, {\sc and} {\sc
  Przymusinski, T.} 2002.
\newblock {LUPS}: A language for updating logic programs.
\newblock {\em Artificial Intelligence\/}~{\em 138,\/}~1-2, 87--116.

\bibitem[\protect\citeauthoryear{Alviano, Dodaro, Leone, and Ricca}{Alviano
  et~al\mbox{.}}{2015}]{aldoleri15a}
{\sc Alviano, M.}, {\sc Dodaro, C.}, {\sc Leone, N.}, {\sc and} {\sc Ricca, F.}
  2015.
\newblock Advances in {WASP}.
\newblock See \citeN{lpnmr15}, 40--54.

\bibitem[\protect\citeauthoryear{Alviano, Faber, and Gebser}{Alviano
  et~al\mbox{.}}{2015}]{alfage15a}
{\sc Alviano, M.}, {\sc Faber, W.}, {\sc and} {\sc Gebser, M.} 2015.
\newblock Rewriting recursive aggregates in answer set programming: Back to
  monotonicity.
\newblock {\em Theory and Practice of Logic Programming\/}~{\em 15,\/}~4-5,
  559--573.
\newblock Available at \url{http://arxiv.org/abs/1507.03923}.

\bibitem[\protect\citeauthoryear{Alviano and Leone}{Alviano and
  Leone}{2015}]{alvleo15a}
{\sc Alviano, M.} {\sc and} {\sc Leone, N.} 2015.
\newblock Complexity and compilation of {GZ}-aggregates in answer set
  programming.
\newblock {\em Theory and Practice of Logic Programming\/}~{\em 15,\/}~4-5,
  574--587.

\bibitem[\protect\citeauthoryear{Andres, Kaufmann, Matheis, and Schaub}{Andres
  et~al\mbox{.}}{2012}]{ankamasc12a}
{\sc Andres, B.}, {\sc Kaufmann, B.}, {\sc Matheis, O.}, {\sc and} {\sc Schaub,
  T.} 2012.
\newblock Unsatisfiability-based optimization in clasp.
\newblock In {\em Technical Communications of the Twenty-eighth International
  Conference on Logic Programming (ICLP'12)}, {A.~Dovier} {and} {V.~{Santos
  Costa}}, Eds. Vol.~17. Leibniz International Proceedings in Informatics
  (LIPIcs), 212--221.

\bibitem[\protect\citeauthoryear{Andres, Rajaratnam, Sabuncu, and
  Schaub}{Andres et~al\mbox{.}}{2015}]{anrasasc15a}
{\sc Andres, B.}, {\sc Rajaratnam, D.}, {\sc Sabuncu, O.}, {\sc and} {\sc
  Schaub, T.} 2015.
\newblock Integrating {ASP} into {ROS} for reasoning in robots.
\newblock See \citeN{lpnmr15}, 69--82.

\bibitem[\protect\citeauthoryear{Balduccini and Janhunen}{Balduccini and
  Janhunen}{2017}]{lpnmr17}
{\sc Balduccini, M.} {\sc and} {\sc Janhunen, T.}, Eds. 2017.
\newblock {\em Proceedings of the Fourteenth International Conference on Logic
  Programming and Nonmonotonic Reasoning (LPNMR'17)}. Lecture Notes in
  Artificial Intelligence, vol. 10377. Springer-Verlag.

\bibitem[\protect\citeauthoryear{Banbara, Kaufmann, Ostrowski, and
  Schaub}{Banbara et~al\mbox{.}}{2017}]{bakaossc16a}
{\sc Banbara, M.}, {\sc Kaufmann, B.}, {\sc Ostrowski, M.}, {\sc and} {\sc
  Schaub, T.} 2017.
\newblock Clingcon: The next generation.
\newblock {\em Theory and Practice of Logic Programming\/}~{\em 17,\/}~4,
  408--461.

\bibitem[\protect\citeauthoryear{Baral}{Baral}{2003}]{baral02a}
{\sc Baral, C.} 2003.
\newblock {\em Knowledge Representation, Reasoning and Declarative Problem
  Solving}.
\newblock Cambridge University Press.

\bibitem[\protect\citeauthoryear{Brewka, Delgrande, Romero, and Schaub}{Brewka
  et~al\mbox{.}}{2015a}]{brderosc15a}
{\sc Brewka, G.}, {\sc Delgrande, J.}, {\sc Romero, J.}, {\sc and} {\sc Schaub,
  T.} 2015a.
\newblock asprin: Customizing answer set preferences without a headache.
\newblock In {\em Proceedings of the Twenty-Ninth National Conference on
  Artificial Intelligence (AAAI'15)}, {B.~Bonet} {and} {S.~Koenig}, Eds. AAAI
  Press, 1467--1474.

\bibitem[\protect\citeauthoryear{Brewka, Delgrande, Romero, and Schaub}{Brewka
  et~al\mbox{.}}{2015b}]{brderosc15b}
{\sc Brewka, G.}, {\sc Delgrande, J.}, {\sc Romero, J.}, {\sc and} {\sc Schaub,
  T.} 2015b.
\newblock Implementing preferences with asprin.
\newblock See \citeN{lpnmr15}, 158--172.

\bibitem[\protect\citeauthoryear{Brewka, Eiter, and McIlraith}{Brewka
  et~al\mbox{.}}{2012}]{kr12}
{\sc Brewka, G.}, {\sc Eiter, T.}, {\sc and} {\sc McIlraith, S.}, Eds. 2012.
\newblock {\em Proceedings of the Thirteenth International Conference on
  Principles of Knowledge Representation and Reasoning (KR'12)}. AAAI Press.

\bibitem[\protect\citeauthoryear{Cabalar and Son}{Cabalar and
  Son}{2013}]{lpnmr13}
{\sc Cabalar, P.} {\sc and} {\sc Son, T.}, Eds. 2013.
\newblock {\em Proceedings of the Twelfth International Conference on Logic
  Programming and Nonmonotonic Reasoning (LPNMR'13)}. Lecture Notes in
  Artificial Intelligence, vol. 8148. Springer-Verlag.

\bibitem[\protect\citeauthoryear{Calimeri, Cozza, and Ianni}{Calimeri
  et~al\mbox{.}}{2007}]{cacoia07a}
{\sc Calimeri, F.}, {\sc Cozza, S.}, {\sc and} {\sc Ianni, G.} 2007.
\newblock External sources of knowledge and value invention in logic
  programming.
\newblock {\em Annals of Mathematics and Artificial Intelligence\/}~{\em
  50,\/}~3-4, 333--361.

\bibitem[\protect\citeauthoryear{Calimeri, Faber, Gebser, Ianni, Kaminski,
  Krennwallner, Leone, Ricca, and Schaub}{Calimeri
  et~al\mbox{.}}{2012}]{aspcore2}
{\sc Calimeri, F.}, {\sc Faber, W.}, {\sc Gebser, M.}, {\sc Ianni, G.}, {\sc
  Kaminski, R.}, {\sc Krennwallner, T.}, {\sc Leone, N.}, {\sc Ricca, F.}, {\sc
  and} {\sc Schaub, T.} 2012.
\newblock {ASP-Core-2}: Input language format.
\newblock Available at
  \url{https://www.mat.unical.it/aspcomp2013/ASPStandardization}.

\bibitem[\protect\citeauthoryear{Calimeri, Ianni, and
  Truszczy{\'n}ski}{Calimeri et~al\mbox{.}}{2015}]{lpnmr15}
{\sc Calimeri, F.}, {\sc Ianni, G.}, {\sc and} {\sc Truszczy{\'n}ski, M.}, Eds.
  2015.
\newblock {\em Proceedings of the Thirteenth International Conference on Logic
  Programming and Nonmonotonic Reasoning (LPNMR'15)}. Lecture Notes in
  Artificial Intelligence, vol. 9345. Springer-Verlag.

\bibitem[\protect\citeauthoryear{{De Cat}, Bogaerts, Bruynooghe, and
  Denecker}{{De Cat} et~al\mbox{.}}{2014}]{debobrde14a}
{\sc {De Cat}, B.}, {\sc Bogaerts, B.}, {\sc Bruynooghe, M.}, {\sc and} {\sc
  Denecker, M.} 2014.
\newblock Predicate logic as a modelling language: The {IDP} system.
\newblock {\em CoRR\/}~{\em abs/1401.6312}.

\bibitem[\protect\citeauthoryear{{de Moura} and Bj{\o}rner}{{de Moura} and
  Bj{\o}rner}{2008}]{dembjo08a}
{\sc {de Moura}, L.} {\sc and} {\sc Bj{\o}rner, N.} 2008.
\newblock {Z3}: An efficient {SMT} solver.
\newblock In {\em Proceedings of the Fourteenth International Conference on
  Tools and Algorithms for the Construction and Analysis of Systems
  (TACAS'08)}, {C.~Ramakrishnan} {and} {J.~Rehof}, Eds. Lecture Notes in
  Computer Science, vol. 4963. Springer-Verlag, 337--340.

\bibitem[\protect\citeauthoryear{{De Pooter}, Wittocx, and Denecker}{{De
  Pooter} et~al\mbox{.}}{2013}]{powide11a}
{\sc {De Pooter}, S.}, {\sc Wittocx, J.}, {\sc and} {\sc Denecker, M.} 2013.
\newblock A prototype of a knowledge-based programming environment.
\newblock In {\em Proceedings of the Nineteenth International Conference on
  Applications of Declarative Programming and Knowledge Management (INAP'11)
  and the Twenty-fifth Workshop on Logic Programming (WLP'11)}, {H.~Tompits},
  {S.~Abreu}, {J.~Oetsch}, {J.~P{\"u}hrer}, {D.~Seipel}, {M.~Umeda}, {and}
  {A.~Wolf}, Eds. Lecture Notes in Computer Science, vol. 7773.
  Springer-Verlag, 279--286.

\bibitem[\protect\citeauthoryear{Delgrande and Faber}{Delgrande and
  Faber}{2011}]{lpnmr11}
{\sc Delgrande, J.} {\sc and} {\sc Faber, W.}, Eds. 2011.
\newblock {\em Proceedings of the Eleventh International Conference on Logic
  Programming and Nonmonotonic Reasoning (LPNMR'11)}. Lecture Notes in
  Artificial Intelligence, vol. 6645. Springer-Verlag.

\bibitem[\protect\citeauthoryear{Delgrande, Schaub, Tompits, and
  Woltran}{Delgrande et~al\mbox{.}}{2008}]{desctowo08a}
{\sc Delgrande, J.}, {\sc Schaub, T.}, {\sc Tompits, H.}, {\sc and} {\sc
  Woltran, S.} 2008.
\newblock Belief revision of logic programs under answer set semantics.
\newblock In {\em Proceedings of the Eleventh International Conference on
  Principles of Knowledge Representation and Reasoning (KR'08)}, {G.~Brewka}
  {and} {J.~Lang}, Eds. AAAI Press, 411--421.

\bibitem[\protect\citeauthoryear{Delgrande, Schaub, Tompits, and
  Woltran}{Delgrande et~al\mbox{.}}{2009}]{desctowo09a}
{\sc Delgrande, J.}, {\sc Schaub, T.}, {\sc Tompits, H.}, {\sc and} {\sc
  Woltran, S.} 2009.
\newblock Merging logic programs under answer set semantics.
\newblock In {\em Proceedings of the Twenty-fifth International Conference on
  Logic Programming (ICLP'09)}, {P.~Hill} {and} {D.~Warren}, Eds. Lecture Notes
  in Computer Science, vol. 5649. Springer-Verlag, 160--174.

\bibitem[\protect\citeauthoryear{Dimopoulos, Gebser, Lühne, Romero, and
  Schaub}{Dimopoulos et~al\mbox{.}}{2017}]{digelurosc17a}
{\sc Dimopoulos, Y.}, {\sc Gebser, M.}, {\sc Lühne, P.}, {\sc Romero, J.},
  {\sc and} {\sc Schaub, T.} 2017.
\newblock plasp 3: Towards effective {ASP} planning.
\newblock See \citeN{lpnmr17}, 286--300.

\bibitem[\protect\citeauthoryear{Dodaro, Gasteiger, Leone, Musitsch, Ricca, and
  Schekotihin}{Dodaro et~al\mbox{.}}{2016}]{dogalemurisc16a}
{\sc Dodaro, C.}, {\sc Gasteiger, P.}, {\sc Leone, N.}, {\sc Musitsch, B.},
  {\sc Ricca, F.}, {\sc and} {\sc Schekotihin, K.} 2016.
\newblock Driving {CDCL} search.
\newblock {\em CoRR\/}~{\em abs/1611.05190}.

\bibitem[\protect\citeauthoryear{Dodaro, Ricca, and Schüller}{Dodaro
  et~al\mbox{.}}{2016}]{dorisc16a}
{\sc Dodaro, C.}, {\sc Ricca, F.}, {\sc and} {\sc Schüller, P.} 2016.
\newblock External propagators in wasp: Preliminary report.
\newblock In {\em Proceedings of the Twenty-third International Workshop on
  Experimental Evaluation of Algorithms for Solving Problems with Combinatorial
  Explosion (RCRA'16)}. Vol. 1745. CEUR Workshop Proceedings, 1--9.

\bibitem[\protect\citeauthoryear{E{\'e}n and Sörensson}{E{\'e}n and
  Sörensson}{2003}]{eensor03b}
{\sc E{\'e}n, N.} {\sc and} {\sc Sörensson, N.} 2003.
\newblock Temporal induction by incremental {SAT} solving.
\newblock {\em Electronic Notes in Theoretical Computer Science\/}~{\em
  89,\/}~4.

\bibitem[\protect\citeauthoryear{E{\'e}n and Sörensson}{E{\'e}n and
  Sörensson}{2004}]{eensor03a}
{\sc E{\'e}n, N.} {\sc and} {\sc Sörensson, N.} 2004.
\newblock An extensible {SAT}-solver.
\newblock In {\em Proceedings of the Sixth International Conference on Theory
  and Applications of Satisfiability Testing (SAT'03)}, {E.~Giunchiglia} {and}
  {A.~Tacchella}, Eds. Lecture Notes in Computer Science, vol. 2919.
  Springer-Verlag, 502--518.

\bibitem[\protect\citeauthoryear{Eiter, Fink, Krennwallner, and Redl}{Eiter
  et~al\mbox{.}}{2012}]{eifikrre12a}
{\sc Eiter, T.}, {\sc Fink, M.}, {\sc Krennwallner, T.}, {\sc and} {\sc Redl,
  C.} 2012.
\newblock Conflict-driven {ASP} solving with external sources.
\newblock {\em Theory and Practice of Logic Programming\/}~{\em 12,\/}~4-5,
  659--679.

\bibitem[\protect\citeauthoryear{Febbraro, Leone, Grasso, and Ricca}{Febbraro
  et~al\mbox{.}}{2012}]{felegrri12a}
{\sc Febbraro, O.}, {\sc Leone, N.}, {\sc Grasso, G.}, {\sc and} {\sc Ricca,
  F.} 2012.
\newblock {JASP}: A framework for integrating answer set programming with
  {J}ava.
\newblock See \citeN{kr12}, 541--551.

\bibitem[\protect\citeauthoryear{Fink, Germano, Ianni, Redl, and
  Schüller}{Fink et~al\mbox{.}}{2013}]{figeiaresc13a}
{\sc Fink, M.}, {\sc Germano, S.}, {\sc Ianni, G.}, {\sc Redl, C.}, {\sc and}
  {\sc Schüller, P.} 2013.
\newblock {ActHEX}: Implementing {HEX} programs with action atoms.
\newblock See \citeN{lpnmr13}, 317--322.

\bibitem[\protect\citeauthoryear{Fusc{\`{a}}, Germano, Zangari, Anastasio,
  Calimeri, and Perri}{Fusc{\`{a}} et~al\mbox{.}}{2016}]{fugezaancape16a}
{\sc Fusc{\`{a}}, D.}, {\sc Germano, S.}, {\sc Zangari, J.}, {\sc Anastasio,
  M.}, {\sc Calimeri, F.}, {\sc and} {\sc Perri, S.} 2016.
\newblock A framework for easing the development of applications embedding
  answer set programming.
\newblock In {\em Proceedings of the Eighteenth International Symposium on
  Principles and Practice of Declarative Programming (PPDP'16)}, {J.~Cheney}
  {and} {G.~Vidal}, Eds. {ACM} Press, 38--49.

\bibitem[\protect\citeauthoryear{Gebser, Grote, Kaminski, Obermeier, Sabuncu,
  and Schaub}{Gebser et~al\mbox{.}}{2012}]{gegrkaobsasc12b}
{\sc Gebser, M.}, {\sc Grote, T.}, {\sc Kaminski, R.}, {\sc Obermeier, P.},
  {\sc Sabuncu, O.}, {\sc and} {\sc Schaub, T.} 2012.
\newblock Stream reasoning with answer set programming: Preliminary report.
\newblock See \citeN{kr12}, 613--617.

\bibitem[\protect\citeauthoryear{Gebser, Grote, Kaminski, and Schaub}{Gebser
  et~al\mbox{.}}{2011}]{gegrkasc11a}
{\sc Gebser, M.}, {\sc Grote, T.}, {\sc Kaminski, R.}, {\sc and} {\sc Schaub,
  T.} 2011.
\newblock Reactive answer set programming.
\newblock See \citeN{lpnmr11}, 54--66.

\bibitem[\protect\citeauthoryear{Gebser, Harrison, Kaminski, Lifschitz, and
  Schaub}{Gebser et~al\mbox{.}}{2015}]{gehakalisc15a}
{\sc Gebser, M.}, {\sc Harrison, A.}, {\sc Kaminski, R.}, {\sc Lifschitz, V.},
  {\sc and} {\sc Schaub, T.} 2015.
\newblock Abstract {G}ringo.
\newblock {\em Theory and Practice of Logic Programming\/}~{\em 15,\/}~4-5,
  449--463.
\newblock Available at \url{http://arxiv.org/abs/1507.06576}.

\bibitem[\protect\citeauthoryear{Gebser, Jost, Kaminski, Obermeier, Sabuncu,
  Schaub, and Schneider}{Gebser et~al\mbox{.}}{2013}]{gejokaobsascsc13a}
{\sc Gebser, M.}, {\sc Jost, H.}, {\sc Kaminski, R.}, {\sc Obermeier, P.}, {\sc
  Sabuncu, O.}, {\sc Schaub, T.}, {\sc and} {\sc Schneider, M.} 2013.
\newblock Ricochet robots: A transverse {ASP} benchmark.
\newblock See \citeN{lpnmr13}, 348--360.

\bibitem[\protect\citeauthoryear{Gebser, Kaminski, Kaufmann, Lindauer,
  Ostrowski, Romero, Schaub, and Thiele}{Gebser
  et~al\mbox{.}}{2015}]{PotasscoUserGuide}
{\sc Gebser, M.}, {\sc Kaminski, R.}, {\sc Kaufmann, B.}, {\sc Lindauer, M.},
  {\sc Ostrowski, M.}, {\sc Romero, J.}, {\sc Schaub, T.}, {\sc and} {\sc
  Thiele, S.} 2015.
\newblock {\em Potassco User Guide\/}, Second edition ed.
\newblock University of Potsdam.

\bibitem[\protect\citeauthoryear{Gebser, Kaminski, Kaufmann, Ostrowski, Schaub,
  and Schneider}{Gebser et~al\mbox{.}}{2011}]{gekakaosscsc11a}
{\sc Gebser, M.}, {\sc Kaminski, R.}, {\sc Kaufmann, B.}, {\sc Ostrowski, M.},
  {\sc Schaub, T.}, {\sc and} {\sc Schneider, M.} 2011.
\newblock Potassco: The {P}otsdam answer set solving collection.
\newblock {\em AI Communications\/}~{\em 24,\/}~2, 107--124.

\bibitem[\protect\citeauthoryear{Gebser, Kaminski, Kaufmann, Ostrowski, Schaub,
  and Thiele}{Gebser et~al\mbox{.}}{2008}]{gekakaosscth08a}
{\sc Gebser, M.}, {\sc Kaminski, R.}, {\sc Kaufmann, B.}, {\sc Ostrowski, M.},
  {\sc Schaub, T.}, {\sc and} {\sc Thiele, S.} 2008.
\newblock Engineering an incremental {ASP} solver.
\newblock In {\em Proceedings of the Twenty-fourth International Conference on
  Logic Programming (ICLP'08)}, {M.~{Garcia de la Banda}} {and} {E.~Pontelli},
  Eds. Lecture Notes in Computer Science, vol. 5366. Springer-Verlag, 190--205.

\bibitem[\protect\citeauthoryear{Gebser, Kaminski, Kaufmann, Ostrowski, Schaub,
  and Wanko}{Gebser et~al\mbox{.}}{2016}]{gekakaosscwa16a}
{\sc Gebser, M.}, {\sc Kaminski, R.}, {\sc Kaufmann, B.}, {\sc Ostrowski, M.},
  {\sc Schaub, T.}, {\sc and} {\sc Wanko, P.} 2016.
\newblock Theory solving made easy with clingo~5.
\newblock In {\em Technical Communications of the Thirty-second International
  Conference on Logic Programming (ICLP'16)}, {M.~Carro} {and} {A.~King}, Eds.
  Vol.~52. Open Access Series in Informatics (OASIcs), 2:1--2:15.

\bibitem[\protect\citeauthoryear{Gebser, Kaminski, Kaufmann, and Schaub}{Gebser
  et~al\mbox{.}}{2012}]{gekakasc12a}
{\sc Gebser, M.}, {\sc Kaminski, R.}, {\sc Kaufmann, B.}, {\sc and} {\sc
  Schaub, T.} 2012.
\newblock {\em Answer Set Solving in Practice}.
\newblock Synthesis Lectures on Artificial Intelligence and Machine Learning.
  Morgan and Claypool Publishers.

\bibitem[\protect\citeauthoryear{Gebser, Kaminski, Kaufmann, and Schaub}{Gebser
  et~al\mbox{.}}{2014}]{gekakasc14b}
{\sc Gebser, M.}, {\sc Kaminski, R.}, {\sc Kaufmann, B.}, {\sc and} {\sc
  Schaub, T.} 2014.
\newblock \textit{Clingo} = {ASP} + control: Preliminary report.
\newblock In {\em Technical Communications of the Thirtieth International
  Conference on Logic Programming (ICLP'14)}, {M.~Leuschel} {and}
  {T.~Schrijvers}, Eds. Theory and Practice of Logic Programming, Online
  Supplement, vol. 14(4-5).
\newblock Available at \url{http://arxiv.org/abs/1405.3694v1}.

\bibitem[\protect\citeauthoryear{Gebser, Kaminski, König, and Schaub}{Gebser
  et~al\mbox{.}}{2011}]{gekakosc11a}
{\sc Gebser, M.}, {\sc Kaminski, R.}, {\sc König, A.}, {\sc and} {\sc Schaub,
  T.} 2011.
\newblock Advances in gringo series 3.
\newblock See \citeN{lpnmr11}, 345--351.

\bibitem[\protect\citeauthoryear{Gebser, Kaminski, Obermeier, and
  Schaub}{Gebser et~al\mbox{.}}{2015}]{gekaobsc15a}
{\sc Gebser, M.}, {\sc Kaminski, R.}, {\sc Obermeier, P.}, {\sc and} {\sc
  Schaub, T.} 2015.
\newblock Ricochet robots reloaded: A case-study in multi-shot {ASP} solving.
\newblock In {\em Advances in Knowledge Representation, Logic Programming, and
  Abstract Argumentation: Essays Dedicated to {G}erhard {B}rewka on the
  Occasion of His 60th Birthday}, {T.~Eiter}, {H.~Strass},
  {M.~Truszczy{\'n}ski}, {and} {S.~Woltran}, Eds. Lecture Notes in Artificial
  Intelligence, vol. 9060. Springer-Verlag, 17--32.

\bibitem[\protect\citeauthoryear{Gebser, Kaufmann, Otero, Romero, Schaub, and
  Wanko}{Gebser et~al\mbox{.}}{2013}]{gekaotroscwa13a}
{\sc Gebser, M.}, {\sc Kaufmann, B.}, {\sc Otero, R.}, {\sc Romero, J.}, {\sc
  Schaub, T.}, {\sc and} {\sc Wanko, P.} 2013.
\newblock Domain-specific heuristics in answer set programming.
\newblock In {\em Proceedings of the Twenty-Seventh National Conference on
  Artificial Intelligence (AAAI'13)}, {M.~{desJardins}} {and} {M.~Littman},
  Eds. AAAI Press, 350--356.

\bibitem[\protect\citeauthoryear{Gebser, Kaufmann, and Schaub}{Gebser
  et~al\mbox{.}}{2013}]{gekasc13a}
{\sc Gebser, M.}, {\sc Kaufmann, B.}, {\sc and} {\sc Schaub, T.} 2013.
\newblock Advanced conflict-driven disjunctive answer set solving.
\newblock In {\em Proceedings of the Twenty-third International Joint
  Conference on Artificial Intelligence (IJCAI'13)}, {F.~Rossi}, Ed. IJCAI/AAAI
  Press, 912--918.

\bibitem[\protect\citeauthoryear{Gebser, Obermeier, and Schaub}{Gebser
  et~al\mbox{.}}{2013}]{geobsc13a}
{\sc Gebser, M.}, {\sc Obermeier, P.}, {\sc and} {\sc Schaub, T.} 2013.
\newblock A system for interactive query answering with answer set programming.
\newblock In {\em Proceedings of the Sixth Workshop on Answer Set Programming
  and Other Computing Paradigms (ASPOCP'13)}, {M.~Fink} {and} {Y.~Lierler},
  Eds. Vol. abs/1312.6143. CoRR.

\bibitem[\protect\citeauthoryear{Gelfond and Lifschitz}{Gelfond and
  Lifschitz}{1988}]{gellif88b}
{\sc Gelfond, M.} {\sc and} {\sc Lifschitz, V.} 1988.
\newblock The stable model semantics for logic programming.
\newblock In {\em Proceedings of the Fifth International Conference and
  Symposium of Logic Programming (ICLP'88)}, {R.~Kowalski} {and} {K.~Bowen},
  Eds. MIT Press, 1070--1080.

\bibitem[\protect\citeauthoryear{Janhunen, Kaminski, Ostrowski, Schaub,
  Schellhorn, and Wanko}{Janhunen et~al\mbox{.}}{2017}]{jakaosscscwa17a}
{\sc Janhunen, T.}, {\sc Kaminski, R.}, {\sc Ostrowski, M.}, {\sc Schaub, T.},
  {\sc Schellhorn, S.}, {\sc and} {\sc Wanko, P.} 2017.
\newblock Clingo goes linear constraints over reals and integers.
\newblock {\em Theory and Practice of Logic Programming\/}~{\em 17,\/}~5-6,
  872--888.

\bibitem[\protect\citeauthoryear{Kaminski, Schaub, and Wanko}{Kaminski
  et~al\mbox{.}}{2017}]{kascwa17a}
{\sc Kaminski, R.}, {\sc Schaub, T.}, {\sc and} {\sc Wanko, P.} 2017.
\newblock A tutorial on hybrid answer set solving with clingo.
\newblock In {\em Proceedings of the Thirteenth International Summer School of
  the Reasoning Web}, {G.~Ianni}, {D.~Lembo}, {L.~Bertossi}, {W.~Faber},
  {B.~Glimm}, {G.~Gottlob}, {and} {S.~Staab}, Eds. Lecture Notes in Computer
  Science, vol. 10370. Springer-Verlag, 167--203.

\bibitem[\protect\citeauthoryear{Kaufmann, Leone, Perri, and Schaub}{Kaufmann
  et~al\mbox{.}}{2016}]{kalepesc16a}
{\sc Kaufmann, B.}, {\sc Leone, N.}, {\sc Perri, S.}, {\sc and} {\sc Schaub,
  T.} 2016.
\newblock Grounding and solving in answer set programming.
\newblock {\em {AI} Magazine\/}~{\em 37,\/}~3, 25--32.

\bibitem[\protect\citeauthoryear{Leone, Pfeifer, Faber, Eiter, Gottlob, Perri,
  and Scarcello}{Leone et~al\mbox{.}}{2006}]{dlv03a}
{\sc Leone, N.}, {\sc Pfeifer, G.}, {\sc Faber, W.}, {\sc Eiter, T.}, {\sc
  Gottlob, G.}, {\sc Perri, S.}, {\sc and} {\sc Scarcello, F.} 2006.
\newblock The {DLV} system for knowledge representation and reasoning.
\newblock {\em ACM Transactions on Computational Logic\/}~{\em 7,\/}~3,
  499--562.

\bibitem[\protect\citeauthoryear{Lifschitz and Turner}{Lifschitz and
  Turner}{1994}]{liftur94a}
{\sc Lifschitz, V.} {\sc and} {\sc Turner, H.} 1994.
\newblock Splitting a logic program.
\newblock In {\em Proceedings of the Eleventh International Conference on Logic
  Programming}. MIT Press, 23--37.

\bibitem[\protect\citeauthoryear{Nieuwenhuis, Oliveras, and
  Tinelli}{Nieuwenhuis et~al\mbox{.}}{2006}]{niolti06a}
{\sc Nieuwenhuis, R.}, {\sc Oliveras, A.}, {\sc and} {\sc Tinelli, C.} 2006.
\newblock Solving {SAT} and {SAT} modulo theories: From an abstract
  {D}avis-{P}utnam-{L}ogemann-{L}oveland procedure to {DPLL}({T}).
\newblock {\em Journal of the ACM\/}~{\em 53,\/}~6, 937--977.

\bibitem[\protect\citeauthoryear{Oikarinen and Janhunen}{Oikarinen and
  Janhunen}{2006}]{oikjan06a}
{\sc Oikarinen, E.} {\sc and} {\sc Janhunen, T.} 2006.
\newblock Modular equivalence for normal logic programs.
\newblock In {\em Proceedings of the Seventeenth European Conference on
  Artificial Intelligence (ECAI'06)}, {G.~Brewka}, {S.~Coradeschi},
  {A.~Perini}, {and} {P.~Traverso}, Eds. IOS Press, 412--416.

\bibitem[\protect\citeauthoryear{Ostrowski and Schaub}{Ostrowski and
  Schaub}{2012}]{ostsch12a}
{\sc Ostrowski, M.} {\sc and} {\sc Schaub, T.} 2012.
\newblock {ASP} modulo {CSP}: The clingcon system.
\newblock {\em Theory and Practice of Logic Programming\/}~{\em 12,\/}~4-5,
  485--503.

\bibitem[\protect\citeauthoryear{Sabuncu and Leite}{Sabuncu and
  Leite}{2017}]{sablei17a}
{\sc Sabuncu, O.} {\sc and} {\sc Leite, J.} 2017.
\newblock moviola: Interpreting dynamic logic programs via multi-shot answer
  set programming.
\newblock See \citeN{lpnmr17}, 336--342.

\bibitem[\protect\citeauthoryear{Simons, Niemelä, and Soininen}{Simons
  et~al\mbox{.}}{2002}]{siniso02a}
{\sc Simons, P.}, {\sc Niemelä, I.}, {\sc and} {\sc Soininen, T.} 2002.
\newblock Extending and implementing the stable model semantics.
\newblock {\em Artificial Intelligence\/}~{\em 138,\/}~1-2, 181--234.

\bibitem[\protect\citeauthoryear{Syrjänen}{Syrjänen}{2001}]{lparseManual}
{\sc Syrjänen, T.} 2001.
\newblock Lparse 1.0 user's manual.

\bibitem[\protect\citeauthoryear{Zhang and Foo}{Zhang and
  Foo}{2006}]{zhafoo06a}
{\sc Zhang, Y.} {\sc and} {\sc Foo, N.} 2006.
\newblock Solving logic program conflict through strong and weak forgettings.
\newblock {\em Artificial Intelligence\/}~{\em 170,\/}~8-9, 739--778.

\end{thebibliography}

\end{document}